\documentclass[letterpaper]{article}%
\usepackage{aaai16}
\usepackage{times}
\usepackage{helvet}
\usepackage{courier}

\usepackage{amsmath}
\usepackage{amssymb}
\usepackage{graphicx}
\usepackage{url}

\newtheorem{prop}{Proposition}
\newtheorem{lemma}{Lemma}
\newtheorem{thm}{Theorem}
\newtheorem{example}{Example}

\newtheorem{cor}{Corollary}
\newtheorem{definition}{Definition}

\long\def\BOC#1\EOC{\message{(Commented text )}}
\long\def\BOCC#1\EOCC{\message{(Commented text )}}
\long\def\BOCCC#1\EOCCC{\message{(Commented text )}}
\long\def\optional#1{\empty}
\long\def\NB#1{}
\long\def\NC#1{}

\def\o{\overline}
\def\bi{\begin{itemize}}
\def\ii{\item}
\def\ei{\end{itemize}}
\def\beq{\begin{equation}}
\def\eeq#1{\label{#1}\end{equation}}
\def\ba{\begin{array}}
\def\ea{\end{array}}
\def\i#1{\hbox{\it #1\/}}
\def\mi#1{\mathit{#1}}

\def\sm{\hbox{\rm SM}}
\def\lpmln{\hbox{\rm LP}^{\rm{MLN}}}
\def\lpmln{{\rm LP}^{\rm{MLN}}}
\def\no{\i{not}}
\def\sneg{\sim\!\!}
\def\ar{\leftarrow}
\def\rar{\rightarrow}

\def\no{\i{not}}
\def\mvis{\!=\!}
\def\false{\hbox{\bf f}}
\def\true{\hbox{\bf t}}

\def\i#1{\hbox{\itshape #1\/}}

\def\qed{\quad \vrule height7.5pt width4.17pt depth0pt \medskip}
\def\smmodels{\models_{\text{\sm}}}
\def\wdbprm{W^{\prime\prime}}
\def\pdbprm{P^{\prime\prime}}
\def\smdbprm{SM^{\prime\prime}}
\def\grdtms{\textbf{\textit{At}}}

\def\L{\mathbb{L}}
\def\P{\mathbb{P}}

\begin{document}

\title{Weighted Rules under the Stable Model Semantics}

\author{Joohyung Lee and Yi Wang\\ 
School of Computing, Informatics and Decision Systems Engineering \\
Arizona State University, Tempe, USA \\
{\tt \{joolee, ywang485\}@asu.edu}
}

\maketitle

\begin{abstract}
We introduce the concept of weighted rules under the stable model semantics following the log-linear models of Markov Logic. This provides versatile methods to overcome the deterministic nature of the stable model semantics, such as resolving inconsistencies in answer set programs, ranking stable models, associating probability to stable models, and applying statistical inference to computing weighted stable models. We also present formal comparisons with related formalisms, such as answer set programs, Markov Logic, ProbLog, and P-log.
\end{abstract}

\section{Introduction} \label{sec:intro}

Logic programs under the stable model semantics \cite{gel88} is the language of Answer Set Programming (ASP). 
Many extensions of the stable model semantics have been proposed to incorporate various constructs in knowledge representation. Some of them are related to overcoming the ``crisp'' or deterministic nature of the stable model semantics by ranking stable models using weak constraints \cite{buccafurri00enhancing}, by resolving inconsistencies using Consistency Restoring rules  \cite{balduccini03logic} or possibilistic measure \cite{bauters10possibilistic}, and by assigning probability to stable models \cite{baral09probabilistic,nickles14probabilistic}. 

In this paper, we present an alternative approach by introducing the notion of weights into the stable model semantics following the log-linear models of Markov Logic \cite{richardson06markov}, a successful approach to combining first-order logic and probabilistic graphical models. 
Instead of the concept of classical models adopted in Markov Logic, language $\lpmln$ adopts stable models as the logical component. The relationship between $\lpmln$ and Markov Logic is analogous to the known relationship between ASP and SAT. 
Indeed, many technical results about the relationship between SAT and ASP naturally carry over between $\lpmln$ and Markov Logic. In particular, an implementation of Markov Logic can be used to compute ``tight'' $\lpmln$ programs, similar to the way ``tight'' ASP programs can be computed by SAT solvers. 

It is also interesting that the relationship between Markov Logic and SAT is analogous to the relationship between $\lpmln$ and ASP: the way that Markov Logic extends SAT in a probabilistic way is similar to the way that $\lpmln$ extends ASP in a probabilistic way. This can be summarized as in the following figure. (The parallel edges imply that the ways that the extensions are defined are similar to each other.)

\begin{center}
  \includegraphics[height=2cm]{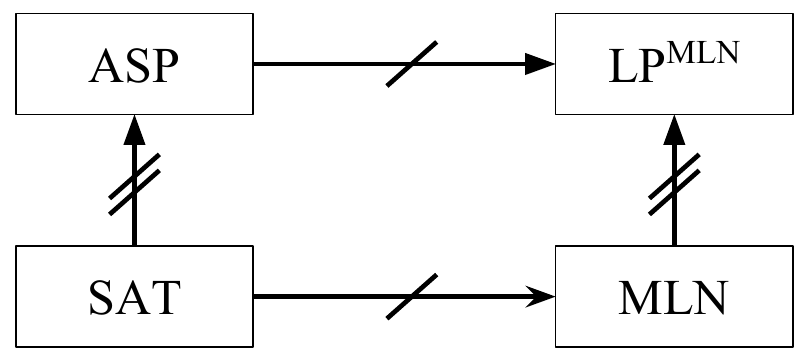}
\end{center}


Weighted rules of $\lpmln$ provides a way to resolve inconsistencies among  ASP knowledge bases, possibly obtained from different sources with different certainty levels. For example, consider the simple ASP knowledge base  $\i{KB}_1$: 
\[
\ba {lrcl}
   & \i{Bird}(x) &\ar& \i{ResidentBird}(x)  \\
   & \i{Bird}(x) &\ar& \i{MigratoryBird}(x)  \\
   &   &\ar& \i{ResidentBird}(x), \i{MigratoryBird}(x) . \\
\ea
\]
One data source $\i{KB}_2$ (possibly acquired by some information extraction module) says that $\i{Jo}$ is a $\i{ResidentBird}$:
\[
\ba {lrcl}
  & \i{ResidentBird}(\i{Jo})
\ea
\] 
while another data source $\i{KB}_3$ states that $\i{Jo}$ is a $\i{MigratoryBird}$: 
\[
\ba {lrcl}
   & \i{MigratoryBird}(\i{Jo}).
\ea
\]
The data about $\i{Jo}$ is actually inconsistent w.r.t.~$\i{KB}_1$, so under the (deterministic) stable model semantics, the combined knowledge base $\i{KB}= \i{KB}_1\cup \i{KB}_2\cup \i{KB}_3$ is not so meaningful. On the other hand, it is still intuitive to conclude that  
$\i{Jo}$ is likely a $\i{Bird}$, and may be a $\i{ResidentBird}$ or a $\i{MigratoryBird}$. 
Such reasoning is supported in $\lpmln$.

Under some reasonable assumption, normalized weights of stable models can be understood as probabilities of the stable models.  We show that ProbLog \cite{deraedt07problog,fierens15inference} can be viewed as a special case of $\lpmln$. Furthermore, we present a subset of $\lpmln$ where probability is naturally expressed and show how it captures
a meaningful fragment of P-log \cite{baral09probabilistic}. In combination of the result that relates $\lpmln$ to Markov Logic, the translation from P-log to $\lpmln$ yields an alternative, more scalable method for computing the fragment of P-log using standard implementations of Markov Logic.

The paper is organized as follows. After reviewing the deterministic stable model semantics, we define the language $\lpmln$ and demonstrate how it can be used for resolving inconsistencies. Then we relate $\lpmln$ to each of ASP, Markov Logic, and ProbLog, and define a fragment of $\lpmln$ language that allows probability to be represented in a more natural way. Next we show how a fragment of P-log can be turned into that fragment of $\lpmln$, and demonstrate the effectiveness of the translation-based computation of the P-log fragment over the existing implementation of P-log.

This paper is an extended version of  \cite{lee15aprobabilistic,lee15markov}. The proofs are available from the longer version at \url{http://reasoning.eas.asu.edu/papers/lpmln-kr-long.pdf}.




\section{Review: Stable Model Semantics}  \label{sec:prelim}


We assume a first-order signature~$\sigma$ that contains no function constants of positive arity, which yields finitely many Herbrand interpretations. 



We say that a formula is {\sl negative} if every occurrence of every atom in this formula is in the scope of negation. 

A {\em rule} is of the form
\beq
\ba l
    A\ar B\land N
\ea 
\eeq{rule}
where $A$ is a disjunction of atoms, $B$ is a conjunction of atoms, and $N$ is a negative formula constructed from atoms using conjunction, disjunction and negation.  
We identify rule~\eqref{rule} with formula $B\land N\rar A$.
We often use comma for conjunction, semi-colon for disjunction, $\no$ for negation, as widely used in the literature on logic programming. For example, $N$ could be 
\[
  \neg B_{m+1}\!\land\!\dots\!\land\!\neg B_n  
  \!\land\!\neg\neg B_{n+1}\!\land\!\dots\!\land\!\neg\neg B_p, 
\]
which can be also written as 
\[
\no\ B_{m+1}, \dots, \no\ B_n, \no\ \no\ B_{n+1},\dots, \no\ \no\ B_p.
\]

We write $\{A_1\}^{\rm ch}\ar \i{Body}$, where $A_1$ is an atom, to denote the rule $A_1\ar \i{Body}\land \neg\neg A_1$. This expression is called a ``choice rule'' in ASP.   If the head of a rule ($A$ in \eqref{rule}) is $\bot$, we often omit it and call such a rule {\em constraint}.

A {\sl logic program} is a finite conjunction of rules. A logic program is called {\sl ground} if it contains no variables.

We say that an Herbrand interpretation $I$ is a {\em model} of a ground program $\Pi$ if $I$ satisfies all implications~\eqref{rule} in~$\Pi$ (as in classical logic). Such models can be divided into two groups: ``stable'' and ``non-stable'' models, which are distinguished as follows. 
The {\em reduct} of $\Pi$ relative to $I$, denoted $\Pi^I$, consists of 
``$
  A\ar B
$''
for all rules~\eqref{rule} in $\Pi$ such that $I\models N$. 
The Herbrand interpretation $I$ is called a {\em (deterministic) stable model} of~$\Pi$ if $I$ is a minimal Herbrand model of $\Pi^I$. 
(Minimality is understood in terms of set inclusion. We identify an Herbrand
interpretation with the set of atoms that are true in it.)
%

The definition is extended to any non-ground program~$\Pi$ by identifying it with $gr_\sigma[\Pi]$, the ground program obtained from~$\Pi$ by replacing every variable with every ground term of~$\sigma$.

\section{Language $\lpmln$} \label{sec:lpmln}

\subsection{Syntax of $\lpmln$}\label{ssec:lpmln-syntax}

The syntax of $\lpmln$ defines a set of weighted rules. More precisely, an $\lpmln$ program $\Pi$ is a finite set of weighted rules $w: R$, where $R$ is a rule of the form~\eqref{rule} and $w$ is either a real number or the symbol $\alpha$ denoting the ``infinite weight.''  We call rule $w:R$ {\em soft} rule if $w$ is a real number, and {\em hard} rule if $w$ is $\alpha$. 

We say that an $\lpmln$ program is {\sl ground} if its rules contain no variables. We identify any $\lpmln$ program~$\Pi$ of signature~$\sigma$ with a ground $\lpmln$ program $gr_\sigma[\Pi]$, whose rules are obtained from the rules of $\Pi$ by replacing every variable with every ground term of~$\sigma$. 
The weight of a ground rule in $gr_\sigma[\Pi]$ is the same as the weight of the rule in~$\Pi$ from which the ground rule is obtained.
By $\o{\Pi}$ we denote the unweighted logic program obtained from $\Pi$, i.e., $\o{\Pi}=\{R\ \mid\ w:R\;\in\;\Pi\}$.

\subsection{Semantics of $\lpmln$} \label{ssec:lpmln-semantics}

A model of a Markov Logic Network (MLN) does not have to satisfy all formulas in the MLN.  For each model, there is a unique maximal subset of the formulas that are satisfied by the model, and the weights of the formulas in that subset determine the probability of the model.

Likewise, a stable model of an $\lpmln$ program does not have to be obtained from the whole program. Instead, each stable model is obtained from some subset of the program, and the weights of the rules in that subset determine the probability of the stable model.
Unlike MLNs, it may not seem obvious if there is a {\sl unique} maximal subset that derives such a stable model.
The following proposition tells us that this is indeed the case, and furthermore that the subset is exactly the set of all rules that are satisfied by~$I$. 

\begin{prop}\label{prop:sm_subset}\optional{prop:sm_subset}
For any (unweighted) logic program~$\Pi$ and any subset~$\Pi'$ of~$\Pi$, if $I$ is a stable model of~$\Pi'$ and $I$ satisfies~$\Pi$, then $I$ is a stable model of~$\Pi$ as well.
\end{prop}

The proposition tells us that if $I$ is a stable model of a program, adding more rules to this program does not affect that $I$ is a stable model of the resulting program as long as $I$ satisfies the rules added. On the other hand, it is clear that $I$ is no longer a stable model if $I$ does not satisfy at least one of the rules added.

For any $\lpmln$ program $\Pi$,  by ${\Pi}_I$ we denote the set of rules $w: R$ in $\Pi$ such that $I\models R$, and 
by $\sm[\Pi]$ we denote the set $\{I \mid \text{$I$ is a stable model of $\o{\Pi_I}$}\}.$ 
We define the {\em unnormalized weight} of an interpretation $I$ under $\Pi$, denoted $W_\Pi(I)$, as 
\[
 W_\Pi(I) =
\begin{cases}
  exp\Bigg(\sum\limits_{w:R\;\in\; {\Pi}_I} w\Bigg) & 
      \text{if $I\in\sm[\Pi]$}; \\
  0 & \text{otherwise}. 
\end{cases}
\]
Notice that $\sm[\Pi]$ is never empty because it always contains~$\emptyset$. It is easy to check that $\emptyset$ always satisfies $\o{{\Pi}_\emptyset}$, and it is the smallest set that satisfies the reduct $(\o{{\Pi}_\emptyset})^\emptyset$.

The {\em normalized weight} of an interpretation $I$ under~$\Pi$, denoted $P_\Pi(I)$, is defined as  
\[ 
  P_\Pi(I)  = 
  \lim\limits_{\alpha\to\infty} \frac{W_\Pi(I)}{\sum_{J\in {\rm SM}[\Pi]}{W_\Pi(J)}}. 
\] 

It is easy to check that normalized weights satisfy the Kolmogorov axioms of probability. So we also call them {\em probabilities}.

We omit the subscript $\Pi$ if the context is clear. 
We say that $I$ is a {\sl (probabilistic) stable model} of $\Pi$ if $P_\Pi(I)\ne 0$.

The intuition here is similar to that of Markov Logic. For each interpretation $I$, we try to find a maximal subset (possibly empty) of~$\o{\Pi}$ for which $I$ is a stable model (under the standard stable model semantics). 
In other words, the $\lpmln$ semantics is similar to the MLN semantics except that the possible worlds are the {\em stable} models of some maximal subset of~$\o{\Pi}$, and the probability distribution is over these stable models. Intuitively, $P_\Pi(I)$ indicates how likely to draw $I$ as a stable model of some maximal subset of $\o{\Pi}$.

For any proposition $A$, $P_\Pi(A)$ is defined as 
\begin{align*}
 P_\Pi(A) = \sum_{I :\  I\models A} P_\Pi(I). 
\end{align*}


\NC{
\cmag
Due to the difference between stable models and classical models, there are subtlety here. 
When $\Pi$ is empty, the only information we can derive is the empty set. Hence $P_\Pi(p\lor \neg p)$ is 1. $P_\Pi(p\lor \sneg p)$ is 0 because the only soft stable model $\emptyset$ does not satisfy $p\lor \neg p$. [shall we move it to later?]
}


Conditional probability under $\Pi$ is defined as usual. For propositions $A$ and $B$,
\begin{align*}
 P_\Pi(A \mid B) =  \frac{P_\Pi(A\land B)}{P_\Pi(B)}.
\end{align*}

Often we are interested in stable models that satisfy all hard rules (hard rules encode definite knowledge), in which case the probabilities of stable models can be computed from the weights of the soft rules only, as described below.

%

For any $\lpmln$ program $\Pi$, by $\Pi^{\rm soft}$ we denote the set of all soft rules in~$\Pi$, and by $\Pi^{\rm hard}$ the set of all hard rules in~$\Pi$.
Let $\sm'[\Pi]$ be the set 
$$\{I \mid \text{$I$ is a stable model of $\o{\Pi_I}$ that satisfy 
$\o{{\Pi^{\rm hard}}}$ }\},$$
and let
\[
 W'_\Pi(I) =
\begin{cases}
  exp\Bigg(\sum\limits_{w:R\;\in\; (\Pi^{\rm soft})_I} w\Bigg) & 
      \text{if $I\in\sm'[\Pi]$}; \\
  0 & \text{otherwise},
\end{cases}
\]
\[
P'_\Pi(I) = 
     \frac{W'_{\Pi}(I)}
         {\sum_{J\in{\rm SM}'[\Pi]}{W'_{\Pi}(J)}}.
\]
Notice the absence of $\lim\limits_{\alpha\to\infty}$ in the definition of $P'_\Pi[I]$. Also, unlike $P_\Pi(I)$, $\sm'[\Pi]$ may be empty, in which case $P'_\Pi(I)$ is not defined. Otherwise, the following proposition tells us that the probability of an interpretation can be computed by considering the weights of the soft rules only.

\begin{prop}\label{prop:soft}
If $\sm'[\Pi]$ is not empty, for every interpretation~$I$, $P'_\Pi(I)$ coincides with $P_\Pi(I)$.
\end{prop}

It follows from this proposition that if $\sm'[\Pi]$ is not empty, 
then every stable model of~$\Pi$ (with non-zero probability) should satisfy all hard rules in $\Pi$. 
\NC{
\cmag Note that Example~\ref{ex:inconsistency} does not satisfy the nonemptiness condition of $\sm'[\Pi]$, whereas the following example does. 
}


\subsection{Examples} 

The weight scheme of $\lpmln$ provides a simple but effective way to resolve certain inconsistencies in ASP programs.

\NB{
1. with alpha weight for KB2 and KB3

2. certainty factor 2 and 1. 

3. with only kb2 

[[ here it's better tod o with mvpf with UEC ]]

[[ add default rule?  0.001 {bird(x)=false}
Pr[Bird(x) = false ] = 0
Pr[Bird(x)= 
}


\begin{example}\label{ex:inconsistency}
The example in the introduction can be represented in $\lpmln$ as 
{\small
\[
\ba {lllr}
\i{KB}_1\ \ \ \ & \alpha: & \i{Bird}(x) \ar \i{ResidentBird}(x) & (r1) \\
     & \alpha: & \i{Bird}(x) \ar \i{MigratoryBird}(x)  & (r2)\\
     & \alpha: &   \ar \i{ResidentBird}(x), \i{MigratoryBird}(x) & (r3) \\[0.3em]
\i{KB}_2\ \ \ \ &  \alpha: & \i{ResidentBird}(\i{Jo}) & (r4) \\ [0.3em]
\i{KB}_3\ \ \ \ &  \alpha: & \i{MigratoryBird}(\i{Jo}) & (r5)
\ea
\]
}

Assuming that the Herbrand universe is $\{Jo\}$, the following table shows the weight and the probability of each interpretation.  

%
{\scriptsize
\begin{tabular}{| c | c | c | c | }
\hline
I & $\Pi_I$ & $W_\Pi(I)$ & $P_\Pi(I)$  \\ \hline
$\emptyset$ & $\{r_1,r_2,r_3\}$ & $e^{3\alpha}$ 
   & $0$  \\
$\{R(\i{Jo})\}$ & $\{r_2,r_3,r_4\}$ & $e^{3\alpha}$ 
  & $0$ \\
$\{M(\i{Jo})\}$ & $\{r_1,r_3,r_5\}$ & $e^{3\alpha}$ 
  & $0$ \\
$\{B(\i{Jo})\}$ & $\{r_1,r_2,r_3\}$ & $0$ 
  & $0$ \\
$\{R(\i{Jo}),B(\i{Jo})\}$ & $\{r_1,r_2,r_3,r_4\}$ & $e^{4\alpha}$ 
  & $1/3$ \\
$\{M(\i{Jo}),B(\i{Jo})\}$ & $\{r_1,r_2,r_3,r_5\}$ & $e^{4\alpha}$ 
  & $1/3$ \\
$\{R(\i{Jo}),M(\i{Jo})\}$ & $\{r_4,r_5\}$ & $e^{2\alpha}$
  & $0$ \\
$\{R(\i{Jo}),M(\i{Jo}),B(\i{Jo})\}$ & $\{r_1,r_2,r_4,r_5\}$ & $e^{4\alpha}$ 
  & $1/3$ \\
\hline
\end{tabular}
}

\smallskip\noindent
(The weight of $I=\{\i{Bird}(\i{Jo})\}$ is 0 because $I$ is not a stable model of $\o{\Pi_I}$.)
Thus we can check that 
\bi

\item $P(\i{Bird}(\i{Jo})) = 1/3 + 1/3 + 1/3 = 1 $.

\item  $P(\i{Bird}(\i{Jo}) \mid \i{ResidentBird}(\i{Jo})) = 1$.

\item  $P(\i{ResidentBird}(\i{Jo})\mid \i{Bird}(\i{Jo})) = 2/3$.
\ei



Instead of $\alpha$, one can assign different certainty levels to the additional knowledge bases, such as
{
\[
\ba {lllr}
  \i{KB}'_2 &  2: & \i{ResidentBird}(\i{Jo})   & (r4') \\[0.5em]
  \i{KB}'_3 &  1: & \i{MigratoryBird}(\i{Jo})  & (r5')
\ea 
\] 
}

Then the table changes as follows. \\[0.5em]
%
{\scriptsize 
\begin{tabular}{| c | c | c | c | }
\hline
I & $\Pi_I$ & $W_\Pi(I)$ & $P_\Pi(I)$  \\ \hline
$\emptyset$ & $\{r_1,r_2,r_3\}$ & $e^{3\alpha}$ 
   & $\frac{e^0}{e^2+e^1+e^0}$  \\
$\{R(\i{Jo})\}$ & $\{r_2,r_3,r_4'\}$ & $e^{2\alpha+2}$ 
  & $0$ \\
$\{M(\i{Jo})\}$ & $\{r_1,r_3,r_5'\}$ & $e^{2\alpha+1}$ 
  & $0$ \\
$\{B(\i{Jo})\}$ & $\{r_1,r_2,r_3\}$ & $0$ 
  & $0$ \\
$\{R(\i{Jo}),B(\i{Jo})\}$ & $\{r_1,r_2,r_3,r_4'\}$ & $e^{3\alpha+2}$ 
  & $\frac{e^2}{e^2+e^1+e^0}$ \\
$\{M(\i{Jo}),B(\i{Jo})\}$ & $\{r_1,r_2,r_3,r_5'\}$ & $e^{3\alpha+1}$ 
  & $\frac{e^1}{e^2+e^1+e^0}$ \\
$\{R(\i{Jo}),M(\i{Jo})\}$ & $\{r_4',r_5'\}$ & $e^{3}$
  & $0$ \\
$\{R(\i{Jo}),M(\i{Jo}),B(\i{Jo})\}$ & $\{r_1,r_2,r_4',r_5'\}$ & $e^{2\alpha+3}$ 
  & $0$ \\
\hline
\end{tabular}
}

\smallskip\noindent
$P(\i{Bird}(\i{Jo})) = (e^2+e^1) / (e^2+e^1+e^0) =  0.67+0.24$, so it becomes less certain, though it is still a high chance that we can conclude that $\i{Jo}$ is a Bird.

Notice that the weight changes not only affect the probability, but also the stable models (having non-zero probabilities) themselves: Instead of $\{R(\i{Jo}),M(\i{Jo}),B(\i{Jo})\}$, the empty set is a stable model of the new program.
\end{example}

Assigning a different certainty level to each rule affects the probability associated with each stable model, representing how certain we can derive the stable model from the knowledge base. This could be useful as more incoming data reinforces the certainty levels of the information.


\BOC
\begin{example}[$\lpmln$ vs. MLN]
Consider a variant of the main example from \cite{bauters10possibilistic}. We are certain that we booked a concert and that we have a long drive ahead of us unless the concert is cancelled. However, there is a 20\% chance that the concert is indeed cancelled. This example can be formalized in $\lpmln$ program $\P$ as 
\[
\ba {rrcl}
\alpha: & \i{ConcertBooked} &\ar&  \\
\alpha: & \i{LongDrive} &\ar& \i{ConcertBooked}, \no\ \i{Cancelled} \\
ln\ 0.2: & \i{Cancelled} &\ar& \\
ln\ 0.8: &   & \ar& \i{Cancelled}.
\ea
\]

Since $\sm'[\P]$ is not empty, in view of Proposition~\ref{prop:soft}, the probability of the two stable models are as follows: 
\bi
\item $I_1 = \{\i{ConcertBooked}, \i{Cancelled}\}$, with 
      $\i{Pr}_\P[I_1] = \frac{e^{ln 0.2}}{e^{ln 0.2}+e^{ln 0.8}} = 0.2$.
\item $I_2 = \{\i{ConcertBooked}, \i{LongDrive}\}$, with 
      $\i{Pr}_\P[I_2] = \frac{e^{ln 0.8}}{e^{ln 0.2}+e^{ln 0.8}} = 0.8$.
\ei

If this program is understood under the MLN semantics, say in the syntax
\[
\ba {rl}
\alpha: & \i{ConcertBooked} \\
\alpha: & \i{ConcertBooked} \land \neg \i{Cancelled} \rar \i{LongDrive} \\
ln\ 0.2: & \i{Cancelled} \\
ln\ 0.8: & \neg \i{Cancelled},
\ea
\]
there are three MLN models with non-zero probabilities:
\bi
\item $I_1 = \{\i{ConcertBooked}, \i{Cancelled}\}$ with $\i{Pr}[I_1] = 0.2 / 1.4 \simeq 0.1429$. 
\item $I_2 = \{\i{ConcertBooked}, \i{LongDrive}\}$ with $\i{Pr}[I_2] = 0.8 / 1.4 \simeq 0.5714$.
\item $I_3 = \{\i{ConcertBooked}, \i{Cancelled}, \i{LongDrive}\}$ with $\i{Pr}[I_3] = 0.2/1.4 \simeq 0.1429$.
\ei
The presence of $I_3$ is not intuitive (why have a long drive when the concert is cancelled?)
\EOC

\medskip\noindent
{\bf Remark.} In some sense, the distinction between soft rules and hard rules in $\lpmln$ is similar to the distinction CR-Prolog~\cite{balduccini03logic} makes between consistency-restoring rules (CR-rules) and standard ASP rules: some CR-rules are added to the standard ASP program part until the resulting program has a stable model. 
On the other hand, CR-Prolog has little to say when the ASP program has no stable models no matter what CR-rules are added ({\bf c.f.} Example~\ref{ex:inconsistency}).

\begin{example}\label{ex:friends}
``Markov Logic has the drawback that it cannot express (non-ground) inductive definitions'' \cite{fierens15inference} because it relies on classical models.
This is not the case with $\lpmln$. 
For instance, consider that $x$ may influence $y$ if $x$ is a friend to $y$, and the influence relation is a minimal relation that is closed under transitivity.
\[
\ba {l}
\alpha: \i{Friend}(A,B) \\ 
\alpha: \i{Friend}(B,C) \\
1:  \i{Influence}(x,y) \ar \i{Friend}(x,y) \\
\alpha: \i{Influence}(x,y) \ar \i{Influence}(x,z), \i{Influence}(z,y). 
\ea 
\]
Note that the third rule is soft: a person does not necessarily influence his/her friend. The fourth rule says if $x$ influences $z$, and $z$ influences $y$, we can say $x$ influences $y$. On the other hand, we do not want this relation to be vacuously true.

Assuming that there are only three people $A$, $B$, $C$ in the domain (thus there are $1+1+9+27$ ground rules), there are four stable models with non-zero probabilities. Let $Z = e^{9} + 2e^{8} + e^{7}$. ($\i{Fr}$ abbreviates for $\i{Friend}$ and $\i{Inf}$ for $\i{Influence}$)

\begin{itemize}
\item $I_1 = \{\i{Fr}(A,B), \i{Fr}(B,C), \i{Inf}(A,B), \i{Inf}(B,C),\\ \i{Inf}(A,C)\}$ with probability $e^{9} / Z$.

\item $I_2 = \{\i{Fr}(A,B), \i{Fr}(B,C), \i{Inf}(A,B)\}$ with probability $e^{8} / Z$.

\item $I_3 = \{\i{Fr}(A,B), \i{Fr}(B,C),  \i{Inf}(B,C)\}$ with probability $e^{8} / Z$.

\item $I_4 = \{\i{Fr}(A,B), \i{Fr}(B,C) \}$ with probability $e^{7} / Z$.
\end{itemize}

Thus we get 
\bi
\item 
$P(\i{Inf}(A,B)) = 
 P(\i{Inf}(B,C)) = (e^{9} + e^{8}) / Z = 0.7311$.

\item
$P(\i{Inf}(A,C)) = e^{9} / Z = 0.5344$.
\ei

Increasing the weight of the third rule yields higher probabilities for deriving  $\i{Influence}(A,B)$, $\i{Influence}(B,C)$, and $\i{Influence}(A,C)$. Still, the first two have the same probability, and the third has less probability than the first two. 
\end{example} 


\section{Relating $\lpmln$ to ASP}  \label{sec:relation-asp-mln}

Any logic program under the stable model semantics can be turned into an $\lpmln$ program by assigning the infinite weight to every rule.  That is, for any logic program $\Pi=\left\{R_1, \dots, R_n\right\}$, the corresponding $\lpmln$ program $\P_{\Pi}$ is 
$\left\{\alpha: R_1,\dots, \alpha:R_n\right\}$.

\begin{thm}\label{thm:lpmln-asp}\optional{thm:lpmln-asp}
For any logic program $\Pi$, the (deterministic) stable models of $\Pi$ are exactly the (probabilistic) stable models of $\P_\Pi$ whose weight is $e^{k\alpha}$, where $k$ is the number of all (ground) rules in $\Pi$. If $\Pi$ has at least one stable model, then all stable models of $\P_\Pi$ have the same probability, and are thus the stable models of $\Pi$ as well.
\end{thm}



\subsection{Weak Constraints and $\lpmln$} \label{ssec:weak} 

The idea of softening rules in $\lpmln$ is similar to the idea of {\em weak constraints} in ASP, which is used for certain optimization problems. 
A weak constraint has the form 
``$
  {\tt :\sim}\ \i{Body} \ \ [\i{Weight}:\i{Level}].
$''
The stable models of a program $\Pi$ (whose rules have the form \eqref{rule}) plus a set of weak constraints are the stable models of~$\Pi$ with the minimum penalty, where a penalty is calculated from \i{Weight} and \i{Level} of violated weak constraints. 
%


Since levels can be compiled into weights \cite{buccafurri00enhancing}, we consider weak constraints of the form 
\beq
  {\tt :\sim}\ \i{Body} \ \ [\i{Weight}]
\eeq{weak}
where $\i{Weight}$ is a positive integer.  We assume all weak constraints are grounded. 
The penalty of a stable model is defined as the sum of the weights of all weak constraints whose bodies are satisfied by the stable model.

Such a program can be turned into an $\lpmln$ program as follows.
Each weak constraint  \eqref{weak} is turned into
\[
  -w:\ \ \bot\ar\neg\i{Body}.
\]
The standard ASP rules are identified with hard rules in $\lpmln$.
For example, the program with weak constraints
\[
\ba {lll}
 a \lor b  & \hspace{1cm} &   :\sim a\ \ [1] \\
 c\ar b    &     & :\sim b\ \ [1] \\
           &     & :\sim c\ \ [1] 
\ea
\]
is turned into
\[
\ba {rccrc}
 \alpha: & a \lor b  & \hspace{1cm} & -1: & \bot\ar \neg a \\
 \alpha: & c \ar b   & \hspace{1cm} & -1: & \bot\ar \neg b  \\
         &           &              & -1: & \bot\ar \neg c .  \\
\ea
\]
The $\lpmln$ program has two stable models: $\{a\}$ with the normalized weight~$\frac{e^{-1}}{e^{-1}+e^{-2}}$ and $\{b,c\}$ with the normalized weight~$\frac{e^{-2}}{e^{-1}+e^{-2}}$.
The former, with the larger normalized weight, is the stable model of the original program containing the weak constraints.

\begin{prop}\label{prop: weak2lpmln}
For any program with weak constraints that has a stable model, its stable models are the same as the stable models of the corresponding $\lpmln$ program with the highest normalized weight. 
\end{prop}

\section{Relating $\lpmln$ to MLNs}

\subsection{Embedding MLNs in $\lpmln$} \label{ssec:relation:mln2lpmln}

Similar to the way that SAT can be embedded in ASP, 
Markov Logic can be easily embedded in $\lpmln$. More precisely, any MLN $\L$ 
can be turned into an $\lpmln$ program $\Pi_\L$ so that the models of $\L$ coincide with the stable models of $\Pi_\L$ while retaining the same probability distribution. 

$\lpmln$ program $\Pi_\L$ is obtained from $\L$ by turning each weighted formula $w:F$ into weighted rule \hbox{$w:\ \  \bot\ar\neg F$} and adding 
\[ 
   w:\ \  \{A\}^{\rm ch}  
\] 
for every ground atom $A$ of $\sigma$ and any weight $w$.
The effect of adding the choice rules is to exempt $A$ from minimization under the stable model semantics. 

\begin{thm}\label{thm:mln-lpmln}\optional{thm:mln-lpmln}
Any MLN $\L$ and its $\lpmln$ representation $\Pi_\L$ have the same probability distribution over all interpretations.
\end{thm}

The embedding tells us that the exact inference in $\lpmln$ is at least as hard as the one in MLNs, which is $\#$P-hard. In fact, it is easy to see that when all rules in $\lpmln$ are non-disjunctive, counting the stable models of $\lpmln$ is in $\#$P, which yields that the exact inference for non-disjunctive $\lpmln$ programs is $\#$P-complete. Therefore, approximation algorithms, such as Gibbs sampling, may be desirable for computing large $\lpmln$ programs. The next section tells us that we can apply the MLN approximation algorithms to computing $\lpmln$ based on the reduction of the latter to the former. 



\subsection{Completion: Turning $\lpmln$ to MLN}\label{ssec:completion}

It is known that the stable models of a tight logic program coincide with the models of the program's completion \cite{erd03}. This yielded a way to compute stable models using SAT solvers. The method can be extended to $\lpmln$ so that probability queries involving the stable models can be computed using existing implementations of MLNs, such as Alchemy (\url{http://alchemy.cs.washington.edu}).


We define the {\em completion} of $\Pi$, denoted $\i{Comp}(\Pi)$, to be the MLN which is the union of $\Pi$ and the hard formula 

{\small
\[
  \alpha:\   
     A\rar\bigvee\limits_{w: A_1\lor\dots\lor A_k\ar \mi{Body} \in\ \Pi \atop A\in\{A_1,\dots, A_k\}}
    \Big(Body\land \bigwedge\limits_{A'\in\{A_1,\dots, A_k\}\setminus \{A\}} \neg A'\Big)
\]
}
for each ground atom $A$.

This is a straightforward extension of the completion from~\cite{lee03a} by simply assigning the infinite weight $\alpha$ to the completion formulas. Likewise, we say that $\lpmln$ program $\Pi$ is {\em tight} if $\o{\Pi}$ is tight according to \cite{lee03a}, i.e., the positive dependency graph of $\o{\Pi}$ is acyclic.

\begin{thm}\label{thm:lpmln-lf}\optional{thm:lpmln-lf}
For any tight $\lpmln$ program $\Pi$ such that $\sm'[\Pi]$ is not empty, 
$\Pi$ (under the $\lpmln$ semantics) and $\i{Comp}(\Pi)$ (under the MLN semantics) have the same probability distribution over
all interpretations.
\end{thm}

The theorem can be generalized to non-tight programs by considering loop formulas \cite{lin04}, which we skip here for brevity.

\section{Relation to ProbLog} \label{sec:relation-problog}


It turns out that $\lpmln$ is a proper generalization of ProbLog, a well-developed probabilistic logic programming language that is based on the distribution semantics by Sato~(\citeyear{sato95astatistical}).

\subsection{Review: ProbLog} \label{ssec:review-problog}

The review follows \cite{fierens15inference}.
As before, we identify a non-ground ProbLog program with its ground
instance. So for simplicity we restrict attention to ground ProbLog
programs only. 

In ProbLog, ground atoms over $\sigma$ are divided into two groups: {\em probabilistic} 
atoms and {\em derived} atoms. 
A \emph{(ground) ProbLog program} $\P$ is a tuple 
$\langle \i{PF}, \Pi\rangle$, where
\begin{itemize}
\item $\i{PF}$ is a set of ground probabilistic facts of the form
\hbox{$pr :: a$}, 
\item $\Pi$ is a set of ground rules of the following form
\[
  A\ar B_1, \dots, B_m, not\ B_{m+1}, \dots, not\ B_n
\]
where $A$, $B_1,\dots B_n$ are atoms from $\sigma$ ($0\le m\le n$), and $A$ is not a probabilistic atom.
\end{itemize}

Probabilistic atoms act as random variables and are assumed to be
independent from each other. 
A \emph{total choice} $\i{TC}$ is any subset of the probabilistic
atoms. Given a total choice $\i{TC}=\left\{a_1, \dots, a_m\right\}$, the {\em probability} of a total choice $\i{TC}$, denoted $\i{Pr}_\P(\i{TC})$, is defined as
\[ 
\ba l
  pr(a_1)\!\times\!\dots\!\times\! pr(a_m)\!\times\! 
     (1\!-\!pr(b_1))\!\times\!\dots\!\times\! (1\!-\!pr(b_n))
\ea
\] 
where $b_1, \dots, b_n$ are probabilistic atoms not belonging to~$\i{TC}$,
and each of $pr(a_i)$ and $pr(b_j)$ is the probability assigned to $a_i$ and
$b_j$ according to the set $\i{PF}$ of ground probabilistic atoms.


The ProbLog semantics is only well-defined for programs $\P=\langle
\i{PF},\Pi\rangle$ such that $\Pi\cup \i{TC}$ has a ``total'' (two-valued) well-founded model for each total choice~$\i{TC}$. 
Given such $\P$, the probability of an interpretation $I$, denoted $P_\P(I)$, is defined as $\i{Pr}_\P(\i{TC})$ if there exists a total
choice $\i{TC}$ such that $I$ is the total well-founded model of
$\Pi\cup \i{TC}$, and $0$ otherwise.  


\BOCC
\begin{thm}\label{lem: wfm_sm}
When $\Pi$ has a total well-founded model, then this model is also the single stable model of $\Pi$; conversely, if $\Pi$ has a single stable model, this model is also the well-founded model of $\Pi$.
\end{thm}

\begin{thm}
For an arbitrary ProbLog program $\P=\langle F, R\rangle$ for which
\begin{itemize}
\item every total choice $C$ leads to one and only one stable model, or equivalently (by Theorem \ref{lem: wfm_sm}),
\item every total choice $C$ leads to a total well-founded model,
\end{itemize}
 $\P$ in the sense of this section gives the same probability distribution as $P$ in the sense of \cite{fierens15inference}. 
\end{thm}
\EOCC

\BOCCC
\begin{example}
Consider the ProbLog program:
\[ 
\ba {lll}
0.6\ ::\ p  & \hspace{1cm} & r\ar p \\
0.4\ ::\ q  &  & r\ar q 
\ea 
\]
For each of the 4 total choices, the probability distribution over
ProbLog models is as follows. 

\begin{center}
{\small
\begin{tabular}{| l | l | c |}
\hline
 {Total choice} &  {ProbLog model} & {Probability} \\  \hline
 $\emptyset$ &	$\emptyset$ &  $0.24$ \\
 $\left\{p\right\}$ & $\left\{p, r\right\}$ &  $0.36$ \\
 $\left\{q\right\}$ & $\left\{q, r\right\}$ &  $0.16$ \\
 $\left\{p, q\right\}$ & $\left\{p, q, r\right\}$ &  $0.24$ \\	\hline
\end{tabular}
}
\end{center}

\end{example}
\EOCCC

\subsection{ProbLog as a Special Case of $\lpmln$} \label{ssec:problog-lpmln}



Given a ProbLog program $\P =\langle \i{PF}, \Pi\rangle$, we construct
the corresponding $\lpmln$ program ${\P'}$ as follows:
\begin{itemize}
\item For each probabilistic fact $pr :: a$ in $\P$, 
  $\lpmln$ program $\P'$ contains
(i) $ln(pr): a$ and $ln(1\!-\!pr):\ \ \ar a$\ \  if $0< pr < 1$;
(ii) $\alpha: a$\ \  if $pr=1$;
(iii) $\alpha:\ \ \ar a$\ \  if $pr = 0$.

\item For each rule $R\in \Pi$, $\P'$ contains 
  $\alpha: R$. In other words, $R$ is identified with a hard rule in
  $\P'$.
\end{itemize}


\begin{thm}\label{thm:lpmln-problog}\optional{thm:lpmln-problog}
Any well-defined ProbLog program $\P$ and its $\lpmln$ representation $\P'$ have the same probability distribution over all interpretations. 
\end{thm}

\begin{example}
Consider the ProbLog program
\[ 
\ba {lll}
0.6\ ::\ p  & \hspace{2cm} & r\ar p \\
0.3\ ::\ q  &  & r\ar q 
\ea 
\]
which can be identified with the $\lpmln$ program 
\[
\ba {llllll}
  ln(0.6):\ \ p  &  & ln(0.3):\ \ q & & \alpha:\ \ r\ar p \\
	ln(0.4):\ \ \ar p & & ln(0.7):\ \ \ar q & & \alpha:\ \ r\ar q 
\ea
\]
\end{example}

Syntactically, $\lpmln$ allows more general rules than ProbLog, such as disjunctions in the head, as well as the empty head and double negations in the body. Further, $\lpmln$ allows rules to be weighted as well as facts, and do not distinguish between probabilistic facts and derived atoms. 
Semantically, while the ProbLog semantics is based on  well-founded models, $\lpmln$ handles stable model reasoning for more general classes of programs. Unlike ProbLog which is only well-defined when each total choice leads to a unique well-founded model, $\lpmln$ can handle multiple stable models in a flexible way similar to the way MLN handles multiple models. 

\section{Multi-Valued Probabilistic Programs}

In this section we define a simple fragment of $\lpmln$ that allows us to represent probability in a more natural way. For simplicity of the presentation, we will assume a propositional signature. An extension to first-order signatures is straightforward. 

We assume that the propositional signature $\sigma$ is constructed from ``constants'' and their ``values.'' 
A {\em constant} $c$ is a symbol that is associated with a finite set $\i{Dom}(c)$, called the {\em domain}. 
The signature $\sigma$ is constructed from a finite set of constants, consisting of atoms $c\!=\!v$~\footnote{%
Note that here ``='' is just a part of the symbol for propositional atoms, and is not  equality in first-order logic. }
for every constant $c$ and every element $v$ in $\i{Dom}(c)$.
If the domain of~$c$ is $\{\false,\true\}$ then we say that~$c$ is {\em Boolean}, and abbreviate $c\mvis\true$ as $c$ and $c\mvis\false$ as~$\sneg c$. 

We assume that constants are divided into {\em probabilistic} constants and {\em regular} constants.
A multi-valued probabilistic program ${\bf \Pi}$ is a tuple $\langle \i{PF}, \Pi \rangle$, where
\begin{itemize}
\item $\i{PF}$ contains \emph{probabilistic constant declarations} of the following form:
\begin{equation}\label{eq:probabilistic-constant-declaration}
p_1: c\mvis v_1\mid\dots\mid p_n: c\mvis v_n
\end{equation}
one for each probabilistic constant $c$, where $\{v_1,\dots, v_n\}=\i{Dom}(c)$, $v_i\ne v_j$, $0\leq p_1,\dots,p_n\leq1$ and $\sum_{i=1}^{n}p_i=1$. We use $M_{\bf \Pi}(c=v_i)$ to denote $p_i$.
In other words, $\i{PF}$ describes the probability distribution over each ``random variable''~$c$. 

\item $\Pi$ is a set of rules of the form~\eqref{rule} such that $A$ contains no probabilistic constants.
\end{itemize}

The semantics of such a program ${\bf \Pi}$ is defined as a shorthand for $\lpmln$ program $T({\bf \Pi})$ of the same signature as follows.
\begin{itemize}
\item For each probabilistic constant declaration (\ref{eq:probabilistic-constant-declaration}), $T({\bf \Pi})$ contains, 
for each $i=1,\dots, n$,
(i) $ln(p_i):  c\mvis v_i$  if $0<p_i<1$; 
(ii) $\alpha:\ c\mvis v_i$ if $p_i=1$;
(iii) $\alpha:\ \ar c\mvis v_i$ if $p_i=0$.

\item  For each rule in $\Pi$ of form \eqref{rule}, $T({\bf \Pi})$ contains
\[ 
\alpha:\ \  A\ar B, N.
\] 

\ii For each constant $c$, $T({\bf \Pi})$ contains the uniqueness of value constraints
\beq
\ba {rl}
   \alpha: & \bot \ar c\mvis v_1\land c=v_2 
\ea 
\eeq{uc}
for all $v_1,v_2 \in\i{Dom}(c)$ such that $v_1\ne v_2$.
For each probabilistic constant $c$, $T({\bf \Pi})$ also contains the existence of value constraint
\beq
\ba {rl}
  \alpha: & \bot \ar \neg \bigvee\limits_{v \in \mi{Dom}(c)} c\mvis v\ .
\ea 
\eeq{ec}
This means that a regular constant may be undefined (i.e., have no values associated with it), while a probabilistic constant is always associated with some value.
\end{itemize}

\begin{example}
The multi-valued probabilistic program
{
\[ 
\ba l
  0.25: Outcome\mvis 6\mid 0.15: Outcome\mvis 5 \\
  ~~\mid 0.15: Outcome\mvis 4\mid 0.15: Outcome\mvis 3 \\
  ~~~~\mid 0.15: Outcome\mvis 2\mid 0.15: Outcome\mvis 1\\
  Win \leftarrow Outcome\mvis 6.
\ea
\] 
}
is understood as shorthand for the $\lpmln$ program
{
\[ 
\ba {rl}
 ln(0.25): & Outcome\mvis 6\\
 ln(0.15): & Outcome\mvis i \hspace{2.3cm}  (i=1,\dots, 5) \\
 \alpha:   &  Win \ar Outcome\mvis 6 \\
 \alpha:   & \bot \ar Outcome\mvis i \land Outcome\mvis j \hspace{0.2cm} (i\ne j)  \\
 \alpha:   & \bot \ar \neg\bigvee_{i=1,\dots 6} Outcome\mvis i .
\ea 
\] 
}
\end{example}

\BOCC
\begin{definition}
We say that multi-valued probabilistic program $\Pi$ containing rules~\eqref{eq:probabilistic-constant-declaration} for each probabilistic constant $c$ is {\em coherent} if $P_\Pi(c\mvis v_i) = p_i$ for each constant $c$.
\end{definition}


\begin{example} \label{ex:coherent}
Consider a machine in a theme park which randomly gives out some toy. There is a $60\%$ probability that a stuffed animal comes out from it, and a $30\%$ probability that a lightsaber comes out from it. David and Bob are the two kids who may come to the theme park. David would be happy if the toy from the machine is a stuffed animal, John would be happy if the toy from the machine is a lightsaber. 
This is formalized as the following multi-valued probabilistic program.
\[
\small
\ba c
0.6: \i{StuffedAnimal} \mid 0.4:\ \sneg\i{StuffedAnimal} \\
0.3: \i{Lightsaber} \mid 0.7:\ \sneg\i{Lightsaber}\\
\{\i{Come}(\i{David})\}^{\rm ch}\\
\{\i{Come}(\i{Bob})\}^{\rm ch} \\
\i{Happy}(\i{David}) \ar \i{StuffedAnimal}\land\i{Come}(\i{David}) \\
\i{Happy}(\i{Bob}) \ar \i{Lightsaber}\land\i{Come}(\i{Bob}).
\ea 
\]

There are $2\times 2\times 2\times 2=16$ stable models.
It can be checked that $P(\i{StuffedAnimal})=0.6$ and $P(\i{Lightsaber})=0.3$. 

\BOC
 the probability of each is listed below ($Z=(0.4\times 0.7+0.6\times 0.7 + 0.4\times 0.3 + 0.6\times 0.3)\times 4=4$):

\begin{tabular}{ l c}
  {\bf Stable Model} & {\bf Probability} \\
	$\left\{\sim stuffed\_animal, \sim lightsaber\right\}$ & $\frac{0.4\times 0.7}{Z}=0.07$ \\
  $\left\{\sim stuffed\_animal, \sim lightsaber, come(david)\right\}$ & $\frac{0.4\times 0.7}{Z}=0.07$ \\
	$\left\{\sim stuffed\_animal, \sim lightsaber, come(bob)\right\}$ & $\frac{0.4\times 0.7}{Z}=0.07$ \\
	$\left\{\sim stuffed\_animal, \sim lightsaber, come(david), come(bob)\right\}$ & $\frac{0.4\times 0.7}{Z}0.07$ \\
	$\left\{stuffed\_animal, \sim lightsaber\right\}$ & $\frac{0.6\times 0.7}{Z}=0.105$ \\
  $\left\{stuffed\_animal, \sim lightsaber, come(david), happy(david)\right\}$ & $\frac{0.6\times 0.7}{Z}=0.105$ \\
	$\left\{stuffed\_animal, \sim lightsaber, come(bob)\right\}$ & $\frac{0.6\times 0.7}{Z}=0.105$ \\
	$\left\{stuffed\_animal, \sim lightsaber, come(david), come(bob), happy(david)\right\}$ & $\frac{0.6\times 0.7}{Z}=0.105$\\
	$\left\{\sim stuffed\_animal, lightsaber\right\}$ & $\frac{0.4\times 0.3}{Z}=0.03$ \\
  $\left\{\sim stuffed\_animal, lightsaber, come(david)\right\}$ & $\frac{0.4\times 0.3}{Z}=0.03$ \\
	$\left\{\sim stuffed\_animal, lightsaber, come(bob), happy(bob)\right\}$ & $\frac{0.4\times 0.3}{Z}=0.03$ \\
	$\left\{\sim stuffed\_animal, lightsaber, come(david), come(bob), happy(bob)\right\}$ & $\frac{0.4\times 0.3}{Z}=0.03$ \\
	$\left\{stuffed\_animal, lightsaber\right\}$ & $\frac{0.6\times 0.3}{Z}=0.045$ \\
  $\left\{stuffed\_animal, lightsaber, come(david), happy(david)\right\}$ & $\frac{0.6\times 0.3}{Z}=0.045$ \\
	$\left\{stuffed\_animal, lightsaber, come(bob), happy(bob)\right\}$ & $\frac{0.6\times 0.3}{Z}=0.045$ \\
	$\left\{stuffed\_animal, lightsaber, come(david), come(bob), happy(bob), happy(david)\right\}$ & $\frac{0.6\times 0.3}{Z}=0.045$ \\
\end{tabular}
It can be checked that $Pr\left[stuffed\_animal\right]=0.6$ and $Pr\left[lightsaber\right]=0.3$. This example cannot be handled by ProbLog since each total choice leads to $4$ stable models.
\EOC
\end{example}
\EOCC


\BOC
Given a multi-valued probabilistic program ${\bf \Pi}=\langle PF, \Pi\rangle$ of signature $\sigma$, we say an interpretation $I$ is {\em consistent} if
\begin{itemize}
\item For all $c\in \sigma^{pf}({\bf \Pi})$, $c=v\in I$ for exactly one $v\in Dom(c)$;
\item For all $c$ of $\sigma$ that are not in $\sigma^{pf}({\bf \Pi})$, $c=v\in I$ for at most one $v\in Dom(c)$.
\end{itemize}
\EOC

We say an interpretation of ${\bf \Pi}$ is {\em consistent} if it satisfies the hard rules \eqref{uc} for every constant and \eqref{ec} for every probabilistic constant.
For any consistent interpretation $I$, we define the set $\i{TC}(I)$ (``Total Choice'') to be
$
\left\{c=v \mid \text{$c$ is a probabilistic constant such that $c=v\in I$} \right\}
$
and define
\[
\ba l
\sm''[{\bf \Pi}] = \{I\mid \text{$I$ is consistent} \\
  \hspace{2.5cm} \text{and is a stable model of $\Pi\cup \i{TC}(I)$}\}.
\ea
\]

For any interpretation $I$, we define
\[
W''_{{\bf \Pi}}(I)=\begin{cases}
\prod\limits_{\text{$c\mvis v\in TC(I)$}} M_{{\bf \Pi}}(c=v) & \text{if $I\in \sm''[{\bf \Pi}]$} \\
0 & \text{otherwise}
\end{cases}
\]
and
\[
P''_{{\bf \Pi}}(I)=\frac{W''_{{\bf \Pi}}(I)}{\sum_{J\in\smdbprm\left[{\bf \Pi}\right]} W''_{{\bf \Pi}}(J)}.
\]

The following proposition tells us that the probability of an interpretation can be computed from the probabilities assigned to probabilistic atoms, similar to the way ProbLog is defined.

\begin{prop}\label{thm:pdbprm}
For any multi-valued probabilistic program~${\bf \Pi}$ such that each $p_i$ in \eqref{eq:probabilistic-constant-declaration} is positive for every probabilistic constant $c$, 
if $\sm''[{\bf \Pi}]$ is not empty, then for any interpretation $I$, $P''_{\bf \Pi}(I)$ coincides with $P_{T({\bf \Pi})}(I)$.
\end{prop}

\section{P-log and $\lpmln$} \label{ssec:plog}

\subsection{Simple P-log}  \label{ssec:miniplog}

In this section, we define a fragment of P-log, which we call {\em simple P-log}.

\subsubsection{Syntax}
Let $\sigma$ be a multi-valued propositional signature as in the previous section.
A simple P-log program $\Pi$ is a tuple
\begin{equation}\label{eq:mini-p-log}
\Pi = \langle R, S, P, Obs, Act \rangle
\end{equation}
where
\begin{itemize}
\item  $R$ is a set of normal rules of the form
\beq 
  A\ar B_1, \dots, B_m, \no\ B_{m+1}, \dots, \no\ B_n.
\eeq{eq:p-log-regular-part}
Here and after we assume $A, B_1,\dots, B_n$ are atoms from~$\sigma$ ($0\leq m\leq n$).

\item $S$ is a set of \emph{random selection rules} of the form
\beq 
 [r]\ \ random(c)\ar B_1, \dots, B_m, \no\ B_{m+1}, \dots, \no\ B_n
\eeq{eq:p-log-random-selection}
where $r$ is an identifier and $c$ is a constant.

\item $P$ is a set of \emph{probability atoms (pr-atoms)} of the form
\[ 
  pr_r(c\mvis v\mid B_1, \dots, B_m, \no\ B_{m+1}, \dots, \no\ B_n)=p
\] 
where $r$ is the identifier of some random selection rule in~$S$, $c$ is a constant, and $v\in\i{Dom}(c)$, and $p\in [0, 1]$.

\item $Obs$ is a set of atomic facts of the form
$Obs(c\mvis v)$ where $c$ is a constant and $v\in\i{Dom}(c)$.

\item $Act$ is a set of atomic facts of the form
$Do(c\mvis v)$ where $c$ is a constant and $v\in\i{Dom}(c)$.

\end{itemize}

\begin{example}\label{ex:dice}
We use the following simple P-log program as our main example ($d\in\{D_1, D_2\}$, $y\in\{1,\dots 6\}$): 
{
\[
\ba c
  \i{Owner}(D_1)\mvis \i{Mike}\\
  \i{Owner}(D_2)\mvis \i{John}\\
  \i{Even}(d)\leftarrow \i{Roll}(d)\mvis y,\  y\ mod\ 2=0 \\
  \sneg \i{Even}(d)\leftarrow \no\ \i{Even}(d) \\
  \left[r(d)\right] random(\i{Roll}(d)) \\
 pr(\i{Roll}(d)\mvis 6\mid \i{Owner}(d)\mvis \i{Mike})=\frac{1}{4}.
\ea 
\]
}
\end{example}

\subsubsection{Semantics}

Given a simple P-log program $\Pi$ of the form~\eqref{eq:mini-p-log}, a (standard) ASP program $\tau(\Pi)$ with the multi-valued signature $\sigma^{\prime}$ is constructed as follows: 

\begin{itemize}

\item $\sigma'$ contains all atoms in $\sigma$, and atom $\i{Intervene}(c)\mvis\true$ (abbreviated as $\i{Intervene}(c)$) for every constant $c$ of $\sigma$; the domain of $\i{Intervene}(c)$ is $\{\true\}$.
 
\item $\tau(\Pi)$ contains all rules in $R$. 

\item For each random selection rule of the form (\ref{eq:p-log-random-selection}) with $\i{Dom}(c)=\left\{v_1,\dots, v_n\right\}$, $\tau(\Pi)$ contains the following rules:
\[
\ba l
c\mvis v_1; \dots; c\mvis v_n\ar \\
 B_1, \dots, B_m, \no\ B_{m+1}, \dots, \no\ B_n, not\ \i{Intervene}(c).
\ea
\]

\item $\tau(\Pi)$ contains all atomic facts in $\i{Obs}$ and $\i{Act}$.

\item For every atom $c\mvis v$ in $\sigma$, 
\begin{equation}\nonumber
\ar Obs(c\mvis v), \no\ c\mvis v.
\end{equation}

\item For every atom $c\mvis v$ in $\sigma$, $\tau(\Pi)$ contains
\[ 
\ba c
   c\mvis v\ar Do(c\mvis v) \\
  \i{Intervene}(c) \ar Do(c\mvis v).
\ea
\] 
\end{itemize}

\smallskip\noindent
{\bf Example~\ref{ex:dice} continued}\ \ 
{\sl 
The following is $\tau(\Pi)$ for the simple P-log program $\Pi$ in Example~\ref{ex:dice} ($x\in\{\i{Mike},\i{John}\}$, $b\in\{\true,\false\}$):
{
\[
\ba c
  \i{Owner}(D_1)\mvis \i{Mike}\\
  \i{Owner}(D_2)\mvis \i{John}\\
  \i{Even}(d) \leftarrow \i{Roll}(d)\mvis y,\ y\ mod\ 2=0\\
  \sneg \i{Even}(d)\leftarrow \no\ \i{Even}(d)\\[0.3em]
  \i{Roll}(d)\mvis 1; \i{Roll}(d)\mvis 2; \i{Roll}(d)\mvis 3;
     \i{Roll}(d)\mvis 4; \hspace{0.0cm}  \\
\hspace{0.8cm} \i{Roll}(d)\mvis 5; \i{Roll}(d)\mvis 6 \leftarrow \no\ \i{Intervene}(\i{Roll}(d)) \\[0.3em]
  \leftarrow Obs(\i{Owner}(d)\mvis x), not\ \i{Owner}(d)\mvis x\\
  \leftarrow Obs(\i{Even}(d)\mvis b), not\ \i{Even}(d)\mvis b\\
  \leftarrow Obs(\i{Roll}(d)\mvis y), not\ \i{Roll}(d)\mvis y\\[0.3em]
  \i{Owner}(d)\mvis x \leftarrow Do(\i{Owner}(d)\mvis x)\\
  \i{Even}(d)\mvis b \leftarrow Do(\i{Even}(d)\mvis b)\\
  \i{Roll}(d)\mvis y \leftarrow Do(\i{Roll}(d)\mvis y)\\[0.3em]
  \i{Intervene}(\i{Owner}(d)) \leftarrow Do(\i{Owner}(d)\mvis x)\\
  \i{Intervene}(\i{Even}(d)) \leftarrow Do(\i{Even}(d)\mvis b)\\
  \i{Intervene}(\i{Roll}(d))\leftarrow Do(\i{Roll}(d)\mvis y).
\ea 
\]
}}
\smallskip

The stable models of $\tau(\Pi)$ are called the \emph{possible worlds} of $\Pi$, and denoted by $\omega(\Pi)$. For an interpretation $W$ and an atom $c\mvis v$, we say $c\mvis v$ is \emph{possible} in $W$ with respect to $\Pi$ if $\Pi$ contains a random selection rule for $c$ 
\[ 
  [r]\ \  random(c)\ar B, 
\] 
where $B$ is a set of atoms possibly preceded with $\no$, 
and $W$ satisfies $B$. 
 We say $r$ is {\em applied} in~$W$ if $W\models B$. 

We say that a pr-atom $pr_r(c\mvis v\mid B)=p$ is {\em applied} in $W$ if $W\models B$ and $r$ is applied in $W$. 

As in \cite{baral09probabilistic}, we assume that simple P-log programs $\Pi$ satisfy the following conditions:

\begin{itemize}

\item {\bf Unique random selection rule}\ \ 
For any constant~$c$, program $\Pi$ contains at most one random selection rule for $c$ that is applied in $W$. 


\item {\bf Unique probability assignment}\ \ 
If $\Pi$ contains a random selection rule $r$ for constant $c$ that is applied in $W$, then, for any two different probability atoms 
\[ 
\ba l
  pr_r(c\mvis v_1\mid B') = p_1\\
  pr_r(c\mvis v_2\mid B'') = p_2
\ea 
\]
in $\Pi$ that are applied in $W$, we have $v_1\ne v_2$ and \hbox{$B'=B''$}.

\end{itemize}

\NB{(This is stronger than the P-log assumption)}

Given a simple P-log program $\Pi$, a possible world $W\in\omega(\Pi)$ and a constant $c$ for which $c\mvis v$ is possible in $W$, we first define the following notations: 

\begin{itemize}
\item
Since $c\mvis v$ is possible in $W$, by the unique random selection rule assumption, it follows that there is exactly one random selection rule $r$ for constant $c$ that is applied in $W$. 
Let $r_{W, c}$ denote this random selection rule. 
By the unique probability assignment assumption, if there are pr-atoms of the form 
$pr_{r_{W,c}} (c\mvis v\mid B)$ that are applied in $W$, all $B$ in those pr-atoms should be the same. 
We denote this $B$ by $B_{W,c}$.
Define $\i{PR}_{W}(c)$ as 
\[
\ba l
  \{ pr_{r_{W,c}}(c\mvis v\mid B_{W,c})=p  \in\Pi \mid v\in\i{Dom}(c)\}.
\ea
\]
if $W\not\models\i{Intervene}(c)$ and $\emptyset$ otherwise.

\item  Define $AV_{W}(c)$ as 
$$\left\{v\mid pr_{r_{W,c}}(c\mvis v\mid B_{W,c})=p\in \i{PR}_{W}(c)\right\}.$$

\item For each $v\in AV_{W}(c)$, define the {\em assigned probability} of $c\mvis v$ w.r.t. $W$, denoted by $ap_W(c\mvis v)$, as the value $p$ for which $pr_{r_{W,c}}(c\mvis v\mid B_{W,c})=p\in \i{PR}_{W}(c)$.

\item Define the {\em default probability} for $c$ w.r.t. $W$, denoted by $dp_{W}(c)$, as
\begin{equation}\nonumber
dp_W(c)=\frac{1-\sum_{v\in AV_{W}(c)}ap_W(c\mvis v)}{|\i{Dom}(c)\setminus AV_W(c)|}.
\end{equation}
\end{itemize}

For every possible world $W\in \omega(\Pi)$ and every atom $c\mvis v$ possible in $W$, the causal probability $P(W, c\mvis v)$ is defined as follows: 
\begin{equation}\nonumber
P(W,c\mvis v)=\begin{cases}
ap_W(c\mvis v) & \text{if $v\in AV_W(c)$}\\
dp_W(c)   & \text{otherwise}. 
\end{cases}
\end{equation}

The \emph{unnormalized probability} of a possible world $W$, denoted by $\hat{\mu}_{\Pi}(W)$, is defined as
\[ 
  \hat{\mu}_{\Pi}(W) = \prod_{c=v\in W \text{ and }\atop c=v \text{ is possible in } W}{P(W, c\mvis v)}.
\]
Assuming $\Pi$ has at least one possible world with nonzero unnormalized probability, the \emph{normalized probability} of $W$, denoted by $\mu_{\Pi}(W)$, is defined as
\[ 
  \mu_{\Pi}(W) = \frac{\hat{\mu}_{\Pi}(W)}{\sum_{W_i\in\omega(\Pi)}\hat{\mu}_{\Pi}(W_i)}.
\]


Given a simple P-log program $\Pi$ and a formula $A$, the probability of $A$ with respect to $\Pi$ is defined as
\[ 
  P_{\Pi}(A)=\sum_{\text{$W$ is a possible world of $\Pi$ that satisfies $A$}} \mu_{\Pi}(W).
\]

We say $\Pi$ is {\em consistent} if $\Pi$ has at least one possible world. 
\BOCC
We say $\Pi$ is {\em coherent} if the following conditions are true:
\begin{itemize}
\item $\Pi$ is consistent.
\item Let $\Pi'$ be the simple P-log program obtained from $\Pi$ by removing all observations and actions. For every random selection rule $r$
\[ 
   random(c)\leftarrow B'
\] 
and every pr-atom 
\[
  pr_r(c\mvis v\mid B)=p
\]
in $\Pi$, if $P_{\Pi}(B\cup K)$ is not equal to $0$ then $P_{\Pi'\cup {Obs(B)\cup Obs(K)}}(c\mvis v)=p$.
\end{itemize}
\EOCC

\smallskip\noindent
{\bf Example~\ref{ex:dice} continued}\ \ 
{\sl 
Given the possible world $W=\{\i{Owner}(D_1)\mvis \i{Mike},  \i{Owner}(D_2)\mvis \i{John}, \i{Roll}(D_1)\mvis 6,$ $\i{Roll}(D_2)\mvis 3, \i{Even}(D_1)\}$, the probability of $\i{Roll}(D_1)\mvis 6$ is $P(W, \i{Roll}(D_1)\mvis 6)=0.25$, the  probability of $\i{Roll}(D_2)\mvis 3$ is $\frac{1}{6}$. The unnormalized probability of $W$, i.e., $\hat{\mu}(W)=P(W, \i{Roll}(D_1)\mvis 6)\cdot P(W, \i{Roll}(D_2)\mvis 3)=\frac{1}{24}$.
}
\smallskip

The main differences between simple P-log and P-log are as follows. 
\begin{itemize}
\item  The unique probability assignment assumption in P-log is more general: it does not require the part $B'=B''$. However, all the examples in the P-log paper~\cite{baral09probabilistic} satisfy our stronger unique probability assignment assumption.

\item  P-log allows a more general random selection rule of the form 
\[ 
  \left[ r\right] random(c:\left\{x:P(x)\right\})\leftarrow B'.
\]
Among the examples in~\cite{baral09probabilistic}, only the ``Monty Hall Problem'' encoding and the ``Moving Robot Problem'' encoding use ``dynamic range $\{x: P(x)\}$'' in random selection rules and cannot be represented as simple P-log programs.
\end{itemize}


\subsection{Turning Simple P-log into Multi-Valued Probabilistic Programs} 

\begin{table*}
	\centering
\begin{minipage}{\textwidth}
{\footnotesize
		\begin{tabular}{| c | c | l | l  | l| l| }
			\hline
			\textbf{Example}& \textbf{Parameter} & \textbf{plog1} & \textbf{plog2} & \textbf{Alchemy} (default) & \textbf{Alchemy} (maxstep=5000)\\
			\hline
			 & $N_{dice}=2$ & $0.00s + 0.00s$\footnote {smodels answer set finding time + probability computing time}& $0.00s + 0.00s$\footnote{partial grounding time + probability computing time}& $0.02s + 0.21s$\footnote{mrf creating time + sampling time} & $0.02s + 0.96s$\\
			 & $N_{dice}=7$ & $1.93 s+ 31.37s$ & $0.00s + 1.24s$ & $0.13s + 0.73s$ & $0.12s + 3.39s$\\
			dice & $N_{dice}=8$ & $12.66s + 223.02s$  & $0.00s + 6.41s$ & $0.16s + 0.84s$ & $0.16s + 3.86s$\\
			 & $N_{dice}=9$ & timeout
 & $0.00s + 48.62s$ & $0.19s + 0.95s$ & $0.19s + 4.37s$\\
			 & $N_{dice}=10$ & timeout & timeout & $0.23s +  1.06s$ & $0.24s + 4.88s$\\
			 & $N_{dice}=100$ & timeout & timeout & $19.64s + 16.34 s$ & $19.55  s + 76.18 s$\\
			\hline
			 & $maxstep=5$ & $0.00 s + 0.00 s$ & segment fault & $2.34 s + 2.54 s$  & $2.3 s  + 11.75 s$\\
			 & $maxstep=10$ & $0.37 s+ 4.86 s$& segment fault & $4.78 s + 5.24 s$ & $4.74 s + 24.34 s$\\
			robot & $maxstep=12$ & $3.65+ 51.76 s$ & segment fault & $5.72 s + 6.34 s$ & $5.75 s + 29.46 s$\\
			 & $maxstep=13$ & $11.68 s +168.15 s$& segment fault & $6.2 s + 6.89 s$ & $6.2 s + 31.96 s$\\
			 & $maxstep=15$ & timeout & segment fault & $7.18 s +  7.99 s$ &$7.34 s + 37.67s$\\
			 & $maxstep=20$ & timeout & segment fault & $9.68 s + 10.78 s$ &$9.74 s  + 50.04 s$\\
			\hline
		\end{tabular}
}
	\caption{\small Performance Comparison between Two Ways to Compute Simple P-log Programs}
\label{tab:PerformanceComparisonBetweenTwoWaysToComputePLogPrograms}
\end{minipage}
\end{table*}
The main idea of the syntactic translation is to introduce auxiliary probabilistic constants for encoding the assigned probability and the default probability.

Given a simple P-log program $\Pi$, a constant $c$, a set of literals $B$,\footnote{%
A literal is either an atom $A$ or its negation $\no\ A$.}
and a random selection rule $\left[ r\right] random(c)\leftarrow B^\prime$ in $\Pi$, we first introduce several notations, which resemble the ones used for defining the P-log semantics.

\begin{itemize}
\item  We define $\i{PR}_{B, r}(c)$ as 
\[ 
\ba l
 \{pr_r(c\mvis v \mid B)=p \in\Pi  \mid v\in\i{Dom}(c) \}
\ea
\]
if $\i{Act}$ in $\Pi$ does not contain $\i{Do}(c\mvis v')$ for any $v'\in\i{Dom}(c)$ 
and $\emptyset$ otherwise.

\item We define $AV_{B,r}(c)$ as
\[ 
  \left\{v \mid  pr_r(c\mvis v \mid B)=p\in \i{PR}_{B,r}(c)\right\}.
\]

\item For each $v\in AV_{B,r}(c)$, we define the {\em assigned probability} of $c\mvis v$ w.r.t. $B,r$, denoted by $ap_{B,r}(c\mvis v)$, as the value $p$ for which $pr_r(c\mvis v \mid B)=p\in \i{PR}_{B,r}(c)$.

\item We define the {\em default probability} for $c$ w.r.t. $B$ and $r$, denoted by $dp_{B,r}(c)$, as
\[ 
dp_{B,r}(c)=\frac{1-\sum_{v\in AV_{B,r}(c)} ap_{B,r}(c\mvis v)}{|\i{Dom}(c)\setminus AV_{B,r}(c)|}.
\] 

\item For each $c\in v$, define its {\em causal probability} w.r.t.  $B$ and $r$, denoted by $P(B, r, c\mvis v)$, as
\begin{align*}
 P(B,r, c\mvis v)=  
\begin{cases}
ap_{B,r}(c\mvis v) & \text{if $v\in AV_{B,r}(c)$}\\
dp_{B,r}(c)   &  \text{otherwise}.
\end{cases}
\end{align*}
\end{itemize}

Now we translate $\Pi$ into the corresponding multi-valued probabilistic program $\Pi^{\lpmln}$ as follows:
\begin{itemize}
\item The signature of $\Pi^{\lpmln}$ is 
\begin{align*}
  \sigma' & \cup \{pf_{B,r}^c \mvis  v \mid \i{PR}_{B, r}(c)\ne\emptyset \text{ and }v\in\i{Dom}(c)\} \\
          & \cup \{pf^c_{\Box,r}\mvis v \mid \text{$r$ is a random selection rule of $\Pi$ for $c$} \\[-0.5em]
 & \hspace{5cm} \text{ and $v\in\i{Dom}(c)$}\} \\ 
          & \cup \{\i{Assigned}_r\mvis \true \mid \text{$r$ is a random selection rule of $\Pi$}\}.
\end{align*}

\item $\Pi^{\lpmln}$ contains all rules in $\tau(\Pi)$.

\item For any constant $c$, any random selection rule $r$ for $c$, and any set $B$ of literals such that $\i{PR}_{B, r}(c)\neq \emptyset$, include in $\Pi^{\lpmln}$:

\begin{itemize}
\item the probabilistic constant declaration:
{
\[
\ba l
P(B,r,c\mvis v_1):pf_{B,r}^c\mvis v_1\mid \dots \\
\hspace{3cm} \mid P(B,r,c\mvis v_n):pf_{B,r}^c\mvis v_n
\ea 
\]
}
for each probabilistic constant $pf^c_{B,r}$ of the signature, where $\{v_1,\dots, v_n\} = \i{Dom}(c)$.
The constant $pf_{B,r}^c$ is used for representing the probability distribution for $c$ when condition $B$ holds in the experiment represented by~$r$.

\item the rules
\beq 
   c\mvis v \ar B, B', pf_{B,r}^{c}\mvis v, \no\ \i{Intervene}(c).
\eeq{pr-assign}
for all $v\in\i{Dom}(c)$, where $B'$ is the body of the random selection rule~$r$. 
These rules assign $v$ to $c$ when the assigned probability distribution applies to $c\mvis v$. 

\item the rule
\[
 \i{Assigned}_r \ar B, B', \no\ \i{Intervene}(c)
\]
where $B'$ is the body of the random selection rule~$r$ (we abbreviate $\i{Assigned}_r\mvis\true$ as $\i{Assigned}_r$).
$\i{Assigned}_r$ becomes true when any pr-atoms for $c$ related to $r$ is applied.

\end{itemize}

\item For any constant $c$ and any random selection rule $r$ for $c$, include in $\Pi^{\lpmln}$:

\begin{itemize}
\item the probabilistic constant declaration
\[ 
  \frac{1}{|\i{Dom}(c)|}:pf_{\Box,r}^c\mvis v_1\mid \dots\mid \frac{1}{|\i{Dom}(c)|}:pf_{\Box,r}^c\mvis v_n
\] 
for each probabilistic constant $pf_{\Box,r}^c$ of the signature, 
where $\{v_1,\dots, v_n\}= \i{Dom}(c)$.
The constant $pf_{\Box,r}^c$ is used for representing the default probability distribution for $c$ when there is no applicable pr-atom.

\item the rules
\[ 
  c\mvis v \ar B', pf_{\Box, r}^{c}\mvis v, \no\ \i{Assigned}_r.
\] 
for all $v\in\i{Dom}(c)$, where $B'$ is the body of the random selection rule $r$. These rules assign $v$ to $c$ when the uniform distribution applies to $c\mvis v$.
\end{itemize}
\end{itemize}

\smallskip\noindent
{\bf Example~\ref{ex:dice} continued}\ \ 
{\sl 
The simple P-log program $\Pi$ in Example \ref{ex:dice} can be turned into the following multi-valued probabilistic program. In addition to $\tau(\Pi)$ we have 
{\small
\[
\ba c
0.25:pf_{O(d)=M,r(d)}^{Roll(d)}\mvis 6 \mid 0.15:pf_{O(d)=M,r(d)}^{Roll(d)}\mvis 5 \mid\\
~~~~~~~0.15:pf_{O(d)=M,r(d)}^{Roll(d)}\mvis 4 \mid 0.15:pf_{O(d)=M,r(d)}^{Roll(d)}\mvis 3 \mid\\ 
~~~~~~~0.15:pf_{O(d)=M,r(d)}^{Roll(d)}\mvis 2 \mid 0.15:pf_{O(d)=M,r(d)}^{Roll(d)}\mvis 1 \\[0.5em]
\frac{1}{6}:pf_{\Box,r(d)}^{Roll(d)}\mvis 6 \mid \frac{1}{6}:pf_{\Box,r(d)}^{Roll(d)}\mvis 5 \mid
\frac{1}{6}:pf_{\Box,r(d)}^{Roll(d)}\mvis 4 \mid \\ ~~~~~~~\frac{1}{6}:pf_{\Box,r(d)}^{Roll(d)}\mvis 3 \mid
\frac{1}{6}:pf_{\Box,r(d)}^{Roll(d)}\mvis 2 \mid \frac{1}{6}:pf_{\Box,r(d)}^{Roll(d)}\mvis 1 \\[0.5em]
 \i{Roll}(d)\mvis x\leftarrow \i{Owner}(d)\mvis \i{Mike}, pf_{O(d)=M,r(d)}^{Roll(d)}\mvis x, \\ \hspace{3cm} \no\ \i{Intervene}(\i{Roll}(d))  \\[0.3em]
 \i{Assigned}_{r(d)}\leftarrow \i{Owner}(d)\mvis \i{Mike}, \no\ \i{Intervene}(\i{Roll}(d)) \\[0.3em]
 \i{Roll}(d)\mvis x\leftarrow pf_{\Box,r(d)}^{Roll(d)}\mvis x, \no\ \i{Assigned}_{r(d)}. \\[-0.5em]
\ea
\]
}}

\begin{thm}\label{thm:p-log-to-lpmln}
For any consistent simple P-log program $\Pi$ of signature $\sigma$ and any possible world $W$ of $\Pi$, we construct a formula $F_W$ as follows.
\[
\ba {rl}
   F_W = & (\bigwedge_{c=v\in W}c\mvis v)\wedge  \\
         & (\bigwedge_{\substack{\text{$c,v:$}\\\text{$c=v$ is possible in $W$,}\\\text{ $W\models c=v$ and $\i{PR}_{W}(c)\neq \emptyset$}}}pf^{c}_{B_{W,c}, r_{W,c}}\mvis v)\ \\ 
        & \wedge (\bigwedge_{\substack{\text{$c,v:$}\\\text{$c=v$ is possible in $W$,}\\\text{ $W\models c=v$ and $\i{PR}_{W}(c)= \emptyset$}}}pf^{c}_{\Box, r_{W,c}}\mvis v)
\ea 
\] 
We have
\[
  \mu_{\Pi}(W) = P_{\Pi^{\lpmln}}(F_W),
\] 
and, for any proposition $A$ of signature $\sigma$,
\[ 
   P_\Pi(A) = P_{\Pi^{\lpmln}}(A).
\]
\end{thm}


\smallskip\noindent
{\bf Example~\ref{ex:dice} continued}\ \ 
{\sl 
For the possible world 
{\small
\[
\ba {rl}
  W= & \{\i{Roll}(D_1)\mvis 6, \i{Roll}(D_2)\mvis 3, \i{Even}(D_1), \sneg \i{Even}(D_2),\\
     &  \i{Owner}(D_1)\mvis \i{Mike}, \i{Owner}(D_2)\mvis \i{John}\}, 
\ea
\]
}
$F_W$ is 
{\small
\[
\ba {rl}
   & \i{Roll}(D_1)\mvis 6\wedge \i{Roll}(D_2)\mvis 3\wedge 
          \i{Even}(D_1)\wedge \sneg \i{Even}(D_2) \\
        & \wedge\  \i{Owner}(D_1)\mvis \i{Mike}\wedge 
           \i{Owner}(D_2)\mvis \i{John}   \\
        & \wedge\ pf_{O(D_1)=M,r}^{Roll(D_1)}\mvis 6 \land
        pf^{Roll(D_2)}_{\Box, r}\mvis 3.
\ea
\]
}
It can be seen that $\hat{\mu}_\Pi(W)=\frac{1}{4}\times \frac{1}{6}=P_{\Pi^{\lpmln}}(F_W)$.
}
\smallskip

The embedding tells us that the exact inference in simple P-log is no harder than the one in $\lpmln$. 

\subsection{Experiments}

Following the translation described above, it is possible to compute a tight P-log program by translating it to $\lpmln$, and further turn that into the MLN instance following the translation introduced in Section~\ref{ssec:relation:mln2lpmln},
and then compute it using an MLN solver.

Table \ref{tab:PerformanceComparisonBetweenTwoWaysToComputePLogPrograms}
shows the performance comparison between this method and the native P-log implementation on some examples, which are modified from the ones  from \cite{baral09probabilistic}.  P-log 1.0.0 (\url{http://www.depts.ttu.edu/cs/research/krlab/plog.php}) implements two algorithms. The first algorithm (plog1) translates a P-log program to an ASP program and uses ASP solver {\sc smodels} to find all possible worlds of the P-log program. The second algorithm (plog2) produces a partially ground P-log program relevant to the query, and evaluates partial possible worlds to compute the probability of formulas. 
{\sc alchemy} 2.0 implements several algorithms for inference and learning. Here we use MC-SAT for lazy probabilistic inference, which combines MCMC with satisfiability testing. {\sc alchemy} first creates Markov Random Field (MRF) and then perform MC-SAT on the MRF created. The default setting of {\sc alchemy} performs 1000 steps sampling. We also tested with 5000 steps sampling to produce probability that is very close to the true probability. 
The experiments were performed on an Intel Core2 Duo CPU E7600 3.06GH with 4GB RAM running Ubuntu 13.10. The timeout was for 10 minutes. 

The experiments showed the clear advantage of the translation method that uses {\sc alchemy}. It is more scalable, and can be tuned to yield more precise probability with more sampling or less precise but fast computation, by changing sampling parameters. 
The P-log implementation of the second algorithm led to segment faults in many cases.




\section{Other Related Work}  \label{ssec:other-related}









We observed that ProbLog can be viewed as a special case of $\lpmln$. This result can be extended to embed Logic Programs with Annotated Disjunctions (LPAD) in $\lpmln$ based on the fact that any LPAD program can be further turned into a ProbLog program by eliminating disjunctions in the heads \cite[Section 3.3]{gutmann11oncontinuous}.

It is known that LPAD is related to several other languages. In~\cite{vennekens04logic}, it is shown that Poole's ICL~\cite{poole97theindependent} can be viewed as LPAD, and that acyclic LPAD programs can be turned into ICL. This indirectly tells us how ICL is related to~$\lpmln$. 

CP-logic \cite{vennekens09cp-logic} is a probabilistic extension of FO(ID) \cite{denecker07inductive} that is closely related to LPAD.



PrASP \cite{nickles14probabilistic} is another probabilistic ASP language. Like P-log and $\lpmln$, probability distribution is defined over stable models, but the weights there directly represent probabilities. 
%



Similar to $\lpmln$, log-linear description logics \cite{niepert11log} follow the weight scheme of log-linear models in the context of description logics.

\section{Conclusion} \label{sec:conclusion}



Adopting the log-linear models of MLN, language $\lpmln$ provides a simple and intuitive way to incorporate the concept of weights into the stable model semantics.
While MLN is an undirected approach, $\lpmln$ is a directed approach, where the directionality comes from the stable model semantics. This makes $\lpmln$ closer to P-log and ProbLog.  On the other hand, the weight scheme adopted in $\lpmln$ makes it amenable to apply the statistical inference methods developed for MLN computation. 
More work needs to be done to find how the methods studied in machine learning will help us to compute weighted stable models. 
While a fragment of $\lpmln$ can be computed by existing implementations of and MLNs, one may design a native computation method for the general case.

The way that we associate weights to stable models is orthogonal to the way the  stable model semantics are extended in a deterministic way. Thus it is rather straightforward to extend $\lpmln$ to allow other advanced features, such as aggregates, intensional functions and generalized quantifiers.

\medskip\noindent
{\bf Acknowledgements}\ \ 
We are grateful to Michael Gelfond for many useful discussions regarding the different ideas behind P-log and $\lpmln$, and to Evgenii Balai, Michael Bartholomew, Amelia Harrison, Yunsong Meng, and the anonymous referees for their useful comments.
This work was partially supported by the National Science Foundation under Grants IIS-1319794 and IIS-1526301, and ICT R\&D program of MSIP/IITP 10044494 (WiseKB).

\fontsize{9.0pt}{10.0pt} \selectfont

\bibliographystyle{aaai}

\bibliography{bib,bib2}

\begin{thebibliography}{}

\bibitem[\protect\citeauthoryear{Balduccini and
  Gelfond}{2003}]{balduccini03logic}
Balduccini, M., and Gelfond, M.
\newblock 2003.
\newblock Logic programs with consistency-restoring rules.
\newblock In {\em International Symposium on Logical Formalization of
  Commonsense Reasoning, AAAI 2003 Spring Symposium Series},  9--18.

\bibitem[\protect\citeauthoryear{Baral, Gelfond, and
  Rushton}{2009}]{baral09probabilistic}
Baral, C.; Gelfond, M.; and Rushton, J.~N.
\newblock 2009.
\newblock Probabilistic reasoning with answer sets.
\newblock {\em TPLP} 9(1):57--144.

\bibitem[\protect\citeauthoryear{Bauters \bgroup et al\mbox.\egroup
  }{2010}]{bauters10possibilistic}
Bauters, K.; Schockaert, S.; De~Cock, M.; and Vermeir, D.
\newblock 2010.
\newblock Possibilistic answer set programming revisited.
\newblock In {\em 26th Conference on Uncertainty in Artificial Intelligence
  (UAI 2010)}.

\bibitem[\protect\citeauthoryear{Buccafurri, Leone, and
  Rullo}{2000}]{buccafurri00enhancing}
Buccafurri, F.; Leone, N.; and Rullo, P.
\newblock 2000.
\newblock Enhancing disjunctive datalog by constraints.
\newblock {\em Knowledge and Data Engineering, IEEE Transactions on}
  12(5):845--860.

\bibitem[\protect\citeauthoryear{De~Raedt, Kimmig, and
  Toivonen}{2007}]{deraedt07problog}
De~Raedt, L.; Kimmig, A.; and Toivonen, H.
\newblock 2007.
\newblock {P}rob{L}og: A probabilistic prolog and its application in link
  discovery.
\newblock In {\em IJCAI}, volume~7,  2462--2467.

\bibitem[\protect\citeauthoryear{Denecker and
  Ternovska}{2007}]{denecker07inductive}
Denecker, M., and Ternovska, E.
\newblock 2007.
\newblock Inductive situation calculus.
\newblock {\em Artificial Intelligence} 171(5-6):332--360.

\bibitem[\protect\citeauthoryear{Erdem and Lifschitz}{2003}]{erd03}
Erdem, E., and Lifschitz, V.
\newblock 2003.
\newblock Tight logic programs.
\newblock {\em TPLP} 3:499--518.

\bibitem[\protect\citeauthoryear{Fierens \bgroup et al\mbox.\egroup
  }{2015}]{fierens15inference}
Fierens, D.; Van~den Broeck, G.; Renkens, J.; Shterionov, D.; Gutmann, B.;
  Thon, I.; Janssens, G.; and De~Raedt, L.
\newblock 2015.
\newblock Inference and learning in probabilistic logic programs using weighted
  boolean formulas.
\newblock {\em TPLP} 15(03):358--401.

\bibitem[\protect\citeauthoryear{Gelfond and Lifschitz}{1988}]{gel88}
Gelfond, M., and Lifschitz, V.
\newblock 1988.
\newblock The stable model semantics for logic programming.
\newblock In Kowalski, R., and Bowen, K., eds., {\em Proceedings of
  International Logic Programming Conference and Symposium},  1070--1080.
\newblock MIT Press.

\bibitem[\protect\citeauthoryear{Gutmann}{2011}]{gutmann11oncontinuous}
Gutmann, B.
\newblock 2011.
\newblock {\em On Continuous Distributions and Parameter Estimation in
  Probabilistic Logic Programs}.
\newblock Ph.D. Dissertation, KU Leuven.

\bibitem[\protect\citeauthoryear{Lee and Lifschitz}{2003}]{lee03a}
Lee, J., and Lifschitz, V.
\newblock 2003.
\newblock Loop formulas for disjunctive logic programs.
\newblock In {\em Proceedings of International Conference on Logic Programming
  ({ICLP})},  451--465.

\bibitem[\protect\citeauthoryear{Lee and Wang}{2015}]{lee15aprobabilistic}
Lee, J., and Wang, Y.
\newblock 2015.
\newblock A probabilistic extension of the stable model semantics.
\newblock In {\em International Symposium on Logical Formalization of
  Commonsense Reasoning, AAAI 2015 Spring Symposium Series}.

\bibitem[\protect\citeauthoryear{Lee, Meng, and Wang}{2015}]{lee15markov}
Lee, J.; Meng, Y.; and Wang, Y.
\newblock 2015.
\newblock Markov logic style weighted rules under the stable model semantics.
\newblock In Technical Communications of the 31st International Conference on
  Logic Programming.

\bibitem[\protect\citeauthoryear{Lin and Zhao}{2004}]{lin04}
Lin, F., and Zhao, Y.
\newblock 2004.
\newblock {A}{S}{S}{A}{T}: Computing answer sets of a logic program by
  {S}{A}{T} solvers.
\newblock {\em Artificial Intelligence} 157:115--137.

\bibitem[\protect\citeauthoryear{Nickles and
  Mileo}{2014}]{nickles14probabilistic}
Nickles, M., and Mileo, A.
\newblock 2014.
\newblock Probabilistic inductive logic programming based on answer set
  programming.
\newblock In {\em 15th International Workshop on Non-Monotonic Reasoning (NMR
  2014)}.

\bibitem[\protect\citeauthoryear{Niepert, Noessner, and
  Stuckenschmidt}{2011}]{niepert11log}
Niepert, M.; Noessner, J.; and Stuckenschmidt, H.
\newblock 2011.
\newblock Log-linear description logics.
\newblock In {\em IJCAI},  2153--2158.

\bibitem[\protect\citeauthoryear{Poole}{1997}]{poole97theindependent}
Poole, D.
\newblock 1997.
\newblock The independent choice logic for modelling multiple agents under
  uncertainty.
\newblock {\em Artificial Intelligence} 94:7--56.

\bibitem[\protect\citeauthoryear{Richardson and
  Domingos}{2006}]{richardson06markov}
Richardson, M., and Domingos, P.
\newblock 2006.
\newblock Markov logic networks.
\newblock {\em Machine Learning} 62(1-2):107--136.

\bibitem[\protect\citeauthoryear{Sato}{1995}]{sato95astatistical}
Sato, T.
\newblock 1995.
\newblock A statistical learning method for logic programs with distribution
  semantics.
\newblock In {\em Proceedings of the 12th International Conference on Logic
  Programming (ICLP)},  715--729.

\bibitem[\protect\citeauthoryear{Vennekens \bgroup et al\mbox.\egroup
  }{2004}]{vennekens04logic}
Vennekens, J.; Verbaeten, S.; Bruynooghe, M.; and A, C.
\newblock 2004.
\newblock Logic programs with annotated disjunctions.
\newblock In {\em Proceedings of International Conference on Logic Programming
  (ICLP)},  431--445.

\bibitem[\protect\citeauthoryear{Vennekens, Denecker, and
  Bruynooghe}{2009}]{vennekens09cp-logic}
Vennekens, J.; Denecker, M.; and Bruynooghe, M.
\newblock 2009.
\newblock {CP}-logic: A language of causal probabilistic events and its
  relation to logic programming.
\newblock {\em TPLP} 9(3):245--308.

\end{thebibliography}

\def\wdbprm{W^{\prime\prime}}
\def\pdbprm{P^{\prime\prime}}
\def\smdbprm{SM^{\prime\prime}}
\def\grdtms{\textbf{\textit{At}}}

\onecolumn
{\large \bf Appendix to ``Weighted Rules under the Stable Model Semantics''}

\section{Proof of Proposition \ref{prop:sm_subset}}
We use $I\models_{SM} \Pi$ to denote ``the interpretation $I$ is a
(deterministic) stable model of the program $\Pi$''.

The proof of Proposition \ref{prop:sm_subset} uses the following
theorem, which is a special case of Theorem 2 in \cite{lee11b}. Given an ASP program $\Pi$ of signature $\sigma$ and a subset $Y$ of $\sigma$, we use $LF_{\Pi}(Y)$ to denote the loop formula of $Y$ for $\Pi$.

\begin{thm}\label{thm:lf_subset}
Let $\Pi$ be a program of a finite first-order signature $\sigma$ with no
function constants of positive arity, and let $I$ be an interpretation of $\sigma$ that satisfies $\Pi$. The following conditions are equivalent to each other:

(a) $I\smmodels \Pi$;

(b) for every nonempty finite subset $Y$ of atoms formed from constants in $\sigma$, $I$ satisfies $LF_{\Pi}(Y)$;

(c) for every finite loop $Y$ of $\Pi$, $I$ satisfies $LF_{\Pi}(Y)$.
\end{thm}

\bigskip

\noindent{\bf Proposition~\ref{prop:sm_subset} \optional{prop:sm_subset}}\
\ 
{\sl
For any (unweighted) logic program~$\Pi$ of signature $\sigma$, and
any subset $\Pi'$ of $\Pi$, if an interpretation $I$ is a stable model of $\Pi'$ and $I$ satisfies~$\Pi$, then $I$ is a stable model of~$\Pi$ as well.
}
\vspace{3 mm}

\begin{proof}

For any subset $L$ of $\sigma$, since $I$ is a stable model of
$\Pi^\prime$, by Theorem \ref{thm:lf_subset}, $I$ satisfies
$LF_{\Pi^\prime}(L)$, that is, $I$ satisfies $L^{\wedge} \rightarrow ES_{\Pi^\prime}(L)$. It can be seen that the disjunctive terms in $ES_{\Pi^\prime}(L)$ is a subset of the disjunctive terms in $ES_{\Pi}(L)$, and thus $ES_{\Pi^\prime}(L)$ entails $ES_{\Pi}(L)$. So $I$ satisfies $L^{\wedge} \rightarrow ES_{\Pi}(L)$, which is $LF_{\Pi}(L)$, and since in addition we have $I\vDash \Pi$, $I$ is a stable model of $\Pi$.
\qed
\end{proof}

\section{Proof of Proposition \ref{prop:soft}}

\noindent{\bf Proposition~\ref{prop:soft} \optional{prop:soft}}\
\ 
{\sl
If $\sm'[\Pi]$ is not empty, for every interpretation~$I$, $P'_\Pi(I)$ coincides with $P_\Pi(I)$.
}
\vspace{3 mm}

\begin{proof}
For any interpretation $I$, by definition, we have

\begin{align}
\nonumber P_{\Pi}(I) &= \lim_{\alpha\to\infty}\frac{W_{\Pi}(I)}{\sum_{J\in SM\left[\Pi\right]}W_{\Pi}(J)}\\
\nonumber &= \lim_{\alpha\to\infty}\frac{W_{\Pi}(I)}{\sum_{J\vDash_{SM}\overline{\Pi_J}}exp(\sum_{w:F\in {\Pi}_J}w)}.
\end{align}

We notice the following fact: If an interpretation $I$ belongs to $SM^\prime\left[\Pi\right]$, then $I$ satisfies $\overline{\Pi^{hard}}$ and $I$ is a stable model of $\overline{\Pi_{I}}$. This can be seen from the fact that if $I\vDash \overline{\Pi^{hard}}$, then we have $\Pi_{I} = \Pi^{hard}\cup (\Pi^{soft})_I$.
  
\begin{itemize}
\item Suppose $I\in SM^\prime\left[\Pi\right]$, which implies that $I$ satisfies $\overline{\Pi^{hard}}$ and is a stable model of $\overline{\Pi_I}$. Then we have

\begin{align}
\nonumber P_{\Pi}(I) =& \lim_{\alpha\to\infty}\frac{exp(\sum_{w:F\in \Pi_I}w)}{\sum_{J\vDash_{SM}\overline{\Pi_J}}exp(\sum_{w:F\in {\Pi}_J}w)}.
\end{align}
Splitting the denominator into two parts: those $J$'s that satisfy $\overline{\Pi^{hard}}$ and those that do not, and extracting the weights of formulas in $\overline{\Pi^{hard}}$, we have
{\tiny
\begin{align}
\nonumber P_{\Pi}(I)=&\lim_{\alpha\to\infty}\frac{exp(|\Pi^{hard}|\cdot\alpha)\cdot exp(\sum_{w:F\in \Pi_I\setminus \Pi^{hard}}w)}{exp(|\Pi^{hard}|\cdot\alpha)\cdot\sum_{J\vDash_{SM}\overline{\Pi_J}: J \vDash \overline{\Pi^{hard}}}exp(\sum_{w:F\in \Pi_J\setminus \Pi^{hard}}w) + \sum_{J\vDash_{SM}\overline{\Pi_J}: J \nvDash \overline{\Pi^{hard}}}exp(|\Pi^{hard}\cap \Pi_J|\cdot\alpha)\cdot exp(\sum_{w:F\in \Pi_J\setminus \Pi^{hard}}w)}.
\end{align}
}
We divide both the numerator and the denominator by $exp(|\Pi^{hard}|\cdot\alpha)$.
{\tiny
\begin{align}
\nonumber P_{\Pi}(I) &= \lim_{\alpha\to\infty}\frac{exp(\sum_{w:F\in \Pi_I\setminus \Pi^{hard}}w)}{\sum_{J\vDash_{SM}\overline{\Pi_J}: J \vDash \overline{\Pi^{hard}}}exp(\sum_{w:F\in \Pi_J\setminus \Pi^{hard}}w) + \frac{\sum_{J\vDash_{SM}\overline{\Pi_J}: J \nvDash \overline{\Pi^{hard}}}exp(|\Pi^{hard}\cap \Pi_J|\cdot\alpha)\cdot exp(\sum_{w:F\in \Pi_J\setminus \Pi^{hard}}w)}{exp(|\Pi^{hard}|\cdot\alpha)}}\\
\nonumber &= \lim_{\alpha\to\infty}\frac{exp(\sum_{w:F\in \Pi_I\setminus \Pi^{hard}}w)}{\sum_{J\vDash_{SM}\overline{\Pi_J}: J \vDash \overline{\Pi^{hard}}}exp(\sum_{w:F\in \Pi_J\setminus \Pi^{hard}}w) + \sum_{J\vDash_{SM}\overline{\Pi_J}: J \nvDash \overline{\Pi^{hard}}}\frac{exp(|\Pi^{hard}\cap \Pi_J|\cdot\alpha)}{exp(|\Pi^{hard}|\cdot\alpha)}\cdot exp(\sum_{w:F\in \Pi_J\setminus \Pi^{hard}}w)}.
\end{align}
}
For $J\nvDash \overline{\Pi^{hard}}$, we note $|\Pi^{hard}\cap \Pi_J|\leq |\Pi^{hard}|-1$, so

\begin{align}
\nonumber P_{\Pi}(I) &= \frac{exp(\sum_{w:F\in \Pi_I\setminus \Pi^{hard}}w)}{\sum_{J\vDash_{SM}\overline{\Pi_J}: J \vDash \overline{\Pi^{hard}}}exp(\sum_{w:F\in \Pi_J\setminus \Pi^{hard}}w)}\\
\nonumber &= P^{\prime}_{\Pi}(I).
\end{align}

\item Suppose $I\notin SM^\prime\left[\Pi\right]$, which implies that $I$ does not satisfy $\overline{\Pi^{hard}}$ or is not a stable model of $\overline{\Pi_I}$. Let $K$ be any interpretation in $SM^\prime\left[\Pi\right]$. By definition, $K$ satisfies $\overline{\Pi^{hard}}$ and $K$ is a
  stable model of $\overline{\Pi_K}$.
\begin{itemize}
\item Suppose $I$ is not a stable model of $\overline{\Pi_I}$. Then by definition, $W_{\Pi}(I)=W^\prime_{\Pi}(I)=0$, and thus $P_{\Pi}(I)=P^\prime_{\Pi}(I)=0$.
\item Suppose $I$ is a stable model of $\overline{\Pi_I}$ but $I$ does not satisfy
  $\overline{\Pi^{hard}}$. 
\begin{align}
\nonumber P_\Pi(I) &= \lim_{\alpha\to\infty}\frac{exp(\sum_{w:F\in \Pi_I}w)}{\sum_{J\vDash_{SM}\overline{\Pi_J}}exp(\sum_{w:F\in \Pi_J}w)}.
\end{align}
Since $K$ satisfies $\overline{\Pi^{hard}}$, we have 
  $\overline{\Pi^{hard}}\subseteq \overline{\Pi_K}$. By assumption we have that $K$ is a
  stable model of $\overline{\Pi_K}$. We split the denominator into $K$ and the other interpretations, which gives
\begin{align}
\nonumber P_\Pi(I)&= \lim_{\alpha\to\infty}\frac{exp(\sum_{w:F\in \Pi_I}w)}{exp(\sum_{w:F\in \Pi_K}w)+\sum_{J\neq K:J\vDash_{SM}\overline{\Pi_J}}exp(\sum_{w:F\in \Pi_J}w)}.
\end{align}
Extracting weights from the formulas in $\Pi^{hard}$, we have
\begin{align}
\nonumber P_{\Pi}(I)&= \lim_{\alpha\to\infty}\frac{exp(|\Pi^{hard}\cap \Pi_I|\cdot\alpha)\cdot exp(\sum_{w:F\in \Pi_I\setminus \Pi^{hard}}w)}{exp(|\Pi^{hard}|\cdot\alpha)\cdot exp(\sum_{w:F\in \Pi_K\setminus \Pi^{hard}}w)+\sum_{J\neq K:J\vDash_{SM}\overline{\Pi_J}}exp(\sum_{w:F\in \Pi_J}w)}\\
\nonumber &\leq \lim_{\alpha\to\infty}\frac{exp(|\Pi^{hard}\cap \Pi_I|\cdot\alpha)\cdot exp(\sum_{w:F\in \Pi_I\setminus \Pi^{hard}}w)}{exp(|\Pi^{hard}|\cdot\alpha)\cdot exp(\sum_{w:F\in \Pi_K\setminus \Pi^{hard}}w)}.
\end{align}

Since $I$ does not satisfy $\o{\Pi^{hard}}$, we have $|\Pi^{hard}\cap \Pi_I| \leq |\Pi^{hard}|-1$, and thus

\begin{align}
\nonumber P_{\Pi}(I) &\leq \lim_{\alpha\to\infty}\frac{exp(|\Pi^{hard}\cap \Pi_I|\cdot\alpha)\cdot exp(\sum_{w:F\in \Pi_I\setminus \Pi^{hard}}w)}{exp(|\Pi^{hard}|\cdot\alpha)\cdot exp(\sum_{w:F\in \Pi_K\setminus \Pi^{hard}}w)} = 0 =P^{\prime}_{\Pi}(I).
\end{align}

\end{itemize}
\end{itemize}
\qed
\end{proof}

The following proposition establishes a useful property.
\begin{prop}\label{prop:smprm}
Given an $\lpmln$ program $\Pi$ such that $SM^\prime\left[\Pi\right]$ is not empty, and an interpretation $I$, the following three statements are equivalent:
\begin{enumerate}
\item $I$ is a stable model of $\Pi$;
\item $I\in SM^\prime\left[\Pi\right]$;
\item $P^\prime_{\Pi}(I)>0$.
\end{enumerate}
\end{prop}

\begin{proof}
Firstly, it is easy to see that the second and third conditions are equivalent. We notice that $exp(x)>0$ for all $x\in (-\infty, +\infty)$. So it can be seen from the definition that $W^\prime_{\Pi}(I)>0$ if and only if $I\in SM^\prime\left[\Pi\right]$, and consequently $P^\prime_{\Pi}(I)>0$ if and only if $I\in SM^\prime\left[\Pi\right]$.

Secondly, by Proposition \ref{prop:soft}, we know that $P^\prime_{\Pi}(I)$ is equivalent to $P_{\Pi}(I)$. By definition, the first condition is equivalent to ``$P_{\Pi}(I)>0$''. So we have that the first condition is equivalent to the third condition.
\qed
\end{proof}

\bigskip

Proposition \ref{prop:smprm} does not hold if we replace ``$SM^\prime\left[\Pi\right]$'' by ``$SM\left[\Pi\right]$''.
\begin{example}
Consider the following $\lpmln$ program $\Pi$:
\begin{align}
\nonumber (r_1)\ \ \ \alpha:\ \ \ \ &p\\
\nonumber (r_2)\ \ \ 1:\ \ \ \ &q
\end{align}
and the interpretation $I=\left\{q\right\}$. $I$ belongs to $SM\left[\Pi\right]$ since $I$ is a stable model of $\o{\Pi_I}$, which contains $r_2$ only. However, $P_{\Pi}(I)=0$ since $I$ does not satisfy the hard rule $r_1$.
On the other hand, $I$ does not belong to $SM^\prime\left[\Pi\right]$.
\end{example}

\section{Proof of Theorem \ref{thm:lpmln-asp}}

\noindent{\bf Theorem~\ref{thm:lpmln-asp} \optional{thm:lpmln-asp}}\
\ 
{\sl
For any logic program $\Pi$, the (deterministic) stable models of $\Pi$ are exactly the (probabilistic) stable models of $\P_\Pi$ whose weight is $e^{k\alpha}$, where $k$ is the number of all (ground) rules in $\Pi$. If $\Pi$ has at least one stable model, then all stable models of $\P_\Pi$ have the same probability, and are thus the stable models of $\Pi$ as well.
}
\vspace{3 mm}

\begin{proof}
We notice that $\overline{(\P_{\Pi})^{hard}} = \Pi$.
We first show that an interpretation $I$ is a stable model of $\Pi$
if and only if it is a stable model of $\P_\Pi$ whose weight is
$e^{k\alpha}$. Suppose $I$ is a stable model of $\Pi$. Then $I$ is a
stable model of $\overline{(\P_{\Pi})^{hard}}$. Obviously $(\P_{\Pi})^{hard}$ is
$(\P_{\Pi})_I$. So the
weight of $I$ is $e^{k\alpha}$. 
Suppose $I$ is a stable model of $\P_\Pi$ whose weight is
$e^{k\alpha}$. Then $I$ satisfies all the rules in $\P_\Pi$, since all rules in $\P_\Pi$ contribute to its
weight, and $I$ is a stable model of
$\overline{(\P_{\Pi})_I}=\overline{(\P_{\Pi})^{hard}}$, which is equivalent to $\Pi$. So $I$ is a
stable model of $\Pi$.

Now suppose $\Pi$ has at least one stable model. It follows that $\overline{(\P_{\Pi})^{hard}}$ has some stable model.
\begin{itemize}
\item Suppose $I$ is not a stable model of $\Pi$.
\begin{itemize}
\item Suppose $I$ does not satisfy $\Pi$. Then $I\nvDash \overline{(\P_{\Pi})^{hard}}$. By Proposition \ref{prop:soft}, $P_{\P_{\Pi}}(I) = 0$, and consequently $I$ is not a stable model of $\P_{\Pi}$.
\item Suppose $I$ satisfies $\Pi$. Then $\overline{(\P_{\Pi})_I}=\Pi$ and $I$ is not a stable model of $\overline{(\P_{\Pi})_I}$. By definition, $W_{\P_{\Pi}}(I)=0$ and consequently $P_{\P_{\Pi}}(I) = 0$, which means that $I$ is not a stable model of $\P_{\Pi}$.
\end{itemize}
\item Suppose $I$ is a stable model of $\Pi$. Then $I\vDash \overline{(\P_{\Pi})^{hard}}$, $\overline{(\P_{\Pi})_I}=\Pi$ and $I$ is a stable model of $\overline{(\P_{\Pi})_I}$.

By Proposition \ref{prop:soft}, 
\begin{align}
\nonumber P_{\P_{\Pi}}(I) &= \frac{exp(\sum_{w:F\in (\P_{\Pi})_I\setminus (\P_{\Pi})^{hard}}w)}{\sum_{J\vDash_{SM}\o{(\P_{\Pi})_J}:J\vDash \overline{(\P_{\Pi})^{hard}}}exp(\sum_{w:F\in \overline{(\P_{\Pi})_J\setminus (\P_{\Pi})^{hard}}}w)}\\
\nonumber &= \frac{exp(0)}{\sum_{J\vDash_{SM}\o{(\P_{\Pi})_J}:J\vDash \o{(\P_{\Pi})^{hard}}}exp(\sum_{w:F\in \overline{(\P_{\Pi})_J\setminus (\P_{\Pi})^{hard}}}w)}.
\end{align}
It can be seen that ``$J\vDash_{SM}\overline{(\P_{\Pi})_J}:J\vDash\overline{(\P_{\Pi})^{hard}}$'' is equivalent to ``$J$ is a stable model of $\Pi$'', since $\overline{(\P_{\Pi})^{hard}} = \Pi$. Furthermore, since $\Pi\setminus \overline{(\P_{\Pi})^{hard}}=\emptyset$, we have $exp(\sum_{w:F\in \overline{(\P_{\Pi})_J\setminus (\P_{\Pi})^{hard}}}w) = exp(0)$ for all $J\vDash_{SM}\overline{\P_J}:J\vDash \overline{(\P_{\Pi})^{hard}}$. So
\begin{align}
\nonumber P_{\P_{\Pi}}(I) &= \frac{exp(0)}{\sum_{J\vDash_{SM} \Pi}exp(0)}\\
\nonumber &= \frac{1}{k}
\end{align}
where $k$ is the number of stable models of $\Pi$.
\end{itemize}
\qed
\end{proof}

\section{Proof of Proposition \ref{prop: weak2lpmln}}

To facilitate the proof, we introduce a formal definition of ASP programs with weak constraints, as follows.

An \emph{ASP program with weak constraints} is a pair
\[
\langle \Pi, \i{CONSTR}\rangle,
\]
where $\Pi$ is a set of standard ASP rules of the form (\ref{rule}), and $\i{CONSTR}$ is a set of weak constraints $C$ of the following form
\beq
	{\tt :\sim}\ \i{Body} \ \ [\i{Weight}],
\eeq{weak2}
where $Weight$ is a positive integer, and $\i{Body}$ is a set of
literals. We will refer to $\i{Body}$ by $Body(C)$, and $\i{Weight}$
by $Weight(C)$. The {\em penalty} that $I$ receives, denoted as $Penalty(I)$, is defined as
\[
Penalty(I)=\sum_{C\in \mi{CONSTR}: I\vDash Body(C)}Weight(C).
\]
The {\em stable models} of an ASP program with weak constraints $\langle \Pi, \i{CONSTR}\rangle$ are the elements of the following set
\[
\left\{I\mid \text{$I\smmodels \Pi$ and there does not exists $J\neq I$ such that $J\smmodels \Pi$ and $Penalty(J)<Penalty(I)$}\right\}.
\]

By $\langle \Pi, \i{CONSTR}\rangle^{\lpmln}$ we denote the following $\lpmln$ program:
\[
\left\{\alpha: R\mid R\in\Pi\right\}\cup\left\{-Weight(C): \bot\leftarrow Body(C)\mid C\in \i{CONSTR}\right\}.
\]

For any interpretation $K$, let $\i{CONSTR}_K$ denote the following set:
\[
\left\{\bot\leftarrow Body(C)\mid C\in \i{CONSTR}, K\vDash \neg Body(C)\right\}
\]

\begin{lemma}\label{lem:weak2lpmln}
For any program with weak constraints $\langle \Pi, \i{CONSTR}\rangle$ that has a stable model, an interpretation $I$ is a stable model of $\Pi$ if and only if $I$ is a stable model of $\langle \Pi, \i{CONSTR}\rangle^{\lpmln}$.
\end{lemma}
\begin{proof}
($\Rightarrow$) Since $\i{CONSTR}_I$ consists of constraints only, we can derive from the fact that $I$ is a stable model of $\Pi$ that $I$ is a stable model of $\Pi\cup \i{CONSTR}_I$, which is $
\o{(\langle \Pi, \i{CONSTR}\rangle^{\lpmln})^{hard}}\cup\o{((\langle \Pi, \i{CONSTR}\rangle^{\lpmln})^{soft})_I}$. So $I\in SM^{\prime}\left[\langle \Pi, \i{CONSTR}\rangle^{\lpmln}\right]$ and by Proposition \ref{prop:smprm}, $I$ is a stable model of $\langle \Pi, \i{CONSTR}\rangle^{\lpmln}$.

($\Leftarrow$) Consider any stable model $I$ of $\langle \Pi, \i{CONSTR}\rangle^{\lpmln}$. By Proposition \ref{prop:smprm}, $I\in SM^{\prime}\left[\langle \Pi, \i{CONSTR}\rangle^{\lpmln}\right]$. This means $I$ is a stable model of $\o{(\langle \Pi, \i{CONSTR}\rangle^{\lpmln})^{hard}}\cup\o{((\langle \Pi, \i{CONSTR}\rangle^{\lpmln})^{soft})_I}$, which is equivalent to $\Pi\cup \i{CONSTR}_I$. Since $\i{CONSTR}_I$ contains constraints only, $I$ is a stable model of $\Pi$.
\qed
\end{proof}

\bigskip

\noindent{\bf Proposition~\ref{prop: weak2lpmln} \optional{prop: weak2lpmln}}\
\ 
{\sl
For any program with weak constraints that has a stable model, its stable models are the same as the stable models of the corresponding $\lpmln$ program with the highest normalized weight. 
}
\vspace{3 mm}

\begin{proof}
\BOCC
For any program with weak constraints that has a stable model, let $I$ be any one of its stable models. Since $I$ is a stable model of $\Pi$, by definition, we have: 1) $I$ is a stable model of the set of standard ASP rules in $\Pi$, 2) there is no stable model of the set of standard ASP rules in $\Pi$ with penalty smaller then $I$. Due to how we construct the corresponding $\lpmln$ program, from 1) we have that $I$ must be the stable model of the corresponding $\lpmln$ program (By Theorem \ref{thm:lpmln-asp}); from 2) we have that there is no stable model of the corresponding $\lpmln$ program with smaller weights obtained from soft rules, which implies that $I$ is the stable model of the corresponding $\lpmln$ program with the highest normalized weight, since all soft rules in the corresponding $\lpmln$ program have negative weights. 
\EOCC

($\Rightarrow$) For any program with weak constraints $\langle \Pi, \i{CONSTR}\rangle$ that has a stable model, let $I$ be any one of its stable models. Since $I$ is a stable model of $\langle \Pi, \i{CONSTR}\rangle$, by definition, we have:
\begin{enumerate}
\item $I\smmodels \Pi$;
\item There does not exist $J\neq I$ such that $J\smmodels \Pi$ and $Penalty(J)<Penalty(I)$.
\end{enumerate}
From the first condition, by Lemma \ref{lem:weak2lpmln}, it follows that $I$ is a stable model of $\langle \Pi, \i{CONSTR}\rangle^{\lpmln}$.

Now we show that there does not exist any $J\neq I$ such that $J$ is a stable model of $\langle \Pi, \i{CONSTR}\rangle^{\lpmln}$ and $P_{\langle \Pi, \mi{CONSTR}\rangle^{\lpmln}}(J)>P_{\langle \Pi, \mi{CONSTR}\rangle^{\lpmln}}(I)$. Assume, for the sake of contradiction, that such $J$ exists. Then $J$ must be a stable model of $\Pi$ by Lemma \ref{lem:weak2lpmln}. Since $P_{\langle \Pi, \mi{CONSTR}\rangle^{\lpmln}}(J)>P_{\langle \Pi, \mi{CONSTR}\rangle^{\lpmln}}(I)$, due to how we translate $\langle \Pi, \i{CONSTR}\rangle$ to $\langle \Pi, \i{CONSTR}\rangle^{\lpmln}$, $Penalty(J)<Penalty(I)$, which is a contradiction to the second condition. So such $J$ does not exist.

So $I$ is a stable model of $\langle \Pi, \i{CONSTR}\rangle^{\lpmln}$ with the highest normalized weight.
 
($\Leftarrow$) Let $I$ be any stable model of $\langle \Pi, \i{CONSTR}\rangle^{\lpmln}$ with the highest normalized weight. 
\begin{itemize}
\item {\bf $I\smmodels \Pi$}: Since $I$ is a stable model of $\langle \Pi, \i{CONSTR}\rangle^{\lpmln}$, by Lemma \ref{lem:weak2lpmln}, $I$ is a stable model of $\Pi$.

\item {\bf There does not exist any $J$ s.t. $J\smmodels \Pi$ and $Penalty(J)<Penalty(I)$}: Suppose, to the contrary, that there exists such $J$. By Lemma \ref{lem:weak2lpmln}, $J$ is a stable model of $\langle \Pi, \i{CONSTR}\rangle^{\lpmln}$. Since $Penalty(J)<Penalty(I)$, $P_{\langle \Pi, \mi{CONSTR}\rangle^{\lpmln}}(J) > P_{\langle \Pi, \mi{CONSTR}\rangle^{\lpmln}}(I)$. This is a contradiction to the fact that $I$ is a stable model of $\langle \Pi, \i{CONSTR}\rangle^{\lpmln}$ with the highest normalized weight. So there cannot exist such $J$.
\end{itemize}

In conclusion, $I$ is a stable model of $\langle \Pi, \i{CONSTR}\rangle$.
\qed
\end{proof}

\section{Proof of Theorem \ref{thm:mln-lpmln} and Theorem \ref{thm:lpmln-lf}}

The following is a review of MLN from~\cite{richardson06markov}, slightly reformulated in order to facilitate our discussion. 

A \emph{Markov Logic Network (MLN)} $\L$ of signature $\sigma$ is a
finite set of pairs $\langle F, w\rangle$ (also written as a
``weighted formula'' $w: F$), where $F$ is a first-order formula of
$\sigma$ and $w$ is either a real number or a symbol $\alpha$ denoting
the ``hard weight.'' 
We say that $\L$ is {\em ground} if its formulas contain no variables.

We first define the semantics for ground MLNs. 
For any ground MLN $\L$ of signature $\sigma$ and any Herbrand
interpretation $I$ of~$\sigma$, we define $\L_I$ to be the set of
formulas in~$\L$ that are satisfied by $I$. The \emph{weight} of an
interpretation $I$ under~$\L$, denoted $W_\L(I)$, is defined as
{\small
\[ 
  W_\L(I) = exp\Bigg(\sum_{w:F\;\in\;\L\atop ~~~F\;\in\; \L_I} w\Bigg). 
\] 
}
The probability of $I$ under $\L$, denoted $P_\L(I)$, is defined as
\[ 
  P_\L(I) = \lim_{\alpha\to\infty} \frac{W_\L(I)}{\sum_{J\in PW}{W_\L(J)}}, 
\] 
where $PW$ (``Possible Worlds'') is the set of all Herbrand interpretations of
$\sigma$. We say that $I$ is a {\em model} of $\L$ if $\i{P}_\L(I)\ne 0$.

The definition is extended to any non-ground MLN by identifying it with
its {\sl ground instance}. 
Any MLN $\L$ of signature $\sigma$ can be identified with the ground
MLN, denoted $gr_\sigma[\L]$, by turning each formula in $\L$ into a set
of ground formulas.
The weight of each ground formula in $gr_\sigma[\L]$ is the same as the weight of the formula in $\L$ from which it is obtained.

Given a signature $\sigma$, we use $\grdtms(\sigma)$ to denote the set of all ground atoms that can be constructed from symbols in $\sigma$.

\bigskip

\noindent{\bf Theorem~\ref{thm:mln-lpmln} \optional{thm:mln-lpmln}}\
\ 
{\sl
Any MLN $\L$ and its $\lpmln$ representation $\Pi_\L$ have the same probability distribution over all interpretations.
}
\vspace{3 mm}

\begin{proof}
We show that for any interpretation $I$, $P_\L(I)=P_{\Pi_\L}(I)$. For a set of atoms ${\bf p}$, let ${\bf Choice}({\bf p})$ denote the set of weighted rules $\bigcup_{p\in {\bf p}}\left\{w: p \leftarrow not\ not\ p\right\}$. 

\begin{align}
\nonumber P_\L(I) &= \lim_{\alpha\to\infty}\frac{W_\L(I)}{\sum_{J\in PW}W_\L(J)}\\
\nonumber &= \lim_{\alpha\to\infty}\frac{exp(\sum_{w:F\in \L_I}w)}{\sum_{J\in PW}exp(\sum_{w:F\in \L_J}w)}.
\end{align}

Multiplying the weight of every interpretation by $exp(|\grdtms(\sigma)|\cdot w)$, we have

\begin{align}
\nonumber P_\L(I) &= \lim_{\alpha\to\infty}\frac{exp(|\grdtms(\sigma)|\cdot w)\cdot exp(\sum_{w:F\in \L_I}w)}{\sum_{J\in PW}exp(|\grdtms(\sigma)|\cdot w)\cdot exp(\sum_{w:F\in \L_J}w)}\\
\nonumber &= \lim_{\alpha\to\infty}\frac{exp(\sum_{w:F\in \L_I\cup {\bf Choice}(\grdtms(\sigma))}w)}{\sum_{J\in PW}exp(\sum_{w:F\in \L_J\cup {\bf Choice}(\grdtms(\sigma))}w)}.
\end{align}

Clearly ${\bf Choice}(\grdtms(\sigma))$ is a set of tautologies, and it can be seen from the construction of $\Pi_\L$ that $(\Pi_\L)_K=\L_K \cup {\bf Choice}(\grdtms(\sigma))$ for any interpretation $K$. So

\begin{align}
\nonumber P_\L(I) &= \lim_{\alpha\to\infty}\frac{exp(\sum_{w:F\in(\Pi_\L)_I}w)}{\sum_{J\in PW}exp(\sum_{w:F\in (\Pi_\L)_J}w)}.
\end{align}

By Theorem 2 in \cite{fer09}, for any interpretation $K$, the stable models of $\overline{(\Pi_\L)_K}$ are exactly the models of $\overline{\L_K}$. Since $K$ itself is a model of $\overline{\L_K}$, $K$ is a stable model of $\o{(\Pi_\L)_K}$. So 

\begin{align}
\nonumber P_\L(I) &= \lim_{\alpha\to\infty}\frac{W_{\Pi_\L}(I)}{\sum_{J\in SM\left[\Pi\right]}W_{\Pi_\L}(J)}\\
\nonumber &= P_{\Pi_\L}(I).
\end{align}
\qed
\end{proof}

\bigskip

We prove a more general version of Theorem \ref{thm:lpmln-lf} here, which is Theorem 4 in \cite{lee15aprobabilistic}. 

For a (deterministic) logic program $\Pi$, we use $LF_{\Pi}$ to denote the set $\left\{LF_{\Pi}(L) \mid \text{$L$ is a loop of $\Pi$}\right\}$.

\begin{lemma}\label{lem: lf_subset}
For any $\lpmln$ program $\Pi$ and any interpretation $I$ of the underlying signature $\sigma$, $I\vDash LF_{\overline{\Pi}}$ if and only if $I\vDash LF_{\overline{\Pi_I}}$.
\BOCC
For any $\lpmln$ program $\Pi$ and any interpretation $I$ of the underlying signature, $I\vDash LF_{\overline{\Pi}}$ if and only if $I$ is a stable model of $\overline{\Pi_I}$.
\EOCC
\end{lemma}

\begin{proof}
($\Rightarrow$) Suppose $I\vDash LF_{\overline{\Pi}}$. Consider any subset $K$ of $\sigma$. There are two possible cases:

\begin{itemize}
\item $I\nvDash K^{\wedge}$. In this case, $K^{\wedge}\rightarrow ES_{\o{\Pi_I}}(K)$ is trivially satisfied by $I$.
\item $I\vDash K^{\wedge}$. Since $I\vDash LF_{\overline{\Pi}}$, by Theorem \ref{thm:lf_subset}, we have
\[
K^{\wedge}\rightarrow\bigvee_{\substack{A\cap K\neq \emptyset\\ A\leftarrow B\wedge N\in \overline{\Pi}\\ B\cap K=\emptyset}}(B\wedge N\wedge\bigwedge_{b\in A\setminus K}\neg b)
\]
is satisfied by $I$. Consider the rules which contribute to the external support for $K$ in $\o{\Pi}$, i.e., $A\leftarrow B\wedge N\in \o{\Pi}$ such that $A\cap K\neq \emptyset$ and $B\cap K=\emptyset$. Since $A\cap K \neq \emptyset$ and $I \vDash K^{\wedge}$, we get $I\vDash A^{\vee}$. So all these rules are satisfied by $I$ and thus they all belong to $\o{\Pi_I}$, which means
\[
\bigvee_{\substack{A\cap K\neq \emptyset\\ A\leftarrow B\wedge N\in \overline{\Pi}\\ B\cap K=\emptyset}}(B\wedge N\wedge\bigwedge_{b\in A\setminus K}\neg b)=\bigvee_{\substack{A\cap K\neq \emptyset\\ A\leftarrow B\wedge N\in \overline{\Pi_I}\\ B\cap K=\emptyset}}(B\wedge N\wedge\bigwedge_{b\in A\setminus K}\neg b).
\]
So
\[
K^{\wedge}\rightarrow\bigvee_{\substack{A\cap K\neq \emptyset\\ A\leftarrow B\wedge N\in \overline{\Pi_I}\\ B\cap K=\emptyset}}(B\wedge N\wedge\bigwedge_{b\in A\setminus K}\neg b)
\]
i.e., 
\[
K^{\wedge}\rightarrow ES_{\o{\Pi_I}}(K)
\]
is satisfied by $I$.

\end{itemize}
In conclusion, $I$ satisfies $K^{\wedge}\rightarrow ES_{\o{\Pi_I}}(K)$ for all subsets $K$ of $\sigma$. By Theorem \ref{thm:lf_subset}, $I\vDash LF_{\overline{\Pi_I}}$. 


\BOCC
Suppose, to the contrary, that $I$ is not a stable model of $\overline{\Pi_I}$. According to \cite{fer06}, $I$ is a stable model of $\overline{\Pi_I}$ if and only if $I\vDash \overline{\Pi_I}\cup LF_{\overline{\Pi_I}}$. By the definition of $\overline{\Pi_I}$, we have $I\vDash \overline{\Pi_I}$. It follows that $I\nvDash LF_{\overline{\Pi_I}}$. Since $I\nvDash LF_{\overline{\Pi_I}}$, there is one subset $L$ of the underlying signature such that $I \vDash L^{\wedge}$ and 

\[
ES_{\overline{\Pi_I}}(L)^I=(\bigvee_{\substack{A\cap L\neq \emptyset\\ A\leftarrow B\wedge N\in \overline{\Pi_I}\\ B\cap L=\emptyset}}(B\wedge N\wedge\bigwedge_{b\in A\setminus L}\neg b))^I=false.
\]

Since $I\vDash LF_{\overline{\Pi}}$, we have

\[
ES_{\overline{\Pi}}(L)^I=(\bigvee_{\substack{A\cap L\neq \emptyset\\ A\leftarrow B\wedge N\in \overline{\Pi}\\ B\cap L=\emptyset}}(B\wedge N\wedge\bigwedge_{b\in A\setminus L}\neg b))^I=true.
\]
Since $\Pi_I\subseteq \Pi$, we have that
\[
\begin{split}
\left\{B\wedge N\wedge \bigwedge_{b\in A\setminus L}\neg b  \mid  A\cap L \neq\emptyset, A\leftarrow B\wedge N\in \overline{\Pi_I}, B\cap L=\emptyset\right\}\\
\subseteq \left\{B\wedge N \wedge \bigwedge_{b\in A\setminus L}\neg b  \mid  A\cap L \neq\emptyset, A\leftarrow B\wedge N\in \overline{\Pi}, B\cap L=\emptyset\right\},
\end{split}
\]
and thus $I\vDash B\wedge N\wedge \bigwedge_{b\in A\cap L}\neg b$ for some
\[
\begin{split}
(B\wedge N\wedge \bigwedge_{b\in A\cap L}\neg b)\in \left\{(B\wedge N \wedge \bigwedge_{b\in A\cap L} \neg b) \mid  A\cap L\neq \emptyset, A\leftarrow B\wedge N\in \overline{\Pi}, B\cap L=\emptyset\right\}\setminus \\
\left\{(B\wedge N \wedge \bigwedge_{b\in A\cap L}\neg b)  \mid  A\cap L\neq \emptyset, A\leftarrow B\wedge N\in \overline{\Pi_I}, B\cap L=\emptyset\right\}.
\end{split}
\]
Since there is some $a\in L$ such that $a\in A$ and $I\vDash L^{\wedge}$, $I$ satisfies $A$; Since $I\vDash B\wedge N\wedge \bigwedge_{b\in A:b \neq a}\neg b$, $I$ satisfies $B\wedge N$. Therefore, $I\vDash B\wedge N\rightarrow A$, which is a contradiction to $B\wedge N\rightarrow A\notin \Pi_I$. So the assumption that $I$ is not a stable model of $\overline{\Pi_I}$ cannot be true. Therefore, $I$ is a stable model of $\overline{\Pi_I}$.
\EOCC

($\Leftarrow$) (The reasoning is similar to the proof of Proposition \ref{prop:sm_subset}) Suppose $I$ satisfies $ LF_{\overline{\Pi_I}}$. For all subsets $L$ of $\sigma$, since $I\vDash LF_{\overline{\Pi_I}}$, by Theorem \ref{thm:lf_subset}, $I\vDash L^{\wedge} \rightarrow ES_{\overline{\Pi_I}}(L)$. Since $\Pi_I\subseteq \Pi$, it can be seen that the disjunctive terms in $ES_{\overline{\Pi_I}}(L)$ is a subset of the disjunctive terms in $ES_{\overline{\Pi}}(L)$, and thus $ES_{\overline{\Pi_I}}(L)$ entails $ES_{\overline{\Pi}}(L)$. So $I\vDash L^{\wedge} \rightarrow ES_{\overline{\Pi}}(L)$. So $I\vDash LF_{\overline{\Pi}}$.
\qed
\end{proof}
\BOCC
\noindent{\bf Theorem~\ref{thm:lpmln-lf} \optional{thm:lpmln-lf}}\
\ 
{\sl
$\Pi$ (under the $\lpmln$ semantics) and $\i{Comp}(\Pi)$ (under the MLN semantics) have the same probability distribution over
all interpretations.
}
\vspace{3 mm}
\EOCC

\begin{lemma}\label{lem:mln-hard-constraint}
Let $\L$ be an MLN, and let 
$\L^{hard}$ be the hard formulas in $\L$. Let $\o{\L^{hard}}$ be the set of formulas obtained from $\L^{hard}$ by dropping all weights. When $\o{\L^{hard}}$ is
satisfiable, 
\begin{itemize}
\item if $I$ satisfies $\o{\L^{hard}}$, 
\[ \i{P}_\L(I) = \frac{exp(\sum_{w:F\in\L_I\setminus \L^{hard}} w)}
     {\sum_{J\in PW: J\models \o{\L^{hard}}}exp(\sum_{w:F\ \in\ \L_J\setminus \L^{hard}} w)}
\]
\item otherwise, $\i{P}_\L(I)= 0$.\footnote{This proposition does not hold when $\o{\L^{hard}}$ is not satisfiable. For example, consider $\L=\left\{\alpha:p, \alpha:\leftarrow p\right\}$ and $I=\left\{p\right\}$. $I\nvDash \o{\L^{hard}}$ but $P_\P(I)=\frac{exp(\alpha)}{exp(\alpha)+exp(\alpha)}=0.5$.}
\end{itemize}
\end{lemma}

\vspace{3 mm}

\begin{proof}

For any interpretation $I$, by definition, we have

\begin{align}
\nonumber P_\L(I) &= \lim_{\alpha\to\infty}\frac{W_\L(I)}{\sum_{J\in PW}W_\L(J)}\\
\nonumber &= \lim_{\alpha\to\infty}\frac{W_\L(I)}{\sum_{J\in PW}exp(\sum_{w:F\in \L_J}w)}.
\end{align}

\begin{itemize}
\item Suppose $I$ satisfies $\o{\L^{hard}}$. We have

\begin{align}
\nonumber P_\L(I) =& \lim_{\alpha\to\infty}\frac{exp(\sum_{w:F\in \L_I}w)}{\sum_{J\in PW}exp(\sum_{w:F\in \L_J}w)}.
\end{align}
Splitting the denominator into two parts: those $J$ that satisfies $\o{\L^{hard}}$ and those that do not, and extracting the weight of formulas in $\L^{hard}$, we have
{\tiny
\begin{align}
\nonumber P_\L(I)=&\lim_{\alpha\to\infty}\frac{exp(|\L^{hard}|\cdot\alpha)\cdot exp(\sum_{w:F\in \L_I\setminus \L^{hard}}w)}{exp(|\L^{hard}|\cdot\alpha)\cdot\sum_{J \vDash \o{\L^{hard}}}exp(\sum_{w:F\in \L_J\setminus \L^{hard}}w) + \sum_{J \nvDash \o{\L^{hard}}}exp(|\L^{hard}\cap \L_J|\cdot\alpha)\cdot exp(\sum_{w:F\in \L_J\setminus \L^{hard}}w)}.
\end{align}
}
We divide both the numerator and the denominator by $exp(|\L^{hard}|\cdot\alpha)$.
\begin{align}
\nonumber P_\L(I) &= \lim_{\alpha\to\infty}\frac{exp(\sum_{w:F\in \L_I\setminus \L^{hard}}w)}{\sum_{J \vDash \o{\L^{hard}}}exp(\sum_{w:F\in \L_J\setminus \L^{hard}}w) + \frac{\sum_{J \nvDash \o{\L^{hard}}}exp(|\L^{hard}\cap \L_J|\cdot\alpha)\cdot exp(\sum_{w:F\in \L_J\setminus \L^{hard}}w)}{exp(|\L^{hard}|\cdot\alpha)}}\\
\nonumber &= \lim_{\alpha\to\infty}\frac{exp(\sum_{w:F\in \L_I\setminus \L^{hard}}w)}{\sum_{J \vDash \o{\L^{hard}}}exp(\sum_{w:F\in \L_J\setminus \L^{hard}}w) + \sum_{J \nvDash \o{\L^{hard}}}\frac{exp(|\L^{hard}\cap \L_J|\cdot\alpha)}{exp(|\L^{hard}|\cdot\alpha)}\cdot exp(\sum_{w:F\in \L_J\setminus \L^{hard}}w)}.
\end{align}

For $J\nvDash \o{\L^{hard}}$, we have $|\L^{hard}\cap \L_J|\leq |\L^{hard}|-1$, so

\begin{align}
\nonumber P_\L(I) &= \frac{exp(\sum_{w:F\in \L_I\setminus \L^{hard}}w)}{\sum_{J \vDash \o{\L^{hard}}}exp(\sum_{w:F\in \L_J\setminus \L^{hard}}w)}.
\end{align}

\item Suppose $I$ does not satisfy $\o{\L^{hard}}$. Since $\o{\L^{hard}}$ is satisfiable, there is at least one interpretation that satisfies $\o{\L^{hard}}$. Let $K$ denote any such interpretation. We have

\begin{align}
\nonumber P_\L(I) &= \lim_{\alpha\to\infty}\frac{exp(\sum_{w:F\in \L_I}w)}{\sum_{J\in PW}exp(\sum_{w:F\in \L_J}w)}.
\end{align}
Splitting the denominator into $K$ and the other interpretations, we have
\begin{align}
\nonumber P_\L(I)&= \lim_{\alpha\to\infty}\frac{exp(\sum_{w:F\in \L_I}w)}{exp(\sum_{w:F\in \L_K}w)+\sum_{J\neq K}exp(\sum_{w:F\in \L_J}w)}.
\end{align}
Extracting the weight from formulas in $\L^{hard}$, we have
\begin{align}
\nonumber P_\L(I)&= \lim_{\alpha\to\infty}\frac{exp(|\L^{hard}\cap \L_I|\cdot\alpha)\cdot exp(\sum_{w:F\in \L_I\setminus \L^{hard}}w)}{exp(|\L^{hard}|\cdot\alpha)\cdot exp(\sum_{w:F\in \L_K\setminus \L^{hard}}w)+\sum_{J\neq K}exp(\sum_{w:F\in \L_J}w)}\\
\nonumber &\leq \lim_{\alpha\to\infty}\frac{exp(|\L^{hard}\cap \L_I|\cdot\alpha)\cdot exp(\sum_{w:F\in \L_I\setminus \L^{hard}}w)}{exp(|\L^{hard}|\cdot\alpha)\cdot exp(\sum_{w:F\in \L_K\setminus \L^{hard}}w)}.
\end{align}

Since $I$ does not satisfy $\o{\L^{hard}}$, $|\L^{hard}\cap \L_I| \leq |\L^{hard}|-1$, and thus

\begin{align}
\nonumber P_\L(I) &\leq \lim_{\alpha\to\infty}\frac{exp(|\L^{hard}\cap \L_I|\cdot\alpha)\cdot exp(\sum_{w:F\in \L_I\setminus \L^{hard}}w)}{exp(|\L^{hard}|\cdot\alpha)\cdot exp(\sum_{w:F\in \L_K\setminus \L^{hard}}w)} = 0.
\end{align}

\end{itemize}
\qed
\end{proof}

For any $\lpmln$ program $\Pi$, define MLN program $\L_{\Pi}$ to be the union of $\Pi$ and $\left\{\alpha: LF_{\overline{\Pi}}(L)\mid \text{$L$ is a loop of $\overline{\Pi}$}\right\}$.

\bigskip

\BOCC
\noindent{\bf Theorem 4}\
\ 
{\sl
For any $\lpmln$ program $\Pi$ such that
\[
\left\{R \mid \alpha: R\in \Pi\right\}\cup\left\{LF_{\overline{\Pi}}(L)\mid \text{$L$ is a loop of $\overline{\Pi}$}\right\}
\]
is satisfiable, $\Pi$ and $\L_{\Pi}$ have the same probability distribution over all interpretations, and consequently, the stable models of $\Pi$ and the models of $\L_{\Pi}$ coincide.
}
\EOCC

\begin{lemma}\label{lem:sm_lf}
For any $\lpmln$ program $\Pi$ and any interpretation $I$, if $I\in SM^{\prime}\left[\Pi\right]$, then $I\vDash LF_{\overline{\Pi}}$.
\end{lemma}
\begin{proof}
Suppose $I\in SM^{\prime}\left[\Pi\right]$, then $I\smmodels \o{\Pi^{hard}}\cup\o{(\Pi^{soft})_I}$, which implies $I\smmodels \o{\Pi_{I}}$, and further implies $I \vDash LF_{\o{\Pi_{I}}}$. By Lemma \ref{lem: lf_subset}, $I\vDash LF_{\overline{\Pi}}$.
\qed
\end{proof}

\noindent{\bf Theorem 4}\
\ 
{\sl
For any $\lpmln$ program $\Pi$ such that $SM^{\prime}\left[\Pi\right]$ is not empty, $\Pi$ and $\L_{\Pi}$ have the same probability distribution over all interpretations, and consequently, the stable models of $\Pi$ and the models of $\L_{\Pi}$ coincide.
}
\vspace{3 mm}

\begin{proof}
We will show that $P_{\Pi}(I) = P_{\L_{\Pi}}(I)$ for all interpretations $I$. Since $SM^{\prime}\left[\Pi\right]$ is not empty, by Lemma \ref{lem:sm_lf}, there exists at least one interpretation $J$ such that $J\vDash LF_{\overline{\Pi}}$.


\begin{itemize}
\item Suppose $I$ is a stable model of $\overline{\Pi_I}$. By definition,

\begin{align}
\nonumber P_{\L_{\Pi}}(I) &= \lim_{\alpha\to\infty}\frac{exp(\sum_{r_i\in \overline{(\L_{\Pi})_I}}w_i)}{\sum_{J\in PW}exp(\sum_{r_i\in \overline{(\L_{\Pi})_J}}w_i)}.
\end{align}
Splitting the denominator into interpretations that satisfy $LF_{\overline{\Pi}}$ and those that do not, we get
\begin{align}
\nonumber P_{\L_{\Pi}}(I) &= \lim_{\alpha\to\infty}\frac{exp(\sum_{r_i\in \overline{(\L_{\Pi})_I}}w_i)}{\sum_{J\vDash LF_{\overline{\Pi}}}exp(\sum_{r_i\in \overline{(\L_{\Pi})_J}}w_i) + \sum_{J\nvDash LF_{\overline{\Pi}}}exp(\sum_{r_i\in \overline{(\L_{\Pi})_J}}w_i)}.
\end{align}
Extracting the weights from the formulas in $LF_{\overline{\Pi}}$, we get
{\tiny\begin{align} 
\nonumber P_{\L_{\Pi}}(I) &= \lim_{\alpha\to\infty}\frac{exp(|LF_{\overline{\Pi}}|\cdot\alpha)\cdot exp(\sum_{r_i\in \overline{(\L_{\Pi})_I}\setminus LF_{\overline{\Pi}}}w_i)}{\sum_{J\vDash LF_{\overline{\Pi}}}exp(|LF_{\o{\Pi}}|\cdot\alpha)\cdot exp(\sum_{r_i\in \o{(\L_\Pi)_J}\setminus LF_{\overline{\Pi}}}w_i) + \sum_{J\nvDash LF_{\overline{\Pi}}}exp(|\o{(\L_{\Pi})_J}\cap LF_{\overline{\Pi}}|\cdot\alpha)\cdot exp(\sum_{r_i\in \overline{(\L_{\Pi})_J}\setminus LF_{\overline{\Pi}}}w_i)}.
\end{align}}
Dividing both the numerator and the denominator by $exp(|LF_{\Pi}|\cdot\alpha)$, we have
{\tiny\begin{align}
\nonumber P_{\L_{\Pi}}(I) &= \lim_{\alpha\to\infty}\frac{exp(\sum_{r_i\in \overline{(\L_{\Pi})_I}\setminus LF_{\overline{\Pi}}}w_i)}{\sum_{J\vDash LF_{\overline{\Pi}}}exp(\sum_{r_i\in \o{(\L_{\Pi})_J}\setminus LF_{\overline{\Pi}}}w_i) + \sum_{J\nvDash LF_{\overline{\Pi}}}\frac{exp(|\o{(\L_{\Pi})_J}\cap LF_{\overline{\Pi}}|\cdot\alpha)}{exp(|LF_{\overline{\Pi}}|\cdot\alpha)}\cdot exp(\sum_{r_i\in \o{(\L_{\Pi})_J}\setminus LF_{\overline{\Pi}}}w_i)}.
\end{align}}
For those $J$ that do not satisfy $LF_{\overline{\Pi}}$, $|\o{(\L_\Pi)_J}\cap LF_{\overline{\Pi}}|\leq |LF_{\overline{\Pi}}| - 1$. So $\lim_{\alpha\to\infty}\frac{exp(|\overline{(\L_{\Pi})_J}\cap LF_{\overline{\Pi}}|\cdot\alpha)}{exp(|LF_{\overline{\Pi}}|\cdot\alpha)}=0$. Consequently
\begin{align}
\nonumber P_{\L_{\Pi}}(I) &= \frac{exp(\sum_{r_i\in \overline{(\L_\Pi)_I}\setminus LF_{\overline{\Pi}}}w_i)}{\sum_{J\vDash LF_{\overline{\Pi}}}exp(\sum_{r_i\in \o{(\L_{Pi})_J}\setminus LF_{\overline{\Pi}}}w_i)}.
\end{align}
From the construction of $\L_{\Pi}$ it can be easily seen that $\overline{(\L_{\Pi})_K}\setminus LF_{\overline{\Pi}} = \overline{\Pi_K}$ for all interpretations $K$. So
\begin{align}
\nonumber P_{\L_{\Pi}}(I) &= \frac{exp(\sum_{r_i\in \overline{\Pi_I}}w_i)}{\sum_{J\vDash LF_{\overline{\Pi}}}exp(\sum_{r_i\in \overline{\Pi_J}}w_i)}.
\end{align}
By Lemma \ref{lem: lf_subset}, for any $J\vDash LF_{\overline{\Pi}}$, we have $J\vDash LF_{\overline{\Pi_J}}$ and thus $J$ is a stable model of $\overline{\Pi_J}$. So
\begin{align}
\nonumber P_{\L_{\Pi}}(I) &= \frac{exp(\sum_{r_i\in \overline{\Pi_I}}w_i)}{\sum_{J\vDash_{SM} \overline{\Pi_J}}exp(\sum_{r_i\in \overline{\Pi_J}}w_i)}\\
\nonumber &= \frac{W_{\Pi}(I)}{\sum_{J\in SM\left[\Pi\right]}W_{\Pi}(J)}\\
\nonumber &= P_{\Pi}(I).
\end{align}

\item Suppose $I$ is not a stable model of $\overline{\Pi_I}$. Then $P_{\Pi}(I)=0$. On the other hand, since $I\vDash \overline{\Pi_I}$ by definition, it must be the case that $I\nvDash LF_{\overline{\Pi_I}}$. By Lemma \ref{lem: lf_subset}, $I \nvDash LF_{\overline{\Pi}}$. So there is at least one subset $L$ of $\sigma$ such that $I \nvDash LF_{\overline{\Pi}}(L)$. Clearly $\alpha: LF_{\overline{\Pi}}(L) \in \L_{\Pi}$ and $LF_{\overline{\Pi}}(L)\in \overline{(\L_{\Pi})^{hard}}$. So $I\nvDash \overline{(\L_\Pi)^{hard}}$. From the construction of $\L_{\Pi}$ we can see that $\overline{(\L_{\Pi})^{hard}}=\overline{\Pi^{hard}}\cup LF_{\overline{\Pi}}$. 
Since $SM^\prime\left[\Pi\right]$ is not empty, there is at least one interpretation $J$ such that $J\smmodels \o{\Pi^{hard}}\cup \o{(\Pi^{soft})_J}$. This interpretation $J$ satisfies $LF_{\o{\Pi^{hard}\cup (\Pi^{soft})_J}}$. By Lemma \ref{lem:sm_lf}, $J$ satisfies $LF_{\o{\Pi}}$. So $J$ satisfies $\overline{\Pi^{hard}}\cup LF_{\overline{\Pi}}$ and thus $(\L_{\Pi})^{hard}$ is satisfiable. By Lemma \ref{lem:mln-hard-constraint}, $P_{\L_{\Pi}}(I)=0$.
\end{itemize}
\qed
\end{proof}

The above theorem is a more general version of Theorem \ref{thm:lpmln-lf} because, for any tight program $\Pi$, $Comp(\Pi)$ coincide with $\L_{\Pi}$. This result can be found in \cite{fer09}.

\section{Proof of Theorem \ref{thm:lpmln-problog}}

In this section and the next section, we write $\sum_{x}f(x)$, where $f$ is some function over a Boolean variable, as a shorthand of $\sum_{x\in\left\{true, false\right\}}f(x)$, and write $\sum_{x_1, \dots, x_m}f(x_1, \dots, x_m)$ as a shorthand of $\sum_{x_1}\sum_{x_2}\dots\sum_{x_m}f(x_1, \dots, x_m)$.

Given a ProbLog program $\P$, let $PA_\P$ denote the set of all probabilistic atoms in $\P$. We say a subset $TC$ of $PA_\P$ is the total choice of an interpretation $I$ if for all $p\in TC$, $I\vDash p$ and for all $q\in PA_\P\setminus TC$, $I\nvDash q$.

\begin{lemma}\label{lem: maginal_one}
For any ProbLog program $\P$,
\begin{equation}\nonumber
\sum_{TC\subseteq PA_\P}Pr_\P(TC) = 1.
\end{equation}
\end{lemma}

\begin{proof}
Suppose $PA_\P=\left\{a_1, a_2, \dots, a_k\right\}$.
\begin{align}
\nonumber &\sum_{TC\subseteq PA_\P}Pr_\P(TC)\\
\nonumber =&\sum_{TC\subseteq PA_\P}(\prod_{a_i\in TC}p_i\cdot\prod_{a_j\in PA_\P\setminus TC}(1-p_j)).
\end{align}
Let ${\bf p_i}(x)$ where $x\in\left\{true, false\right\}$ be defined as
\[
{\bf p_i}(x)=\begin{cases}\nonumber p_i&if\ x=true\\1-p_i&if\ x=false\end{cases}.
\]
Clearly $\sum_{a}{\bf p_i}(a)=1$ for any $i\in\left\{1, \dots, k\right\}$. $\sum_{TC\subseteq PA_\P}Pr_\P(TC)$ can be rewritten as
\begin{align}
\nonumber &\sum_{TC\subseteq PA_\P}Pr_\P(TC)\\
\nonumber =& \sum_{{\bf a}_1, {\bf a}_2, \dots, {\bf a}_k}{\bf p_i}({\bf a}_i)\cdot\dots\cdot {\bf p_i}({\bf a}_i)
\end{align}
where ${\bf a}_1, {\bf a}_2, \dots, {\bf a}_k$ are Boolean variables representing whether or not $a_i\in TC$, i.e., ${\bf a_i}=true$ if $a_i\in TC$, ${\bf a_i}=false$ otherwise. Rearranging the equation we have
\begin{align}
\nonumber &\sum_{C\subseteq PA_\P}Pr_\P\left[C\right]\\
\nonumber =& \sum_{{\bf a}_1}{\bf p_1}({\bf a}_1)\sum_{{\bf a}_2}{\bf p_2}({\bf a}_2)\cdot\sum_{{\bf a}_k}{\bf p_k}({\bf a}_k)\\
\nonumber =& 1.
\end{align}
\end{proof}

\begin{thm}\label{lem: wfm_sm}
When $\Pi$ has a total well-founded model, then this model is also the single stable model of $\Pi$.
\end{thm}

\begin{proof}
Proven in \cite{van91}.
\end{proof}

\begin{lemma}\label{lem:lpmln-problog}
Let $\P$ be any ProbLog program that does not contain any probabilistic atom for which the probability is $0$ or $1$. $\P$ and its $\lpmln$
representation $\P^\prime$ have the same probability distribution over
all interpretations. 
\end{lemma}

\begin{proof}
Since $\P$ is a well-defined ProbLog program, for all $TC\subseteq PA_\P$, $TC\cup \Pi$ has one total well-founded model. Let $TC(I)$ denote the total choice of $I$.
\begin{itemize}
\item Suppose $I$ is the total well-founded model of $TC(I)\cup \Pi$. According to the definition,
\begin{align}
\nonumber P_\P(I) &= Pr_\P(TC(I))\\
\nonumber &= \prod_{a_i\in TC(I)}p_i\cdot\prod_{b_j\in PA_\P\setminus TC(I)}(1-p_j).
\end{align}

By Theorem \ref{lem: wfm_sm}, $I$ is also the unique stable model of $TC(I)\cup \Pi$. It can be seen that $I$ is the only stable model of $TC(I)\cup \Pi\cup \left\{\leftarrow p \mid p\notin TC(I)\right\}$, which is $\o{\P^\prime_{I}}$. Clearly $\o{(\P^\prime)^{hard}}=\Pi\subseteq\o{\P^\prime_I}$ and consequently $I\smmodels \o{(\P^\prime)^{hard}}\cup \o{(\P^\prime)^{soft}_I}$. By Proposition \ref{prop:soft},

\begin{align}
\nonumber P_{\P^\prime}(I) &= \frac{exp(\sum_{F_i\in \o{\P^\prime_I}\setminus \o{(\P^\prime)^{hard}}}w_i)}{\sum_{J\in SM^\prime\left[\P^\prime\right]}exp(\sum_{F_i\in \o{\P^\prime_J}\setminus \o{(\P^\prime)^{hard}}}w_i)}\\
\nonumber &= \frac{exp(\sum_{a_i\in PA_\P:I\vDash a_i}ln(p_i)+\sum_{a_i\in PA_\P:I\nvDash a_i}ln(1-p_i))}{\sum_{J\in SM^\prime\left[\P^\prime\right]}exp(\sum_{a_i\in PA_\P:J\vDash a_i}ln(p_i)+\sum_{a_i\in PA_\P:J\nvDash a_i}ln(1-p_i))}\\
\nonumber &= \frac{\prod_{a_i\in TC(I)}p_i\prod_{a_i\notin TC(I)}(1-p_i)}{\sum_{J\in SM^\prime\left[\P^\prime\right]}\prod_{a_i\in TC(J)}p_i\prod_{a_i\notin TC(J)}(1-p_i)}.
\end{align}
Clearly for every $J$ such that $J\in SM^\prime\left[\P^\prime\right]$, there is a total choice $TC(J)$. And since the ProbLog program $\P$ is well-defined, for every total choice $C$ there is a total well-founded model of $C\cup \Pi$. By Theorem \ref{lem: wfm_sm}, this means for every total choice $C$ there is a unique stable model of $C\cup \Pi$. It can be seen that this stable model is also the unique stable model of $C\cup \Pi\cup \left\{\neg p \mid p\notin TC(I)\right\}$. So
\begin{align}
\nonumber P_{\P^\prime}(I) &= \frac{\prod_{a_i\in TC(I)}p_i\prod_{a_i\notin TC(I)}(1-p_i)}{\sum_{TC^{\prime}\subseteq PA_\P}\prod_{a_i\in C}p_i\prod_{a_i\notin C}(1-p_i)}\\
\nonumber &= \frac{\prod_{a_i\in TC(I)}p_i\prod_{a_i\notin TC(I)}(1-p_i)}{\sum_{TC^{\prime}\subseteq PA_\P}Pr_\P(TC^\prime)}.
\end{align}
By Lemma \ref{lem: maginal_one}, the denominator equals $1$, so
\begin{align}
\nonumber P_{\P^\prime}(I) &= \prod_{a_i\in TC(I)}p_i\prod_{a_i\notin TC(I)}(1-p_i)\\
\nonumber &= P_\P(I).
\end{align}

\item Suppose $I$ is not the total well-founded model of $TC(I)\cup \Pi$. Then $P_\P(I) = 0$. Since $\P$ is well-defined. The total well-founded model $J$ of $TC(I)\cup \Pi$ exists and by Theorem \ref{lem: wfm_sm}, $J$ is also the unique stable model of $TC(I)\cup \Pi$. It must be the case that $I\neq J$ and thus $I$ cannot be a stable model of $TC(I)\cup \Pi$. There are following two cases:
\begin{itemize}
\item Suppose $I\nvDash TC(I)\cup \Pi$. Since $TC(I)$ is the total choice of $I$, $I \vDash TC(I)$. It follows that $I \nvDash \Pi$, i.e., there is at least one rule $F\in \Pi$ such that $I\nvDash F$. According to the definition, $\alpha:F\in (\P^\prime)^{hard}$. By Proposition \ref{prop:soft}, $P_{\P^\prime}(I) = 0$.
\item Suppose $I\vDash TC(I)\cup \Pi$ but $I$ is not a stable model of $TC(I)\cup \Pi$. By Theorem \ref{thm:lf_subset}, it follows that there must be at least one loop $L$ of $\o{\Pi}$ such that $I^\vDash L^\wedge$ but $I\nvDash ES_{TC(I)\cup \Pi}(L)$.
It can be seen that
\[
\P^\prime_{I}=TC(I)\cup \Pi\cup \left\{\leftarrow p\mid p\notin TC(I)\right\}.
\]
It can be seen that $ES_{\P^{\prime}_{I}}(L) = ES_{TC(I)\cup \Pi}(L)$.
It follows that $I\nvDash_{SM} \o{\P^\prime_{I}}$. So $W_{\P^\prime}(I)=0$ and thus $P_{\P^\prime}(I)=0$.
\end{itemize}
\end{itemize}
\end{proof}

\bigskip
\noindent{\bf Theorem~\ref{thm:lpmln-problog} \optional{thm:lpmln-problog}}\
\ 
{\sl
Any well-defined ProbLog program $\P$ and its $\lpmln$
representation $\P^\prime$ have the same probability distribution over
all interpretations. 
}
\vspace{3 mm}


\begin{proof}
We first convert $\P=\langle PF, \Pi\rangle$ into a ProbLog program that does not contain any probabilistic atom for which the probability is $0$ or $1$ as follows.
\begin{itemize}
\item For each probabilistic atom $p$ such that $pr(p) = 0$:
\begin{itemize}
\item Remove all the rules in $\Pi$ where $p$ occurs in the body positively (i.e., as the literal $p$);
\item Remove all the literals $\no\ p$ that occurs in $\Pi$.
\end{itemize}
\item For each probabilistic atom $q$ such that $pr(q) = 1$:
\begin{itemize}
\item Remove all the literals $p$ that occurs in $\Pi$;
\item Remove all the rules in $\Pi$ where $p$ occurs in the body negatively (i.e., as the literal $\no\ p$).
\end{itemize}
\end{itemize}

Let $\P^\prime$ denote the program obtained from $\P$ as above. Clearly $\P^\prime$ specifies the same probability distribution as $\P$, if we restrict attention to atoms other than those atoms for which the probability is $0$ or $1$. By Lemma \ref{lem:lpmln-problog}, $\P^\prime$ and its $\lpmln$
representation $\o{\P^\prime}$ have the same probability distribution over
all interpretations. From the construction of $\P^\prime$, it can be seen that $\P^\prime$ specifies the same probability distribution as $\o{\P^\prime}$ if we restrict attention to atoms other than those atoms for which the probability is $0$ or $1$. Also it is clearly that those atoms in $\P$ for which the probability is $0$ or $1$ have exactly the same constant truth values ass these atoms in $\P^\prime$. So $\P$ and its $\lpmln$
representation $\P^\prime$ have the same probability distribution over
all interpretations. 
\end{proof}

\section{Proof of Theorem \ref{thm:pdbprm}}

\BOCC
Given a multi-valued probabilistic program $\langle PF, \Pi\rangle$, we use $\sigma(\langle PF, \Pi\rangle)$ to denote all constants of its signature. We use $\sigma^{pf}(\langle PF, \Pi\rangle)$ to denote the set of all probabilistic constants in $\langle PF, \Pi\rangle$. We identify a multi-valued probabilistic program with its $\lpmln$ translation.
We say an interpretation $I$ is {\em consistent} if
\begin{itemize}
\item For all $c\in \sigma^{pf}(\langle PF, \Pi\rangle)$, $c=v\in I$ for exactly one $v\in Dom(c)$;
\item For all $c\in \sigma(\langle PF, \Pi\rangle)\setminus\sigma^{pf}(\langle PF, \Pi\rangle)$, $c=v\in I$ for at most one $v\in Dom(c)$.
\end{itemize}
For a consistent interpretation $I$, we further define the set $TC(I)$ to be the set
\[
\left\{c=v \mid \text{$c\in \sigma^{pf}(\langle PF, \Pi\rangle)$ and $c=v\in I$} \right\}.
\]

We define
\[
\smdbprm\left[\langle PF, \Pi\rangle\right] = \left\{I\mid \text{$I$ is consistent and $I\smmodels \Pi\cup TC(I)$}\right\}.
\]
For any interpretation $I\in \smdbprm\left[\langle PF, \Pi\rangle\right]$ and any $c\in \sigma^{pf}(\langle PF, \Pi\rangle)$, we denote $v$ as $c^I$ if $c=v\in I$. 
\EOCC

Given a multi-valued probabilistic $\lpmln$ program ${\bf \Pi}=\langle PF, \Pi\rangle$, we use $\sigma^{pf}({\bf \Pi})$ to denote the set of all probabilistic constants in ${\bf \Pi}$. It can be seen that, if we have $M_{{\bf \Pi}}(c=v)>0$ for all constants $c$ and $v\in Dom(c)$, then given a consistent interpretation $I$, we have ${\bf \Pi}^{\rm hard}=UEC\cup \Pi\cup SINGLE$, where
\begin{align}
\nonumber UEC &= \left\{\bot \leftarrow c=v_1\wedge c=v_2 \mid \text{$c$ is a constants of $\sigma$ and $v_1,v_2\in Dom(c), v_1\neq v_2$}\right\}\cup\\
\nonumber & \left\{\bot\leftarrow \neg \bigvee_{v\in Dom(c)} c=v \mid \text{$c\in \sigma^{pf}({\bf \Pi})$}\right\},
\end{align}
and
\[
SINGLE = \left\{c=v\mid M_{{\bf \Pi}}(c=v)=1\right\},
\]
and $({\bf \Pi}^{\rm soft})_I=TC(I)\setminus SINGLE$.

\begin{lemma}\label{lem:smprm_dbprm}
For any multi-valued probabilistic program ${\bf \Pi}=\langle PF, \Pi\rangle$, for which $\smdbprm\left[{\bf \Pi}\right]$ is not empty and $M_{{\bf \Pi}}(c=v)>0$ for all constants $c$ and $v\in Dom(c)$, and any interpretation $I$, $I$ belongs to $SM^\prime\left[T({\bf \Pi})\right]$ if and only if $I$ belongs to $\smdbprm\left[{\bf \Pi}\right]$.
\end{lemma}

\begin{proof}
It can be seen that
\begin{align}
\nonumber  &\o{T({\bf \Pi})^{\rm hard}}\cup \o{(T({\bf \Pi})^{\rm soft})_I}\\
\nonumber =& \Pi\cup UEC\cup SINGLE\cup (TC(I)\setminus SINGLE).
\end{align}

($\Rightarrow$) Suppose $I$ belongs to $SM^\prime\left[T({\bf \Pi})\right]$. By definition, $I$ satisfies $\o{T({\bf \Pi})^{hard}}$, which contains $UEC$. Obviously since $I$ satisfies $UEC$, $I$ is consistent. For those $c=v\in SINGLE$, it must be the case that $Dom(c)=\left\{v\right\}$. In this case, we have $c=v\in I$ since $I$ is consistent. So $SINGLE \subseteq TC(I)$ and thus $SINGLE\cup (TC(I)\setminus SINGLE)=TC(I)$. So we have 
\begin{align}
\nonumber &\o{T({\bf \Pi})^{\rm hard}}\cup \o{(T({\bf \Pi})^{\rm soft})_I}\\
\nonumber =& \Pi\cup UEC\cup TC(I).
\end{align}
and since $I$ is a stable model of $\o{T({\bf \Pi})^{\rm hard}}\cup \o{(T({\bf \Pi})^{\rm soft})_I}$, $I$ is a stable model of $\Pi\cup UEC\cup TC(I)$. It follows that $I$ is a stable model of $\Pi\cup TC(I)$ since $UEC$ contains constraints only. Since in addition we have $I$ is consistent, $I$ belongs to $\smdbprm\left[{\bf \Pi}\right]$.

($\Leftarrow$) Suppose $I$ belongs to $\smdbprm\left[{\bf \Pi}\right]$. By definition, $I$ is consistent, and $I$ is a stable model of $\Pi\cup TC(I)$. Clearly $I$ satisfies $UEC$ since $I$ is consistent. Since $UEC$ contains constraints only, $I$ is a stable model $\Pi\cup TC(I)\cup UEC$. For those $c=v\in SINGLE$, it must be the case that $Dom(c)=\left\{v\right\}$. In this case, we have $c=v\in I$ since $I$ is consistent. So $SINGLE \subseteq TC(I)$ and thus $SINGLE\cup (TC(I)\setminus SINGLE)=TC(I)$. So we have
\begin{align}
\nonumber &\Pi\cup UEC\cup TC(I)\\
\nonumber &=\Pi\cup UEC\cup SINGLE\cup (TC(I)\setminus SINGLE)\\
\nonumber &=\o{{\bf \Pi}^{\rm hard}}\cup \o{({\bf \Pi}^{\rm soft})_I}\\
\end{align}
So $I$ is a stable model of $\o{T({\bf \Pi})^{\rm hard}}\cup \o{(T({\bf \Pi})^{\rm soft})_I}$, and by definition $I$ belongs to $SM^\prime\left[T({\bf \Pi})\right]$.
\BOCC

For those $c=v\in SINGLE$, it must be the case that $Dom(c)=\left\{v\right\}$. In this case, we have $c=v\in I$ since $I$ is consistent. So $SINGLE \subseteq TC(I)$ and thus $SINGLE\cup (TC(I)\setminus SINGLE)=TC(I)$. So we have 
\begin{align}
\nonumber &\o{{\bf \Pi}^{\rm hard}}\cup \o{({\bf \Pi}^{\rm soft})_I}\\
\nonumber =& \Pi\cup UEC\cup TC(I).
\end{align}

Clearly, $I$ is consistent if and only if $I$ satisfies $UEC$. Since $UEC$ consists of constraints only, ``$I$ is consistent and $I$ is a stable model of $\Pi\cup TC(I)$'' is equivalent to ``$I$ is a stable model of $UEC\cup\Pi\cup TC(I)$'', which is further equivalent to ``$I$ is a stable model of $\o{{\bf \Pi}^{\rm hard}}\cup \o{({\bf \Pi}^{\rm soft})_I}$''.
So $I\in\smdbprm\left[{\bf \Pi}\right]$ if and only if $I\in SM^{\prime}\left[{\bf \Pi}\right]$.
\EOCC
\qed
\end{proof}

\bigskip

\begin{lemma}\label{lem:smdbprm}
For any multi-valued probabilistic program ${\bf \Pi}=\langle PF, \Pi\rangle$, for which $\smdbprm\left[{\bf \Pi}\right]$ is not empty and $M_{{\bf \Pi}}(c=v)>0$ for all constants $c$ and $v\in Dom(c)$, and any interpretation $I$, $I$ is a stable model of $T({\bf \Pi})$ if and only if $I\in \smdbprm\left[{\bf \Pi}\right]$.
\end{lemma}
\bigskip
\begin{proof}

By Lemma \ref{lem:smprm_dbprm}, $I$ belongs to $SM^\prime\left[T({\bf \Pi})\right]$ if and only if $I$ belong to $\smdbprm\left[{\bf \Pi}\right]$.
By Proposition \ref{prop:smprm}, $I$ is a stable model of $T({\bf \Pi})$ if and only if $I\in SM^{\prime}\left[T({\bf \Pi})\right]$. So $I$ is a stable model of $T({\bf \Pi})$ if and only if $I\in \smdbprm\left[{\bf \Pi}\right]$.

\BOCC
($\Rightarrow$) Suppose $I$ is a stable model of ${\bf \Pi}$, i.e., $P_{{\bf \Pi}}(I)>0$. 

\begin{itemize}
\item Clearly $I$ must be consistent, otherwise either the existence constraints on those constants $c\in \sigma^{pf}({\bf \Pi})$ 
\[
\bot \leftarrow \neg \bigvee_{v\in Dom(c)}c=v
\]
or the uniqueness constraints on all constants
\[
\bot \leftarrow c=v_1 \wedge v_2
\]
would be violated. Since both of them are hard constraints in ${\bf \Pi}$, by Proposition \ref{prop:soft}, $P_{{\bf \Pi}}(I)=0$, and thus $I$ would not be a stable model, which is a contradiction.

\item Furthermore, $I$ must be a stable model of $\Pi\cup TC(I)$.
\begin{itemize}
\item $I$ satisfies $\Pi\cup TC(I)$: For $I$ to be a stable model of ${\bf \Pi}$, $I$ must be a stable model of $\o{{\bf \Pi}_I}$ and it must be the case that $\Pi\cup TC(I)\subseteq\o{{\bf \Pi}_I}$ since all rules in $\Pi$ are hard rules and by definition $I\vDash TC(I)$. 
\item There does not exist any $J\subset I$ such that $J\vDash (\Pi\cup TC(I))^I$: Assume, to the contrary, that such $J$ exists. Consider the reduct of $\o{{\bf \Pi}_I}$ w.r.t. $I$. All the rule in $\o{{\bf \Pi}_I}\setminus \Pi\cup TC(I)$ are turned into $\top$ in $\o{{\bf \Pi}_I}^I$ because they are constraints satisfied by $I$. So the reduct of $\o{{\bf \Pi}_I}$ w.r.t. $I$ and the reduct of $\Pi\cup TC(I)$ w.r.t. $I$ are equivalent. Since there is an interpretation $J\subset I$ such that $J\vDash (\Pi\cup TC(I))^I$, we have $J\vDash \o{{\bf \Pi}_I}^I$ as well. This means $I$ is not a stable model of ${\bf \Pi}$, which is a contradiction. 
\end{itemize}
\end{itemize}
In conclusion, $I\in \smdbprm\left[{\bf \Pi}\right]$.

($\Leftarrow$) Suppose we have $I\in SM^{\prime\prime}{\bf \Pi}$, i.e., $I$ is consistent and $I\smmodels \Pi\cup TC(I)$.
\begin{itemize}
\item $I\vDash {\bf \Pi}^{\rm hard}$: Since $I$ is consistent, $I$ satisfies all the hard rules in $\o{{\bf \Pi}}\setminus \Pi$ (The fact that all the probabilities are positive ensures that no probabilistic constant declaration will be turned into hard rules, unless for probabilistic constants with singleton domain, in which case, the hard rule from its declaration must be satisfied for $I$ to be consistent). Since $I\smmodels \Pi\cup TC(I)$, $I\vDash\Pi$. So $I$ satisfies all the hard rules in $\o{\langle PF, \Pi\rangle}$.
\item $I$ is a stable model of ${\bf \Pi}^{\rm hard}\cup ({\bf \Pi}^{\rm soft})_I$: It can be seen that 
\[
{\bf \Pi}^{\rm hard}\cup ({\bf \Pi}^{\rm soft})_I = \Pi\cup TC(I)\cap UEC,
\]
where $UEC$ denotes the hard constraints in ${{\bf \Pi}}^{\rm hard}\setminus\Pi$.
Since $I\vDash UEC$, all the constraints in $UEC$ turn into $\top$ in $({\bf \Pi}^{\rm hard}\cup ({\bf \Pi}^{\rm soft})_I)^I$. So $({\bf \Pi}^{\rm hard}\cup ({\bf \Pi}^{\rm soft})_I)^I$ is equivalent to $(\Pi\cup TC(I))^I$. Since there does not exists any $J\subset I$ such that $J\vDash (\Pi\cup TC(I))^I$, there does not exists any $J\subset I$ such that $J\vDash ({\bf \Pi}^{\rm hard}\cup ({\bf \Pi}^{\rm soft})_I)^I$.
\end{itemize}
Consequently, by Proposition \ref{prop:soft}, $P_{{\bf \Pi}}(I)>0$ and thus $I$ is a stable model of ${\bf \Pi}$.
\EOCC

\qed
\end{proof}

\bigskip

Lemma \ref{lem:smdbprm} does not hold when $M_{{\bf \Pi}}(c=v)=0$ for some constant $c$ and $v\in Dom(c)$.
\begin{example}
Consider the following multi-valued probabilistic $\lpmln$ ${\bf \Pi}$:
\begin{align}
\nonumber &1: c=1 \mid 0: c=2\\
\nonumber &p
\end{align}
which translates into
\begin{align}
\nonumber \alpha\ \ \ &:\ \ \ \ c=1\\
\nonumber \alpha\ \ \ &:\ \ \ \ \bot\leftarrow c=2\\
\nonumber \alpha\ \ \ &:\ \ \ \ p.
\end{align}
The interpretation $I=\left\{c=2, p\right\}$ belongs to the set $\smdbprm\left[{\bf \Pi}\right]$. However, it is not a stable model of $T({\bf \Pi})$, since one hard rule is violated.
\end{example}



\bigskip
\noindent{\bf Theorem~\ref{thm:pdbprm} \optional{thm:pdbprm}}\
\ 
{\sl
For any multi-valued probabilistic program~${\bf \Pi}$ such that each $p_i$ in \eqref{eq:probabilistic-constant-declaration} is positive for every probabilistic constant $c$, 
if $\sm''[{\bf \Pi}]$ is not empty, then for any interpretation $I$, $P''_{\bf \Pi}(I)$ coincides with $P_{T({\bf \Pi})}(I)$.
\vspace{3 mm}

\bigskip

\begin{proof}
\begin{itemize}
\item Suppose $I\in \smdbprm\left[{\bf \Pi}\right]$. By Lemma \ref{lem:smprm_dbprm}, we have $I\in SM^{\prime}\left[{\bf \Pi}\right]$. By Proposition \ref{prop:soft}, we have
\begin{align}
\nonumber P_{T({\bf \Pi})}(I) &= P_{T({\bf \Pi})}^{\prime}(I)\\
\nonumber &= \frac{W^{\prime}_{T({\bf \Pi})}(I)}{\sum_{J\in SM^{\prime}\left[T({\bf \Pi})\right]}W^{\prime}_{T({\bf \Pi})}(J)}\\
\nonumber &= \frac{exp(\sum_{w:R\in T({\bf \Pi})_I}w)}{\sum_{J\in SM^{\prime}\left[T({\bf \Pi})\right]}exp(\sum_{w:R\in T({\bf \Pi})_J}w)}\\
\nonumber &= \frac{\prod_{w:R\in T({\bf \Pi})_I}exp(w)}{\sum_{J\in SM^{\prime}\left[T({\bf \Pi})\right]}\prod_{w:R\in T({\bf \Pi})_J}exp(w)}\\
\nonumber &= \frac{\prod_{\text{$c\in\sigma^{pf}({\bf \Pi})$ and $c^I=v$}}M_{{\bf \Pi}}(c=v)}{\sum_{J\in SM^{\prime}\left[T({\bf \Pi})\right]}\prod_{\text{$c\in\sigma^{pf}({\bf \Pi})$ and $c^J=v$}}M_{{\bf \Pi}}(c=v)}\\
\nonumber &=\pdbprm_{{\bf \Pi}}(I)
\end{align}
\item Suppose $I\notin \smdbprm\left[{\bf \Pi}\right]$. By Lemma \ref{lem:smdbprm}, $I$ is not a stable model of $T({\bf \Pi})$, so $P_{T({\bf \Pi})}(I)=0$. On the other hand, $\pdbprm_{{\bf \Pi}}(I) = 0$ since $\wdbprm_{{\bf \Pi}}(I) = 0$.
\end{itemize}
\qed
\end{proof}

\section{Proof of Theorem \ref{thm:p-log-to-lpmln}}
It can be easily seen from the definition of $P(B, r, c=v)$ and the definition of $P(W, c=v)$ that the following two lemmas hold:
\begin{lemma}\label{lem:W-to-B-c}
For any mini P-log program $\Pi$, any possible world $W$ of $\Pi$, any constant $c$ and any $v\in Dom(c)$ such that $c=v$ is possible in $W$, we have
\begin{equation}\nonumber
P(W, c=v)=P(B_{W,c}, r_{W,c}, c=v)
\end{equation}
\end{lemma}

\begin{lemma}\label{lem:PR-W-to-B-c}
For any mini P-log program $\Pi$, any possible world $W$ of $\Pi$, any constant $c$ and any $v\in Dom(c)$ such that $c=v$ is possible in $W$, we have
\begin{equation}\nonumber
PR_W(c)=PR_{B_{W,c}, r_{W,c}}(c).
\end{equation}
\end{lemma}

Furthermore, the following corollary can be derived:
\begin{cor}\label{lem:CausalProbability-PfAtom}
For any mini P-log program $\Pi$, any possible world $W$ of $\Pi$, any constant $c$ and any $v\in Dom(c)$ such that $c=v$ is possible in $W$ and $W\vDash c=v$, we have
\begin{itemize}
\item If $PR_{W}(c)\neq \emptyset$, then
\begin{equation}\nonumber
P(W, c=v) = M_{\Pi^{\lpmln}}(pf^{c}_{B_{W,c}, r_{W,c}}=v);
\end{equation}

\item If $PR_{W}(c)= \emptyset$, then
\begin{equation}\nonumber
P(W, c=v) = M_{\Pi^{\lpmln}}(pf^{c}_{\Box, r_{W,c}}=v).
\end{equation}

\end{itemize}
\end{cor}

For any interpretation $I$ of $\Pi$, we define the set $SM_{\Pi}(I)$ of stable models of $\Pi^{\lpmln}$ as follows:
\begin{equation}\nonumber
SM_{\Pi}(I)=\left\{J\mid \text{$J$ is a (probabilistic) stable model of $\Pi^{\lpmln}$ such that $J\vDash F_I$}\right\}.\footnote{The formula $F_I$ is defined in Theorem \ref{thm:p-log-to-lpmln}.}
\end{equation}

\bigskip

The proof of the next lemma uses a restricted version of the splitting theorem in \cite{fer09a}, which is reformulated as follows:

\begin{thm}\label{thm:splitting}
Let $\Pi_1$, $\Pi_2$ be two finite ground programs where rules are of the form (\ref{rule}), and ${\bf p}$, ${\bf q}$ be disjoint tuples of distinct atoms. If
\begin{itemize}
\item Each strongly connected component of the dependency graph of $\Pi_1\cup\Pi_2$ w.r.t. ${\bf p}\cup{\bf q}$ is a subset of ${\bf p}$ or a subset of ${\bf q}$.
\item No atom in ${\bf p}$ has a strictly positive occurrence in $\Pi_2$, and
\item No atom in ${\bf q}$ has a strictly positive occurrence in $\Pi_1$.
\end{itemize}
then an interpretation $I$ of $\Pi_1\cup \Pi_2$ is a stable model of $\Pi_1\cup \Pi_2$ relative to ${\bf p}\cup{\bf q}$ if and only if $I$ is a stable model of $\Pi_1$ relative to ${\bf p}$ and $I$ is a stable model of $\Pi_2$ relative to ${\bf q}$.
\end{thm}

\begin{lemma}\label{lem:irre_random}
Given a mini P-log program $\Pi$ and a possible world $I$ of $\Pi$, let $AIRRE_{\Pi}(I)$ denote the set of all assignments of the constants in the set
\[
IRRE_{\Pi}(I)=\sigma^{pf}(\Pi^{\lpmln})\setminus \left\{pf^{c}_{\Box, r_{I,c}}\mid \substack{\text{$c=v,I:$}\\\text{$c=v$ is possible in $I$}, \\\text{$I\vDash c=v$} \\\text{and $PR_{I}(c)= \emptyset$}}\right\}\setminus\left\{pf^{c}_{B_{I,c}, r_{I,c}}\mid \substack{\text{$c=v,I:$}\\\text{$c=v$ is possible in $I$}, \\\text{$I\vDash c=v$} \\\text{and $PR_{I}(c)\neq \emptyset$}}\right\}.
\].

There is a 1-1 correspondence between $SM_{\Pi}(I)$ and $AIRRE_{\Pi}(I)$.

\end{lemma}

\begin{proof}
We use $\sigma$ to refer to the signature of $\tau(\Pi)$, and $\sigma^{\prime}$ to refer to the signature of $\Pi^{\lpmln}$. We construct the 1-1 correspondence as follows.

Given an element $J$ in $SM_{\Pi}(I)$, i.e., a stable model of $\Pi^{\lpmln}$ which satisfies $F_I$, due to the UEC constraint for constants in $IRRE_{\Pi}(I)$, $SM_{\Pi}(I)$ must assign some value to all constants in $IRRE_{\Pi}(I)$ to be a stable model. We extract the assignment of atoms in $IRRE_{\Pi}(I)$ from $J$ to obtain the corresponding element in $AIRRE_{\Pi}(I)$.

Given any arbitrary assignment of constants in $IRRE_{\Pi}(I)$, we extend this assignment by assigning the constants in $\sigma(\Pi^{\lpmln})\setminus IRRE_{\Pi}(I)$ in the following way, to obtain the corresponding element $J$ in $SM_{\Pi}(I)$:
\begin{itemize}
\item For all $c=v\in I$, set $c^J=v$.
\item For all constants of the form $pf^{c}_{\Box, r_{I, c}}$, where $c\in \sigma$, $c^I=v$, $c=v$ is possible in $I$ and $PR_{I}(c)=\emptyset$, set $(pf^{c}_{\Box, r_{I, c}})^J=v$, and set $(Assigned_{r_{I, c}})^J$ to be undefined.
\item For all constants of the form $pf^c_{B_{I,c}, r_{I, c}}$, where $c\in\sigma$, $c^I=v$, $c=v$ is possible in $I$ and $PR_{I}(c)\neq \emptyset$, set $(pf^c_{B_{I,c}, r_{I, c}})^J=v$, and set $(Assigned_{r_{I, c}})^J={\bf t}$.
\end{itemize}
The above construction of $J$ guarantees that $J$ satisfies $\o{(\Pi^{\lpmln})^{hard}}$ and $F_I$. Next we show that $J$ is a stable model of $\Pi^{\lpmln}$:

We split rules in $\o{\Pi^{\lpmln}_J}$ into two subsets $\o{\Pi^{\lpmln}_{J,1}}$ and $\o{\Pi^{\lpmln}_{J,2}}$ as follows:

\begin{itemize}
\item $\o{\Pi^{\lpmln}_{J, 1}}$ contains all rules in $\tau(\Pi)$, and rules of the following forms:
\begin{enumerate}
\item $c=v \leftarrow B, B^\prime, pf^{c}_{B^\prime, r}=v, not\ intervene(c)$, where $c$ is a constant of $\sigma$, $v\in Dom(c)$, $B$ is the body of some random selection rule $r$ of the form $\left[r\right] random(c)\leftarrow B$, and $B^\prime$ appears in some pr-atom of the form $pr(c=v \mid B^{\prime}) = p$ where $p\in\left[0, 1\right]$;
\item $c=v \leftarrow B, pf^{c}_{\Box, r}=v, not\ Assigned_{r}, not\ intervene(c)$, where $c$ is a constant of $\sigma$, $v\in Dom(c)$, and $B$ is the body of some random selection rule $r$ of the form $\left[r\right] random(c)\leftarrow B$;
\end{enumerate} 
\item $\o{\Pi^{\lpmln}_{J, 2}}$ is $\o{\Pi^{\lpmln}_J}\setminus\o{\Pi^{\lpmln}_{J, 1}}$
\end{itemize}

It can be seen that no atom in $\sigma$ has a strictly positive occurrence in $\o{\Pi^{\lpmln}_{J, 2}}$, and no atom in $\sigma^\prime\setminus\sigma$ (Atoms of the form ``$Assigned_r$'' and ``$pf^{c}_{\_, r}$'') has a strictly positive occurrence in $\o{\Pi^{\lpmln}_{J, 1}}$. Furthermore, the construction of $\Pi^{\lpmln}$ guarantees that all loops of size greater than one involves atoms in $\sigma$ only. So each strongly connected component of the dependency graph of $\Pi^{\lpmln}_{J}$ w.r.t. $\sigma^\prime$ is a subset of $\sigma$ or a subset of $\sigma^{\prime}\setminus\sigma$. By Theorem \ref{thm:splitting}, it suffices to show that $J$ is a stable model of $\o{\Pi^{\lpmln}_{J, 1}}$ relative to $\sigma$ and $J$ is a stable model of $\o{\Pi^{\lpmln}_{J, 2}}$ relative to $\sigma^{\prime}\setminus\sigma$.

\begin{itemize}
\item {\bf $J$ is a stable model of $\o{\Pi^{\lpmln}_{J, 1}}$ relative to $\sigma$}: Since $I$ is a stable model of $\tau(\Pi)$ relative to $\sigma$, $J$ is a stable model of $\tau(\Pi)$ relative to $\sigma$. It can be easily seen from the construction of $J$ that $J\vDash \o{\Pi^{\lpmln}_{J, 1}}$. Since $\tau(\Pi)$ is a subset of $\o{\Pi^{\lpmln}_{J, 1}}$, by Proposition \ref{prop:sm_subset}, $J$ is a stable model of $\o{\Pi^{\lpmln}_{J, 1}}$ relative to $\sigma$.

\item {\bf $J$ is a stable model of $\o{\Pi^{\lpmln}_{J, 2}}$ relative to $\sigma^{\prime}\setminus\sigma$}: It can be easily seen from the construction of $J$ that $J\vDash \o{\Pi^{\lpmln}_{J, 2}}$. Also as we discussed earlier, all loops of size greater than one do not involve atoms in $\sigma^{\prime}\setminus\sigma$. So it suffices to show that the loop formula of each loop consisting of a single atom in $\sigma^{\prime}\setminus\sigma$ is satisfied by $J$. $\sigma^{\prime}\setminus\sigma$ contains two types of atoms: 1) atoms of the form $Assigned_r$, where $r$ is some random selection rule, and 2) atoms of the form $pf^{c}_{\_, r}=v$, where $c$ is a constant of $\sigma$, $\_$ is $\Box$ or $B$ such that $pr(c=v^\prime\mid B) = p$ is a pr-atom in $\Pi$, $v\in Dom(c)$, and $r$ is a random selection rule of the form $\left[r\right] random(c)\leftarrow B^\prime$. 
\begin{itemize}
\item Consider atoms of the form 1). These atoms appear and only appear at the head of rules of the form
\[
Assigned_r\leftarrow B, B^\prime not\ Intervene(c).
\]
where $c$ is the atom associated with the random selection rule $r$, $B^\prime$ is the body of the random selection rule $r$, and $B$ occurs in some pr-atom $pr(c=v \mid B)=p$. The body of this rule involves atoms in $\sigma$ only. The construction of $J$ sets $Assigned_r$ to be true only when $PR_I(c)\neq \emptyset$, which implies $Assigned_r$ is true in $J$ only when $J$ satisfies $not\ Intervene(c)$, $B$ and $B^\prime$. Note that $B, not\ Intervene(c)$ does not contain $Assigned_r$. So clearly $B, B^\prime not\ Intervene(c)$ is a one disjunctive term in $ES_{\Pi^{\lpmln}}(\left\{Assigned_r\right\})$. So $Assigned_r \rightarrow ES_{\Pi^{\lpmln}}(\left\{Assigned_r\right\})$ is satisfied.

\item Consider atoms of the form 2). Each of these atoms appears and only appears as an atomic fact in $\o{\Pi^{\lpmln}_{J, 2}}$. So the loop formulas for these atoms are of the form $pf^{c}_{\_, r}=v\rightarrow \top$. Clearly these formulas are satisfied by $J$.
\end{itemize}
So $J$ must be a stable model of $\o{\Pi^{\lpmln}_{J, 2}}$ relative to $\sigma^{\prime}\setminus\sigma$.
\end{itemize}

\BOCC
\begin{itemize}
\item The reduct of $\o{(\Pi^{\lpmln})_J}$ w.r.t. $J$ contains the reduct of $\tau(\Pi)$ w.r.t. $J$, which is the same as the reduct of $\tau(\Pi)$ w.r.t. $I$ since $\tau(\Pi)$ involves atoms in $I$ only. Since $I$ is a possible world of $\Pi$, $I$ is a stable model of $\tau(\Pi)$. So there is no subset of $I$ that satisfies the reduct of $\tau(\Pi)$ w.r.t. $I$. This means, the subset of $J$ which consists of atoms in $I$ cannot be replaced by a smaller subset with $J$ still satisfying the reduct of $\o{(\Pi^{\lpmln})_J}$ w.r.t. $J$.
\item Consider those atoms in $J$ but not in $I$. Those atoms are of the following two forms
\[
Assigned_r=\true
\]
where $r$ is a random selection rule of $\Pi$ for c, and
\[
pf^{c}_{\_, r}=v
\]
where $c$ is a constant, $v\in Dom(c)$, $r$ is the random selection rule in $\Pi$ for $c$, and $\_$ denotes a set of literals or $\Box$.
\begin{itemize}
\item For each atom of the form $pf^{c}_{\_, r}=v$ that are in $J$, there is a corresponding weighted atomic fact for this atom in $\Pi^{\lpmln}$. This atomic fact is clearly in $(\Pi^{\lpmln})_{J}$. Thus all atoms of the form $pf^{c}_{\_, r}=v$ that are in $J$ are needed to satisfy the reduct of $\o{(\Pi^{\lpmln})_J}$ w.r.t. $J$.
\item For those atoms of the form $Assigned_r=\true$ that are in $J$, there must be a constant $c$, an applied random selection rule of $c$ whose body is $B^\prime$ (so that $I\vDash B^\prime$ and thus $J\vDash B^\prime$), and some set of literals $B$ for which $pr(c=v\mid B)=p\in PR_{I}(c)$ for some $v\in Dom(c)$ and $p\in \left[0, 1\right]$, such that $I\vDash B$ (and thus $J\vDash B$). $\o{(\Pi^{\lpmln})_J}$ contains the rule $Assigned_r\ar B, B^\prime$ due to the translation from $\Pi$ to $\Pi^{\lpmln}$. The rule ensures that $Assigned_r$ must be satisfied by $K$ for any interpretation $K$ to satisfy the reduct of $\o{(\Pi^{\lpmln})_J}$ w.r.t. $J$.
\end{itemize}
\end{itemize}
So there is no proper subset of $J$ that satisfies $\o{(\Pi^{\lpmln})_J}$. So $J$ is a stable model of $\Pi^{\lpmln}$. so $J\in SM_{\Pi}(I)$.
\EOCC
\qed
\end{proof}

\begin{lemma}\label{thm:p-log-to-lpmln-unnormalized}
For any mini P-log program $\Pi$ and any possible world $I$ of $\Pi$, we have
\begin{equation}\nonumber
\hat{\mu}_{\Pi}(I) = \sum_{J:J\in SM_{\Pi}(I)}\wdbprm_{\Pi^{\lpmln}}(J).
\end{equation}
\end{lemma}

\begin{proof}
\begin{align}
\nonumber \hat{\mu}_{\Pi}(I) &=\prod_{\substack{\text{$c=v,I:$}\\\text{$c=v$ is possible in $I$}\\\text{and $I\vDash c=v$}}}P(I,c=v)\\
\nonumber & =\prod_{\substack{\text{$c=v,I:$}\\\text{$c=v$ is possible in $I$}, \\\text{$I\vDash c=v$} \\\text{and $PR_{I}(c)\neq \emptyset$}}}P(I,c=v) \times \prod_{\substack{\text{$c=v,I:$}\\\text{$c=v$ is possible in $I$}, \\\text{$I\vDash c=v$} \\\text{and $PR_{I}(c)= \emptyset$}}}P(I,c=v)\\
\nonumber (\text{Corollary \ref{lem:CausalProbability-PfAtom}})\hspace{2mm}& =\prod_{\substack{\text{$c=v,I:$}\\\text{$c=v$ is possible in $I$}, \\\text{$I\vDash c=v$} \\\text{and $PR_{I}(c)\neq \emptyset$}}}M_{\Pi^{\lpmln}}(pf^{c}_{B_{I,c}, r_{I,c}}=v) \times \prod_{\substack{\text{$c=v,I:$}\\\text{$c=v$ is possible in $I$}, \\\text{$I\vDash c=v$} \\\text{and $PR_{I}(c)= \emptyset$}}}M_{\Pi^{\lpmln}}(pf^{c}_{\Box, r_{I,c}}=v)\\
\nonumber & =\prod_{\substack{\text{$c=v,I:$}\\\text{$c=v$ is possible in $I$}, \\\text{$I\vDash c=v$} \\\text{and $PR_{I}(c)\neq \emptyset$}}}M_{\Pi^{\lpmln}}(pf^{c}_{B_{I,c}, r_{I,c}}=v) \times \prod_{\substack{\text{$c=v,I:$}\\\text{$c=v$ is possible in $I$}, \\\text{$I\vDash c=v$} \\\text{and $PR_{I}(c)= \emptyset$}}}M_{\Pi^{\lpmln}}(pf^{c}_{\Box, r_{I,c}}=v)\times \\
\nonumber &\prod_{\substack{\text{pf:}\\\text{$pf\in \sigma^{pf}(\Pi^{\lpmln})\setminus \left\{pf^{c}_{\Box, r_{I,c}}\mid \substack{\text{$c=v,I:$}\\\text{$c=v$ is possible in $I$}, \\\text{$I\vDash c=v$} \\\text{and $PR_{I}(c)= \emptyset$}}\right\}\setminus$}\\\text{$\left\{pf^{c}_{B_{I,c}, r_{I,c}}\mid \substack{\text{$c=v,I:$}\\\text{$c=v$ is possible in $I$}, \\\text{$I\vDash c=v$} \\\text{and $PR_{I}(c)\neq \emptyset$}}\right\}$}}}\sum_{v:v\in Dom(pf)} M_{\Pi^{\lpmln}}(pf=v)
\end{align}
Consider interpretations in the set $SM_{\Pi}(I)$. By Lemma \ref{lem:irre_random}, there is a 1-1 correspondence between those interpretations and assignments to constants in the set $\sigma^{pf}(\Pi^{\lpmln})\setminus \left\{pf^{c}_{\Box, r_{I,c}}\mid \substack{\text{$c=v,I:$}\\\text{$c=v$ is possible in $I$}, \\\text{$I\vDash c=v$} \\\text{and $PR_{I}(c)= \emptyset$}}\right\}\setminus\left\{pf^{c}_{B_{I,c}, r_{I,c}}\mid \substack{\text{$c=v,I:$}\\\text{$c=v$ is possible in $I$}, \\\text{$I\vDash c=v$} \\\text{and $PR_{I}(c)\neq \emptyset$}}\right\}$. Furthermore, for each of those interpretations $J$, $\wdbprm(J)$ is precisely the product of the probability assigned to constants in $\sigma^{pf}(\Pi^{\lpmln})$. Since the third term of the last equation above ranges over all assignments to constants in the set $\sigma^{pf}(\Pi^{\lpmln})\setminus \left\{pf^{c}_{\Box, r_{I,c}}\mid \substack{\text{$c=v,I:$}\\\text{$c=v$ is possible in $I$}, \\\text{$I\vDash c=v$} \\\text{and $PR_{I}(c)= \emptyset$}}\right\}\setminus\left\{pf^{c}_{B_{I,c}, r_{I,c}}\mid \substack{\text{$c=v,I:$}\\\text{$c=v$ is possible in $I$}, \\\text{$I\vDash c=v$} \\\text{and $PR_{I}(c)\neq \emptyset$}}\right\}$, we have
\begin{align}
\nonumber \hat{\mu}_{\Pi}(I)&=\prod_{\substack{\text{$c=v,I:$}\\\text{$c=v$ is possible in $I$}, \\\text{$I\vDash c=v$} \\\text{and $PR_{I}(c)\neq \emptyset$}}}M_{\Pi^{\lpmln}}(pf^{c}_{B_{I,c}, r_{I,c}}=v) \times \prod_{\substack{\text{$c=v,I:$}\\\text{$c=v$ is possible in $I$}, \\\text{$I\vDash c=v$} \\\text{and $PR_{I}(c)= \emptyset$}}}M_{\Pi^{\lpmln}}(pf^{c}_{\Box, r_{I,c}}=v)\times \\
\nonumber &\sum_{J:J\in SM_{\Pi}(I)}\prod_{\substack{\text{pf:}\\\text{$pf\in \sigma^{pf}(\Pi^{\lpmln})\setminus \left\{pf^{c}_{\Box, r_{I,c}}\mid \substack{\text{$c=v,I:$}\\\text{$c=v$ is possible in $I$}, \\\text{$I\vDash c=v$} \\\text{and $PR_{I}(c)= \emptyset$}}\right\}\setminus$}\\\text{$\left\{pf^{c}_{B_{I,c}, r_{I,c}}\mid \substack{\text{$c=v,I:$}\\\text{$c=v$ is possible in $I$}, \\\text{$I\vDash c=v$} \\\text{and $PR_{I}(c)\neq \emptyset$}}\right\}$}}} M_{\Pi^{\lpmln}}(pf=pf^J)\\
\nonumber &=\sum_{J:J\in SM_{\Pi}(I)}\Big[\prod_{\substack{\text{pf:}\\\text{$pf\in \sigma^{pf}(\Pi^{\lpmln})\setminus \left\{pf^{c}_{\Box, r_{I,c}}\mid \substack{\text{$c=v,I:$}\\\text{$c=v$ is possible in $I$}, \\\text{$I\vDash c=v$} \\\text{and $PR_{I}(c)= \emptyset$}}\right\}\setminus$}\\\text{$\left\{pf^{c}_{B_{I,c}, r_{I,c}}\mid \substack{\text{$c=v,I:$}\\\text{$c=v$ is possible in $I$}, \\\text{$I\vDash c=v$} \\\text{and $PR_{I}(c)\neq \emptyset$}}\right\}$}}} M_{\Pi^{\lpmln}}(pf=pf^J)\times\\
\nonumber &\prod_{\substack{\text{$c=v,I:$}\\\text{$c=v$ is possible in $I$}, \\\text{$I\vDash c=v$} \\\text{and $PR_{I}(c)\neq \emptyset$}}}M_{\Pi^{\lpmln}}(pf^{c}_{B_{I,c}, r_{I,c}}=v) \times \prod_{\substack{\text{$c=v,I:$}\\\text{$c=v$ is possible in $I$}, \\\text{$I\vDash c=v$} \\\text{and $PR_{I}(c)= \emptyset$}}}M_{\Pi^{\lpmln}}(pf^{c}_{\Box, r_{I,c}}=v)\Big]\\
\nonumber &=\sum_{J:J\in SM_{\Pi}(I)}\prod_{\text{$c=v\in\sigma^{pf}(\Pi^{\lpmln})$ and $c^J=v$}}M_{\Pi^{\lpmln}}(c=v)\\
\nonumber &=\sum_{J:J\in SM_{\Pi}(I)}\wdbprm_{\Pi^{\lpmln}}(J).
\end{align}
\qed
\end{proof}

\begin{lemma}\label{lem:lpmln-to-plog}
Given a consistent mini P-log program $\Pi$ of signature $\sigma$, for every stable model $J$ of $\Pi^{\lpmln}$ (whose signature is denoted by $\sigma^\prime$), $J$'s restriction on $\sigma$ is a possible world of $\Pi$.
\end{lemma}

\begin{proof}
We construct $J$'s restriction on $\sigma$ by defining $c^I=c^J$ for all $c\in \sigma$. 
\begin{itemize}
\item Clearly $J\in SM_{\Pi}(I)$.

\item Now we show that $I$ is a possible world of $\Pi$. Since $\Pi$ is consistent, $\tau(\Pi)$ is satisfiable, and thus $J\vDash \tau(\Pi)$ (Otherwise $J$ would not be a stable model of $\Pi^{\lpmln}$ according to Proposition \ref{prop:soft}). Since $J\vDash \tau(\Pi)$, we get $I\vDash \tau(\Pi)$. To see that $I$ is a stable model of $\Pi^{\lpmln}$, we consider the loop formula $L^{\wedge}\rightarrow ES_{\tau(\Pi)}(L)$ for any loop of $L$ of $\tau(\Pi)$ such that $I\vDash L^{\wedge}$. $L$ is a loop of $\o{\Pi^{\lpmln}_J}$ as well since $J\vDash \tau(\Pi)$, and it is satisfied by $J$ since $I\subseteq J$. Since $J$ is a stable model of $\o{\Pi^{\lpmln}_J}$, we have
\[
J\vDash ES_{\o{\Pi^{\lpmln}_J}}(L),
\]
i.e.,
\[
J\vDash \bigvee_{\substack{A\cap L\neq \emptyset\\ A\leftarrow B\wedge N\in \o{\Pi^{\lpmln}_J}\\ B\cap L=\emptyset}}(B\wedge N\wedge\bigwedge_{b\in A\setminus L}\neg b).
\]

Consider the following two cases:
\begin{itemize}
\item $L$ contains only atoms that are not possible in $I$. Since those atoms do not occur in the head of any rules in $\o{\Pi^{\lpmln}}\setminus \tau(\Pi)$, those rules do not contribute in $ES_{\o{\Pi^{\lpmln}_J}}(L)$. So $ES_{\o{\Pi^{\lpmln}_J}}(L)=ES_{\tau(\Pi)}(L)$ in this case. Since $\tau(\Pi)$ involves atoms in $\sigma$ only, and $I$ and $J$ agree on atoms in $\sigma$, we have
\[
I\vDash ES_{\tau(\Pi)}(L).
\]
\item $L$ contains some atoms that are possible in $I$. In this case, since $J\vDash ES_{\o{\Pi^{\lpmln}_J}}(L)$, there must be at least one rule $A\leftarrow B\wedge N\in \o{\Pi^{\lpmln}_J}$ such that $A\cap L\neq \emptyset$, $B\cap L=\emptyset$ and $J\vDash B\wedge N\wedge \bigwedge_{b\in A\setminus L}\neg b$. There are again two possible cases:
\begin{itemize}
\item $A\leftarrow B\wedge N\in \tau(\Pi)$. In this case, since $\tau(\Pi)$ involves atoms in $\sigma$ only, and $I$ and $J$ agree on atoms in $\sigma$, we have $I\vDash B\wedge N\wedge \bigwedge_{b\in A\setminus L}\neg b$. Since this rule contributes to $ES_{\tau(\Pi)}(L)$ as well, we have $I\vDash ES_{\tau(\Pi)}(L)$.
\item $A\leftarrow B\wedge N\notin \tau(\Pi)$. According to the construction of $\Pi^{\lpmln}$, $A\leftarrow B\wedge N$ must be of one of the following two forms:
\[
c=v\leftarrow B^\prime, pf^{c}_{\Box, r}=v, not\ Assigned_r
\]
or
\[
c=v\leftarrow B^{\prime\prime}, B^\prime, pf^{c}_{B, r}=v, not\ Intervene(c)
\]
where $c=v$ is some atom possible in $I$, $r$ is the random selection rule of the form
\[
\left[r\right] random(c) \leftarrow B^\prime,
\]
and $B^{\prime\prime}$ is the body of some pr-atom related to $c$ and $r$. In either case, $J$ satisfies $B^\prime$, which involves atoms in $\sigma$ only. So $I$ satisfies $B^\prime$ as well.

Consider the following rule in $\tau(\Pi)$:
\beq
c=v_1;\dots;c=v_n\leftarrow B^\prime, not\ Intervene(c).
\eeq{eq:rule_randomselect}
There are two possible cases:

\begin{itemize}
\item $J$ does not satisfy $Intervene(c)$. In this case, (\ref{eq:rule_randomselect}) is satisfied by $J$, and clearly
\[
c=v_1;\dots;c=v_n\leftarrow B^\prime, not\ Intervene(c)\in \left\{A\leftarrow B\wedge N\mid\substack{A\cap L\neq \emptyset\\ A\leftarrow B\wedge N\in \tau(\Pi)\\ B\cap L=\emptyset}\right\}.
\]
So one disjunctive term of $ES_{\tau(\Pi)}(L)$ is satisfied by $I$. So $ES_{\tau(\Pi)}(L)$ is satisfied by $I$.
\item $J$ satisfies $Intervene(c)$. In this case, for $J$ to be a stable model of $\Pi^{\lpmln}$, there must be a rule of the following form
\[
Intervene(c) \leftarrow Do(c=v)
\]
in $\tau(\Pi)$, where $c=v\in J$ and $c=v\in I$, whose body is satisfied by $J$, which means the following rule
\[
c=v \leftarrow Do(c=v)
\]
in $\tau(\Pi)$ is satisfied by $J$. Clearly
\[
c=v \leftarrow Do(c=v)\in \left\{A\leftarrow B\wedge N\mid\substack{A\cap L\neq \emptyset\\ A\leftarrow B\wedge N\in \tau(\Pi)\\ B\cap L=\emptyset}\right\}.
\]
So one disjunctive term of $ES_{\tau(\Pi)}(L)$ is satisfied by $I$. So $ES_{\tau(\Pi)}(L)$ is satisfied by $I$.
\end{itemize}

\end{itemize}
So $I$ satisfies $ES_{\tau(\Pi)}(L)$ for all loops $L$ of $\tau(\Pi)$. Consequently, $I$ is a stable model of $\tau(\Pi)$, and thus $I$ is a possible world of $\Pi$.
\end{itemize}

\BOCC

Let $\sigma_I$ denote the signature of $I$. Note that $\tau(\Pi)^I=\tau(\Pi)^J$ and $\tau(\Pi)^J\subseteq \o{(\Pi^{\lpmln})_J}^J$. Furthermore, $\tau(\Pi)$ involves atoms in $\sigma_I$ only. Since $J$ is a stable model of $\Pi^{\lpmln}$, there does not exist any interpretation $K$ such that $K\subset J$ and $K\vDash \o{(\Pi^{\lpmln})_J}^J$. 

we consider the following two possible cases for each $c=v\in I$:
\begin{itemize}
\item $c=v$ is possible in $I$. Then by definition there is a random selection rule $r$ for $c=v$ in $\Pi$ which is applied in $I$, and thus in $\tau(\Pi)$ there must be a rule of the form
\begin{equation}\nonumber
c=v_1;\dots,c=v_n\leftarrow B, not\ Intervene(c)
\end{equation}
where $B$ is the body of $r$. Since ``$B, not\ Intervene(c)$'' is satisfied by $I$, $c=v^\prime$ for all $v^\prime\neq v$ is not satisfied by $I$, $c=v$ appears in the head of a rule whose body is satisfied, and $c=v$ has to be satisfied to make this rule satisfied. This rule accounts for its stability.

\item $c=v$ is not possible in $I$. For $c=v$ to be in $J$, there must be some rules in $\Pi^{\lpmln}$ to account for its stability. By the construction of $\Pi^{\lpmln}$ there can be no rules in $\Pi^{\lpmln}\setminus \tau(\Pi)$ that accounts for its stability. So there must be rules in $\tau(\Pi)$ to account for its stability.
\end{itemize}
\EOCC

So $I$ is a stable model of $\tau(\Pi)$, and thus a possible world of $\Pi$.
\end{itemize}
\qed
\end{proof}

\BOCC
\begin{cor}
For any mini P-log program $\Pi$ and any formula $F$ over $\sigma(\Pi)$, we have
\begin{equation}\nonumber
P_{\Pi}(F) = P_{\Pi^{\lpmln}}(F).
\end{equation}
\end{cor}
\EOCC

\noindent{\bf Theorem~\ref{thm:p-log-to-lpmln} \optional{thm:p-log-to-lpmln}}\
\ 
{\sl
For any consistent mini P-log program $\Pi$ of signature $\sigma$ and any possible world $I$ of $\Pi$, we construct a formula $F_I$ as follows.
\[
\ba {rl}
   F_I = & (\bigwedge_{c=v\in I}c=v)\wedge  \\
         & (\bigwedge_{\substack{\text{$c,v:$}\\\text{$c=v$ is possible in $I$,}\\\text{ $I\models c=v$ and $PR_{I}(c)\neq \emptyset$}}}pf^{c}_{B_{I,c}, r_{I,c}}=v)\ \\ 
        & \wedge (\bigwedge_{\substack{\text{$c,v:$}\\\text{$c=v$ is possible in $I$,}\\\text{ $I\models c=v$ and $PR_{I}(c)= \emptyset$}}}pf^{c}_{\Box, r_{I,c}}=v)
\ea 
\] 
We have
\[
  \mu_{\Pi}(I) = P_{\Pi^{\lpmln}}(F_I).
\] 
For any proposition $A$ of signature $\sigma$,
\[ 
   P_\Pi(A) = P_{\Pi^{\lpmln}}(A).
\]

}
\vspace{3 mm}

\begin{proof}

We first show
\begin{equation}\nonumber
\sum_{\text{$I$ is a possible world of $\Pi$}}\hat{\mu}_{\Pi}(I)\ \ \ \ \ \ \ \ \ =\ \ \ \ \ \ \ \ \ \sum_{J\in \smdbprm\left[\Pi^{\lpmln}\right]}\wdbprm_{\Pi^{\lpmln}}(J)
\end{equation}
i.e., the normalization factor of $\hat{\mu}$ is the normalization factor of $\wdbprm_{\Pi^{\lpmln}}$.

By Lemma \ref{thm:p-log-to-lpmln-unnormalized} we have, 
\begin{align}
\label{lem7eq} \sum_{\text{$I$ is a possible world of $\Pi$}}\hat{\mu}_{\Pi}(I)\ \ \ \ \ \ \ \ \ &=\ \ \ \ \ \ \ \ \  \sum_{\text{$I$ is a possible world of $\Pi$}}\sum_{J\in SM_{\Pi}(I)}\wdbprm_{\Pi^{\lpmln}}(J)
\end{align}
By Lemma \ref{lem:lpmln-to-plog}, for every stable model $J$ of $\Pi^{\lpmln}$, there exists a possible world $I$ of $\Pi$ such that $J\in SM_{\Pi}(I)$. So we can enumerate all stable models of $\Pi^{\lpmln}$ by enumerating all possible worlds $I$ of $\Pi$ and enumerating all elements in $SM_{\Pi}(I)$ for each $I$, and thus the right-hand side of (\ref{lem7eq}) can be rewritten as

\begin{align}\nonumber
\sum_{\text{$J$ is a stable model of $\Pi^{\lpmln}$}}\wdbprm_{\Pi^{\lpmln}}(J).
\end{align}

By Lemma \ref{lem:smdbprm}, an interpretation $J$ is a stable model of $\Pi^{\lpmln}$ if and only if $J\in \smdbprm\left[\Pi^{\lpmln}\right]$. So the right-hand side of (\ref{lem7eq}) can be further rewritten as

\begin{align}\nonumber
\sum_{J\in \smdbprm\left[\Pi^{\lpmln}\right]}\wdbprm_{\Pi^{\lpmln}}(J).
\end{align}

Thus we have
\begin{align}\nonumber
\mu_{\Pi}(I) &= \frac{\hat{\mu}_{\Pi}(I)}{\sum_{\text{$I$ is a possible world of $\Pi$}}\hat{\mu}(I)}\\
\nonumber  &=\frac{\hat{\mu}_{\Pi}(I)}{\sum_{J\in \smdbprm\left[\Pi^{\lpmln}\right]}\wdbprm_{\Pi^{\lpmln}}(J)}\\
\nonumber (\text{By Lemma \ref{thm:p-log-to-lpmln-unnormalized}}) &=\frac{\sum_{J\in SM_{\Pi}\left[I\right]}\wdbprm_{\Pi^{\lpmln}}(J)}{\sum_{J\in \smdbprm\left[\Pi^{\lpmln}\right]}\wdbprm_{\Pi^{\lpmln}}(J)}\\
\nonumber &=\sum_{J\in SM_{\Pi}\left[I\right]}\frac{\wdbprm_{\Pi^{\lpmln}}(J)}{\sum_{J\in \smdbprm\left[\Pi^{\lpmln}\right]}\wdbprm_{\Pi^{\lpmln}}(J)}\\
\nonumber &=\sum_{J\in SM_{\Pi}\left[I\right]}\pdbprm_{\Pi^{\lpmln}}(J)
\end{align}
For those interpretations $J$ that do not belong to $SM_{\Pi}\left[I\right]$ but satisfy $F_I$, it must be the case that $J$ is not a stable model of $\Pi^{\lpmln}$. By Lemma \ref{lem:smdbprm}, $\pdbprm_{\Pi^{\lpmln}}(J) = 0$. So we have
\begin{align}
\nonumber \mu_{\Pi}(I) &=\sum_{\text{$J\in SM_{\Pi}\left[I\right]$ and $J\vDash F_I$}}\pdbprm_{\Pi^{\lpmln}}(J) + \sum_{\text{$J\notin SM_{\Pi}\left[I\right]$ and $J\vDash F_I$}}\pdbprm_{\Pi^{\lpmln}}(J)\\
&=\sum_{J\vDash F_I}\pdbprm_{\Pi^{\lpmln}}(J)
\end{align}
and consequently by Theorem \ref{thm:pdbprm},
\begin{align}
\nonumber \mu_{\Pi}(I) &= \sum_{J\vDash F_I}P_{\Pi^{\lpmln}}(J)\\
\label{eq:plog2lpmln} &=P_{\Pi^{\lpmln}}(F_I).
\end{align}

According to the definition,
\begin{align}
\nonumber P_{\Pi}(F)=\sum_{\text{$W$ is a possible world of $\Pi$ that satisfies $F$}}\mu_{\Pi}(W).
\end{align}
Using the above result (\ref{eq:plog2lpmln}), we have
\begin{align}
\nonumber P_{\Pi}(F)&=\sum_{\text{$W$ is a possible world of $\Pi$ that satisfies $F$}}P_{\Pi^{\lpmln}}(F_W)\\
\nonumber &=\sum_{\text{$W$ is a possible world of $\Pi$ that satisfies $F$}}\sum_{I\in SM_{\Pi}(W)}P_{\Pi^{\lpmln}}(I).
\end{align}
The right-hand side of the last equation is the sum of the probabilities of a collection of stable models of $\Pi^{\lpmln}$. Clearly all those stable models of $\Pi^{\lpmln}$ satisfies $F$ since they are all from some $SM_{\Pi}(I)$ for some possible world $I$ of $\Pi$ that satisfies $F$. Furthermore, given any stable model $J$ of $\Pi^{\lpmln}$ that satisfies $F$, by lemma \ref{lem:lpmln-to-plog}, there exists a possible world $I$ of $\Pi$ such that $J\in SM_{\Pi}(I)$. Since $I$ and $J$ agree on all atoms in $\sigma(\Pi)$ and $J\vDash F$, $I\vDash F$. So the probability of $J$ is counted in the right-hand side of the above equation. Finally, obviously no two stable models of $\Pi^{\lpmln}$ are counted twice. Hence, the right-hand side can be rewritten as
\begin{align}
\nonumber &P_{\Pi^{\lpmln}}(F),
\end{align}
and thus we have
\begin{equation}\nonumber
P_{\Pi}(F) = P_{\Pi^{\lpmln}}(F).
\end{equation}
\qed
\end{proof}

\end{document}